\documentclass{article}
\usepackage{float}

\PassOptionsToPackage{numbers, compress}{natbib}
\usepackage[preprint]{neurips_2025}

\usepackage[utf8]{inputenc} % allow utf-8 input
\usepackage[T1]{fontenc}    % use 8-bit T1 fonts
\usepackage{hyperref}       % hyperlinks
\usepackage{url}            % simple URL typesetting
\usepackage{booktabs}       % professional-quality tables
\usepackage{amsfonts}       % blackboard math symbols
\usepackage{nicefrac}       % compact symbols for 1/2, etc.
\usepackage{microtype}      % microtypography
\usepackage{xcolor}         % colors

\usepackage{algorithm}
\usepackage{algpseudocode}
\usepackage{algorithm}

\usepackage{dsfont}

\usepackage{amssymb}
\usepackage{amsmath}
\usepackage{amsthm}
\DeclareMathOperator*{\argmax}{arg\,max}

\newtheorem{theorem}{Theorem}

\newtheorem{definition}{\bf Definition}
\newtheorem{proposition}{\bf Proposition}
\newtheorem{remark}{\bf Remark}
\newtheorem{example}{Example}[section]

\usepackage{graphicx, wrapfig}
\usepackage{caption}
\usepackage{float}
\title{DO-IQS: Dynamics-Aware Offline Inverse Q-Learning \\ for Optimal Stopping with Unknown Gain Functions}

\author{%
  Anna Kuchko  \\
  Department of Statistics\\
  University of Warwick\\
  Coventry, UK \\
  \texttt{anna.kuchko@warwick.ac.uk} \\
}

\begin{document}

\maketitle

\begin{abstract}
  We consider the Inverse Optimal Stopping (IOS) problem where, based on stopped expert trajectories, one aims to recover the optimal stopping region through the continuation and stopping gain functions approximation. The uniqueness of the stopping region allows the use of IOS in real-world applications with safety concerns. Although current state-of-the-art inverse reinforcement learning methods recover both a Q-function and the corresponding optimal policy, they fail to account for specific challenges posed by optimal stopping problems. These include data sparsity near the stopping region, the non-Markovian nature of the continuation gain, a proper treatment of boundary conditions, the need for a stable offline approach for risk-sensitive applications, and a lack of a quality evaluation metric. These challenges are addressed with the proposed \textbf{D}ynamics-Aware \textbf{O}ffline \textbf{I}nverse \textbf{Q}-Learning for Optimal \textbf{S}topping (\textbf{DO-IQS}), which incorporates temporal information by approximating the cumulative continuation gain together with the world dynamics and the Q-function without querying to the environment.  In addition, a confidence-based oversampling approach is proposed to treat the data sparsity problem. We demonstrate the performance of our models on real and artificial data including an optimal intervention for the critical events problem.
\end{abstract}

\section{Introduction}
\label{sec:introduction}

Rapidly advancing autonomous systems require efficient and precise detection and anticipation of hazardous situations. The theory of optimal stopping (OS) can be applied to problems like stopping of an autonomous vehicle, identification of a shutdown time for a production system, change point detection, etc. Designing a reward function to solve the OS problems requires a high degree of involvement and an extensive knowledge of the system's dynamics. Recent advances in the area of Inverse Reinforcement Learning (IRL) allow one to learn both a Q-function and the optimal policy directly from demonstrations \cite{garg2021,Ho2016}. For safety-sensitive applications, the world dynamics model can also be learnt offline for greater risk awareness \cite{zeng2023understanding,zeng2023when}. While useful for general optimal control problems, these methods often fail while applied to OS problems, mainly due to sparsity of the expert data near the stopping region, non-Markovian nature of the continuation gain and/or the state-space dynamics. We address these challenges with our proposed Dynamics-Aware Offline Inverse Q-Learning model for OS problems with unknown gain functions (DO-IQS). 
\begin{figure*}[htbp]
\centering
\includegraphics[width=1\linewidth]{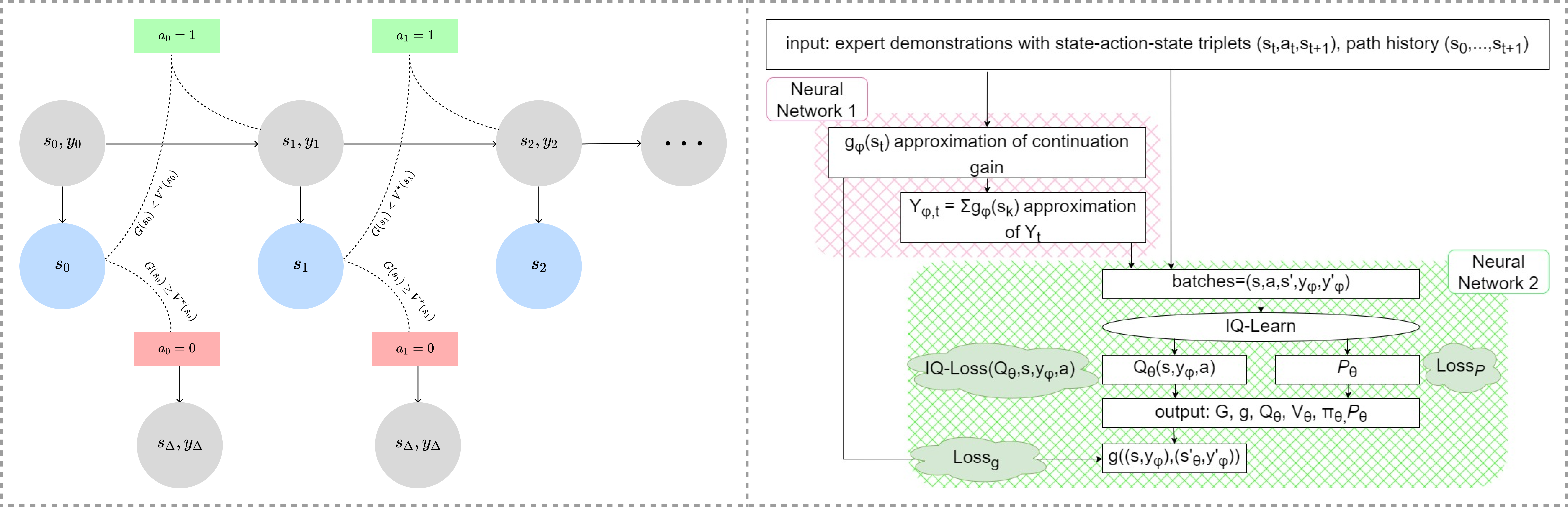} 
\caption{POMDP structure of the OS problem with cumulative continuation gain (left) and DO-IQS model structure (right).}
% \vskip 2pt
\centering
\label{fig:DIQS}
\end{figure*}

\subsection{Outline of contributions}\label{sec:outline}
The literature on the (inverse) OS problems is mostly limited to financial applications and considers very specific examples that overlook a more general setup and do not fully address the main challenges of a practical application. In this paper, we make the following contributions to the theory and practice of IOS problems: 
\begin{itemize}
    \item We formally define a reinforcement learning (RL) formulation of the OS problem accounting for both stopping and continuation gain functions, allowing us to use IRL methods (Definition \ref{def:smdp}).  
    \item An Inverse Optimal Stopping problem definition is presented (Definition \ref{def:ios_irl}). Fitting a stopping surface can be challenging in higher dimensions. Estimating gain functions instead provides higher precision and greater flexibility (\hyperlink{insight:reward_shaping}{Insight 1}). 
    \item We discuss the main challenges that arise in IOS problems, including the scarcity of data near the stopping boundary and the sparsity of stopping decisions caused by a natural class imbalance between continuation and stopping actions (Section \ref{sec:inverse_os}). We propose a solution to the problem via oversampling strategies (\hyperlink{insight:data_sparsity}{Insight 2}).
    \item To allow offline inference (while also correcting the Bellman error of the IQ-Learning algorithm), we estimate the environment model (\hyperlink{insight:dynamics_aware_reward}{Insight 3}). We further develop an approach, which we call a Dynamics-Aware Offline Inverse Q-Learning for Optimal Stopping (DO-IQS, Figure \ref{fig:DIQS}), to incorporate non-Markovian continuation reward learning to (Section \ref{sec:DO-IQS}) and to fight the sparsity problem.
    % \item For the purpose of hyper-parameter tuning and model selection, we propose treating the OS region recovery problem as a classification problem by using classification metrics for imbalanced data. We use Balanced Accuracy as a metric of choice and also report a trade-off between median time-to-event and median event-miss rate (Section \ref{sec:quality_evaluation}).
\end{itemize}
We demonstrate the performance of our approach on OS problems, including an OS of a two-dimensional Brownian motion and learning optimal intervention policies for critical events (Section \ref{sec:experiments}). 

\subsection{Related work}\label{sec:related_work}
The literature on IOS and related problems is scarce and often lacks generality. \citet{kruse2017inverse} considers a problem of finding a way to change (unknown) gain function in a way that makes a given stopping rule optimal and show how such a problem can be analytically solved for a particular class of one-dimensional diffusion processes. We note that this approach lacks generality and is hard to use in higher dimensions. \citet{qiao2013behavior} use Bayesian modification of apprenticeship learning first introduced by \citet{abbeel-ng-apprentice} to solve a standard secretary problem. \citet{alhafez2023lsiq} modify the IQ-Learn algorithm of \citet{garg2021} to account for problems with absorbing states by introducing a fixed target reward function and regressing it towards the absorbing states. Although this approach has some similarities with OS problems, it does not make use of the natural structure of the OS rewards and cannot be applied to more challenging OS examples. An analytical solution to a class of mean-field stopping games is also presented by \citet{huang2022class}. \citet{2023inversebayesos} address an inverse optimal Bayesian stopping problem using a modification of the Max-Entropy IRL (\citet{Ziebart2008}) with application to inverse sequential hypothesis testing and inverse Bayesian search problems over multiple environments. In our approach, we take on a more general set-up that does not require learning in different environments or a manual engineering of feature vectors. This allows us to conduct inference in higher dimensions and for a wider range of OS problems. Moreover, the authors evaluate the results of their inference by comparing the true environment costs to the IRL estimates. Instead, we use classification metrics to evaluate the quality of the resulting stopping region which, unlike the gain functions, is unique. Another recent work by \citet{2022oti} studies an OS approach to optimal time to intervention (OTI) problems, where the stopping gain function takes the form of a hazard rate process and is recovered using deep recurrent survival analysis. We show that our approach can also be applied to inverse OTI problems without requiring threshold tuning, imposing constraints on the shape of the hazard rate, and avoiding an error made by the survival model.

\section{Background and problem setup}\label{sec:background}
\subsection{Optimal stopping: general formulation}\label{sec:optimal_stopping_general}
Let \(S  = \{s_t\}_{t\geq 0} \in \mathbb{R}^d\) be a time-homogeneous Markov state-space process defined on a probability space \( (\Omega, \mathcal{F}, \mathbb{P}_s, s\in S) \) equipped with filtration \((\mathcal{F}_t)_{t\geq 0}\) generated by the history of observed states up to time \(t\). A stopping time \(\tau:\Omega\rightarrow \{0,1,...,\infty\}\) is a random variable \( \tau \in \{0,1,2,...,\infty\}\). Given a \textit{continuation gain function} \(g:\mathbb{R}^d\rightarrow \mathbb{R}\) and a \textit{stopping gain function} \(G:\mathbb{R}^d \rightarrow \mathbb{R}\) a value function defines the expected cumulative future payoff, if the state process is stopped at the time \(\tau\):
\begin{equation}\label{optimal_stopping_value_function}
    V^{\tau}(s) = \mathbb{E}_s\left[\sum_{t=0}^{\tau-1}\gamma^t g(S_t)\mathds{1}_{\{ \tau \geq 1\}}  + \gamma^{\tau}G(S_{\tau})\right],
\end{equation}
where \(\gamma\in \left[0,1\right]\) is a discount factor. The corresponding \textit{optimal stopping problem} consists in finding an optimal stopping time \(\tau^{\star}\) maximising the value function, i.e. \(V^{\tau^{\star}}(s) = \sup_{\tau}V^{\tau}(s)\). By the general theory of OS \cite{Tsitsiklis1999}, \(V^{\tau^{\star}}=V^{\star}\), which is a fixed point of the dynamic programming operator \(T\): \(V^{\star} = TV^{\star}\), where \(TV = max \{G, g+\gamma PV\}\), and \(P\) is a transition operator \(\mathbb{E}\left[V(s')|\mathcal{F}_t\right] = \int V(y)P(s, dy) = (PV)(s)\). Then the \textit{optimal stopping region} can be defined as a set of states:
\begin{equation}\label{eq:stopping_region}
    \begin{aligned}
    D^{\star} = \{s| G(s)\geq V^{\star}(s)\}
    = \{s| G(s)\geq g(s)+\gamma\mathbb{E}_{s'} V^{\star}(s')\},
    \end{aligned}
\end{equation}
with the \textit{optimal stopping time} being the shortest stopping time when the process hits the stopping region \(\tau^{\star}=\min\{t\geq 0|s_t\in D^{\star}\}\).

\subsection{Optimal stopping via reinforcement learning} \label{sec:opt_st_via_rl}
To apply SOTA methods to the OS, we first formulate it as an RL problem. Consider a problem related to the OS problem that we define in terms of a Stopped Markov Decision Process. \par 
\begin{definition}[\bf{SMDP}]\label{def:smdp}
A Stopped Markov Decision Process (SMDP) is a tuple\\ \(M_{\Delta} = (S_{\Delta}, A,\mathcal{P}, g, G, \gamma)\), where \(\mathcal{S}_{\Delta} := \{s\}\cup \{\Delta\}\) is a Borel set of states augmented with a set of cemetery points \(\{\Delta\}\). \(A:\mathcal{S}_{\Delta}\rightarrow A\) is a mapping on \(\mathcal{S}_{\Delta}\) such that \(A\subseteq \mathcal{A}\) is a non-empty set of admissible binary actions \(\mathcal{A} = \{0,1\}\) where the decision to stop is associated with \(a=0\) and the decision to continue with \(a=1\). The history of observations up to time \(n\geq0\) is \(o_n:=(s_k,a_k)_{0\leq k\leq n}\) defined on the space \(O_n:=(\mathcal{S}_{\Delta}\times \mathcal{A})^n\times \mathcal{S}_{\Delta}\). \(g,G: \mathcal{S}_\Delta\times\mathcal{A} \rightarrow \overline{\mathbb{R}}\) are the continuation and stopping gain functions, respectively, with \(\overline{\mathbb{R}} = \mathbb{R}\cup\{-\infty, +\infty\}\). For convenience, the rewards are assumed to be bounded and Borel measurable (discounted MDPs with countable state space and unbounded gains can be turned into analogous problems with bounded gains \cite{Wal1980}). \(\gamma\in(0,1)\) is a discount factor. The MDP dynamic is given by a family of well-defined probabilities:
    \begin{equation}
    \mathcal{P}(B|s,a) =
    \begin{cases}
      \mathbb{P}_s(s_1\in B) \text{ if } s\in\mathcal{S}, a=1, \\
      \delta_{\Delta}(B)\text{ otherwise,}\\
    \end{cases}\
\end{equation}\par
for any \(B\subseteq\mathcal{S}_{\Delta}\), and where we assume that for \(s\in\mathcal{S}_{\Delta}\), \(\delta_s\) is the Dirac measure with the concentration in \(s\).
\end{definition}

For this SMDP, a Markov deterministic policy \(\pi:\mathcal{S}_{\Delta}\rightarrow A\) is defined by a set of Dirac measures \(\pi:\mathcal{S}_{\Delta}\rightarrow\{\delta_0,\delta_1\}\), such that \(\pi(\cdot|\Delta)=\delta_0(\cdot)\). Any such policy, together with an initial states distribution \(p_0\), yields a pair of processes \((S_t^{\dag}, A_t)_{t\geq0}\) on a probability space \((\Omega^{\dag},\mathcal{F}^{\dag},\mathbb{P}^{\pi}_{p_0})\). Introduce a composite reward function:
\begin{equation}\label{eq:stopping_reward}
r(s,a) = g(s)a+G(s)(1-a),    
\end{equation}
for \(s\in\mathcal{S}_{\Delta},a\in \mathcal{A}\) and the corresponding value function:
\begin{equation}\label{equation:opt_st_rl}
    {V^\pi}(s)=\mathbb{E}^{\pi}_s\left[\sum_{t=0}^{\infty}\gamma^t r(S_t^{\dag},A_t)\right],
\end{equation}

where we write the expectation with respect to \(\mathbb{P}^{\pi}_{p_0}\) as \(\mathbb{E}^{\pi}_{p_0}\) (\(\mathbb{P}^{\pi}_s\) and \(\mathbb{E}^{\pi}_s\) for \(p_0=\delta_s\), respectively). This form of reward function will be central in translating an OS problem into a universal RL problem.\par 
Following the definitions above, for the bounded reward functions there exists \cite{Harrison1972} an optimal stationary deterministic policy solving 
\begin{equation}\label{eq:opt_st_rl_optimality}
    \begin{aligned}
    {V}^{\pi^\star}(s)=\sup_{\pi\in\Pi}\mathbb{E}^{\pi}_s\left[\sum_{t=0}^{\infty}\gamma^t r(S_t^{\dag},A_t)\right] 
    =\mathbb{E}^{\pi^{\star}}_s\left[\sum_{t=0}^{\infty}\gamma^t r(S^{\dag}_t,A_t)\right].    
    \end{aligned}
\end{equation}
\begin{proposition}\label{proposition:opt_st_rl_equivalence}
\(V^{\tau^\star}(s)={V}^{\pi^\star}(s)\) are equivalent.
\end{proposition}
% \begin{proof}[Sketch of proof] 
% Proof through coupling of \(S_t\) and \(S^{\dag}_t\) and showing that the stopping time under the optimal policy in the RL problem is equal to the hitting time of the optimal stopping region in the OS problem. See the supplementary materials for details. 
% \end{proof}
\textit{Scetch of proof} 
Proof through coupling of \(S_t\) and \(S^{\dag}_t\) and showing that the stopping time under the optimal policy in the RL problem is equal to the hitting time of the optimal stopping region in the OS problem. See the supplementary materials for details. \par 

In what follows, we will focus on SMDPs and write \(S\) instead of \(S^{\dag}\) for brevity.\par
It is often useful to consider a state-action \(Q\)-function \(Q^{\pi}\in \mathbb{R}^{|\mathcal{S}|\times|\mathcal{A}|}\), representing the value of a policy \(\pi\) that started in a state \(s\) and executed an action \(a\), defined as \(Q^{\pi}(s,a) = \mathbb{E}^{\pi}_{s_0=s,a_0=a}\left[\sum_{t=0}^{\infty}\gamma^tr(s_t,a_t)\right] = r(s,a)+\gamma V^{\pi}(s')\).

\subsection{Stochastic policies and soft-Q learning}\label{sec:stochastic_policies}
We first note that the optimal deterministic policy can be represented by a set of Dirac-delta functions, assigning a probability mass to one of the optimal actions. \cite{Bertsekas1995} showed that for any MDP there exists a stationary deterministic policy that is optimal. In this work, we assume that all the policies are (time-)stationary, i.e. the probability distribution over actions does not change over time. The OS policy defined by a set of Delta-functions conditioned on the continuation and stopping region is equivalent to a greedy policy with respect to the \(Q\)-function, i.e.
\begin{equation}
\begin{aligned}
    \pi^{\star}(a|s) = \begin{cases}
     \delta_1(a):s\in C^{\star} \\
     \delta_0(a):s\in D^{\star}_{\Delta}\\
    \end{cases} 
    =\delta\left(a=\argmax_a Q^{\pi^{\star}}(s,a)\right).
\end{aligned}
\end{equation}

Here we write \(D^\star_\Delta = D^\star \cup \{\Delta\}\) to define the optimal stopping region augmented with the cemetery states, and \(C^\star= \{s|G(s) < V^\star(s)\}\) to define the optimal continuation region. Since many SOTA algorithms for (inverse) RL are entropy based, we define a stochastic exploratory policy using the Boltzmann distribution (or softargmax function), which, in limit, converges to the optimal deterministic Delta-policy:
\begin{equation}\label{eq:limit_of_stochastic_policy}
    \begin{aligned}
        \lim_{\epsilon\rightarrow 0}\pi_{\epsilon}(a|s) = \lim_{\epsilon\rightarrow 0}\frac{\exp\{Q(s,a)/\epsilon\}}{\sum_{a'\in\mathcal{A}}\exp\{Q(s,a')/\epsilon\}}
        =\delta\left(a=\argmax_{a'}Q(s,a')\right),
    \end{aligned}
\end{equation}
where \(\epsilon\) is the energy parameter defining the level of the stochasticity of the Boltzmann policy. See supplementary material for details.\par 
This stochastic policy turns out to be an optimal policy under maximum entropy RL \cite{zhou2017entropy}. We define the soft Bellman operator \(\mathcal{B}^{\pi_{\epsilon}} :\mathbb{R}^{\mathcal{S}\times\mathcal{A}}\rightarrow\mathbb{R}^{\mathcal{S}\times\mathcal{A}}\) acting on the set of soft Q-functions: \((\mathcal{B}^{\pi_{\epsilon}}Q)(s,a)=r(s,a)+\gamma\mathbb{E}_{s'\sim\mathcal{P}(\cdot|s,a)}V^{\pi_{\epsilon}}(s')\), where\\ \(V^{\pi_{\epsilon}}(s)=\mathbb{E}_{a\sim\pi_{\epsilon}(\cdot|s)}\left[Q(s,a)-\epsilon\log\pi_{\epsilon}(a|s)\right]=\sigma_{a}^\epsilon Q(s,a)\) and\\ \(\sigma_{a}^\epsilon Q(s,a)=\epsilon\log\sum_{a'}\exp\{Q(s',a')/\epsilon\}\) is a softmax function with the temperature parameter \(\epsilon\). The soft Bellman operator is a contraction mapping of the soft Q-function, i.e. \(\mathcal{B}^{\pi_{\epsilon}} Q=Q\) \cite{haarnoja2018soft}. In what follows, we write \(H(\pi)=\mathbb{E}_{\pi}\left[-\log \pi(a|s)\right]\) to define the discounted causal entropy and \(\rho_\pi(s,a)=\pi(a|s)\sum_{t=0}^{\infty}\gamma^tP(s_t=s|\pi)\) to define an occupancy measure (state visitation counts) of a policy \(\pi\).\par In the case of an OS, we call the stochastic policy (\ref{eq:limit_of_stochastic_policy}) an \textbf{\(\epsilon\)-optimal stopping policy} and the corresponding V- and Q-functions \textbf{soft value functions}. This results in near-optimal stopping times.
\begin{definition}[\bf{\(\epsilon\)-optimal stopping time}]
We say that the stopping time is near-optimal, or \(\epsilon\)-optimal if it is a solution to the problem (\ref{equation:opt_st_rl}) under an \(\epsilon\)-optimal stopping policy, i.e.\(\tau^{\star}_{\epsilon} = \inf\{t\geq0 |s_t\in D^{\star}_{\Delta,\epsilon}\}\),
    and 
    \begin{equation}\label{eq:epsilon_stopping_region}
        D^{\star}_{\Delta,\epsilon}=\{s|G(s)\geq V^{\pi^{\star}_{\epsilon}}(s)\}.
    \end{equation}
\end{definition}
For soft-Q learning, noting that the softargmax function (Boltzmann distribution) is the gradient of log-sum-exp function, for \(\epsilon\rightarrow 0\) the bound on the value function, and hence on the stopping region, can be made as tight as required: \(V^{\pi^\star}<V^{\pi^\star_\epsilon}\leq V^{\pi^\star}+\epsilon\log(2)\).\par 

\subsection{Inverse reinforcement learning}\label{sec:irl}
An inverse reinforcement learning (IRL) aims to recover a reward function \(r\in\mathcal{R}\) and a corresponding optimal policy \(\pi\in\Pi\) based on a set of expert observations \(O_t\) produced by an expert policy \(\pi_E\). \par 
\textbf{Maximum entropy IRL.} In case of Max Entropy IRL \cite{Ziebart2008,Ziebart2010,Ho2016}, the objective is defined as follows:
\begin{equation}\label{eq: max_entropy_irl}
    \begin{aligned}
        {}&L(\pi, r)= d_\psi(\rho,\rho_E)-H(\pi)
        =\mathbb{E}_{\rho_E}\left[r(s,a)\right]-\mathbb{E}_\rho\left[r(s,a)\right]-H(\pi)-\psi(r),\\
        {}&\max_{r\in\mathcal{R}}\min_{\pi\in\Pi} L(\pi,r),
    \end{aligned}
\end{equation}
where \(\psi:\mathbb{R}^{S\times\mathcal{A}}\rightarrow\mathbb{R}\) is a convex reward regulariser, \(\rho_E\) is an expert's occupancy measure and \(d_\psi=\psi^*(\rho_E-\rho)\) is a statistical distance between two occupancy measures with \(\psi^*\) being a convex conjugate of \(\psi\).\par 
\textbf{Inverse soft-Q learning.} In recently introduced Inverse soft-Q learning approach (IQ-Learn) \cite{garg2021} the authors suggested projecting the problem (\ref{eq: max_entropy_irl}) onto the Q-policy space. They first introduce an inverse Bellman operator \((\mathcal{T} ^\pi Q)(s,a)=Q(s,a)-\gamma\mathbb{E}_{s'\sim\mathcal{P}(\cdot|s,a)}V^\pi(s')\),
and then use the following objective function:
\begin{equation}\label{eq:iq_learn_optimisation}
    \begin{aligned}
        {}&\mathcal{J}(Q)= \mathbb{E}_{\rho_E}\left[\phi\left(Q(s,a)-\gamma\mathbb{E}_{s'\sim\mathcal{P}(\cdot|s,a)}\sigma_{a'}^\epsilon \left(Q(s',a')\right)\right)\right]
        -(1-\gamma)\mathbb{E}_{\rho_0}\left[\sigma_{a}^\epsilon \left(Q(s_0,a)\right)\right],\\{}&\max_{Q\in\Omega}\mathcal{J}(Q) = \max_{r\in\mathcal{R}}\min_{\pi\in\Pi} L(\pi,r).
    \end{aligned}
\end{equation}
IQ-Learn allows recovering the Q-function corresponding to the expert's behaviour, which in turn gives both a reward function \(r(s,a) = Q(s,a)-\gamma \mathbb{E}_{s'\sim\mathcal{P}(\cdot|s,a)}V^{\pi_\epsilon}(s')\) and the corresponding stochastic policy \(\pi_\epsilon\) as in (\ref{eq:limit_of_stochastic_policy}). This will be beneficial for recovering the stopping region in IOS using the stopping rule (\ref{eq:stopping_region}).

\section{Inverse optimal stopping with unknown gain functions}\label{sec:inverse_os}
In the most general sense, the \textit{Inverse Optimal Stopping (IOS)} problem is concerned with finding a continuation reward function \(g(\cdot)\) and a stopping reward function \(G(\cdot)\), given an OS time \(\tau^{\star}\) or an optimal value function evaluated at the OS time \(V^{\star}\). This problem is ill-posed and its solution is often intractable. Instead, we define the IOS problem using the IRL formulation.

% \begin{figure}[tbp]
% \centering
% \includegraphics[width=0.75\linewidth]{images/sparsity_plus_confusion.png}
%   \caption{\textbf{Left}: Stopping with sparse observations. Continuation and stopping region are represented by green and red areas respectively. Red dots: expert stopping decisions. \textbf{Right}: OS True Negatives (TN), True Positives (TP), False Negatives (FN) and False Positives (FP).}
% \label{fig:sparse_stopping_region}
% \end{figure}

\begin{definition}[\bf{Inverse optimal stopping problem: IRL formulation}]\label{def:ios_irl}
Assume that the realisations \((S_t,A_t)_{t=0}^\infty\) of SMDP as defined in \ref{def:smdp} are given and can be characterised by some unknown optimal policy \(\pi^\star\). The problem of finding the continuation and stopping gain functions \(g(\cdot)\) and \(G(\cdot)\) explaining \(\pi^\star\) is called an Inverse Optimal Stopping problem.
\end{definition}

Although the reward function in IOS problems is non-unique, unlike general IRL problems, OS have a \textit{unique stopping region}. This means that an inverse algorithm should be able to find this unambiguous stopping region, which provides clear benefits for safety-sensitive applications. Here, we aim to fully utilise the OS rule (\ref{eq:stopping_region}) which allows us to choose an action by comparing the Q-functions evaluated for the respective actions. Our RL formulation (\ref{def:smdp}) of the OS problem allows using an IQ-Learn algorithm (\ref{eq:iq_learn_optimisation}), while preserving the structure of the original problem. We address the main challenges associated with the IOS problem with the following insights.\par

\hypertarget{insight:reward_shaping}{\textbf{Insight 1: IOS as a reward shaping problem}}. Although IQ-Learn and other SOTA methods for IRL show a high positive correlation with the ground-truth rewards \cite{garg2021} (some of them recovering the true rewards up to a constant, e.g. GAIL \cite{Ho2016}), most of them suffer from a covariance shift problem and considerably high uncertainty in the reward estimation. Incorporating a model-based approach into the IRL model and shaping the reward in a more informed way helps to reduce the uncertainty (see, e.g. \cite{Ho2016,yue2023clare}). \par 
\textit{Our solution:} 
We suggest to: \textbf{(a)} Iteratively update both Q-function and the model of dynamics in a bi-level constrained optimisation manner similar to \cite{zeng2023understanding} and \cite{pmlr-v51-herman16}. This means that at each iteration we update our beliefs about the environment dynamics model \(\hat{\mathcal{P}}\) to fit the observed dynamics \(\mathcal{P}\) under the current Q-model and, similarly, update \(\hat{Q}\) to fit the approximated dynamics model. For details, see Section \ref{sec:offline_iq_learn_for_ios}; \textbf{(b)} Utilise the reward form suggested in (\ref{eq:stopping_reward}) to incorporate the stopping and continuation gain functions into the learning process to ensure consistency with the OS update (\ref{eq:stopping_region}).\par 

\setlength{\textfloatsep}{0pt}

\begin{wrapfigure}{r}{0.5\linewidth}
\vskip -10pt
\centering
\includegraphics[width=\linewidth]{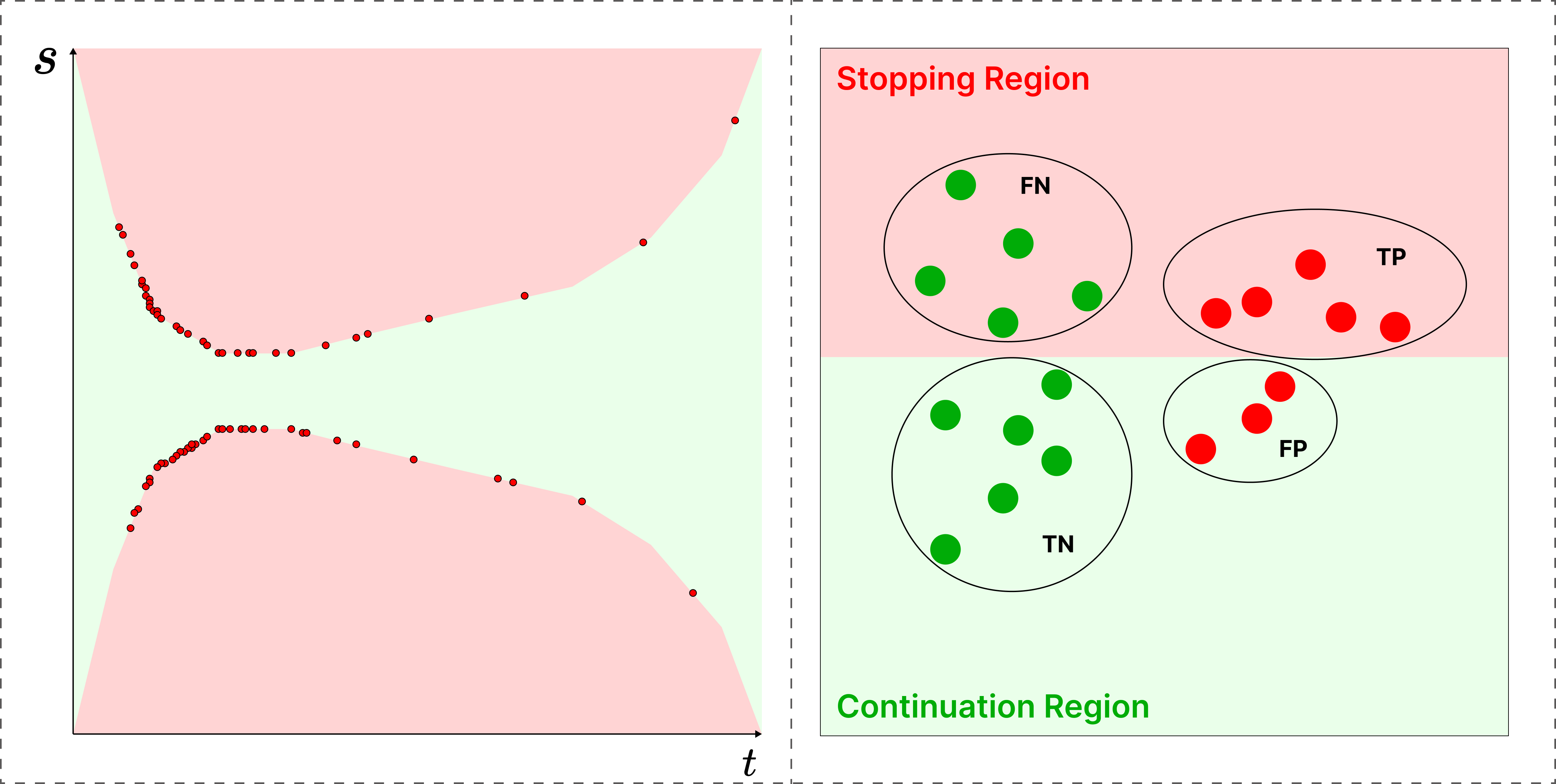}
  \caption{\textbf{Left}: OS  with sparse observations. Red and green dots represent expert stopping and continuation decisions. \textbf{Right}: True Negatives (TN), True Positives (TP), False Negatives (FN) and False Positives (FP).}
\label{fig:sparse_stopping_region}
\vskip -15pt
\end{wrapfigure}
\hypertarget{insight:data_sparsity}{\textbf{Insight 2: Data sparsity near and on the stopping boundary}}. First, there is a natural data sparsity caused by the class imbalance between stopping and continuation actions. For example, for a stopped expert dataset with \(N\) paths (hence \(N\) stopping actions) and \(M\) observations in total, the ratio between stopping and continuation actions is \(\frac{N}{M-N}\) with \(M\gg N\). Second, in certain cases, only a few trajectories will reach further side of the stopping boundary at later times (e.g., Figure \ref{fig:sparse_stopping_region}).

\textit{Our solution:}  We first propose to apply oversampling strategies where the oversampling error is leveraged by IRL (see Sections \ref{sec:oversampling_approach} for details). Since this only works in Markovian case, we show how approximating a cumulative continuation gain improves the performance of the algorithm in the presence of sparsity. We will describe the latter approach in the next insight. \par 

\hypertarget{insight:dynamics_aware_reward}{\textbf{Insight 3: Cumulative continuation gain for dynamics-aware reward shaping}}. 
Looking back to the RL formulation (\ref{equation:opt_st_rl}), we could notice that the problem cannot be studied in a Markovian way, since at each step the stopping decision is made if: \(G(s_t)\geq \sum_{t'\leq t, t\geq0}\gamma^{t'}g(s_t)+\gamma^{t+1}\mathbb{E}_{s_{t+1}\sim\mathcal{P}(\cdot|s_t,a_t)}V^{\star}(s_{t+1})\), turning the original RL problem into a problem with non-Markovian rewards.\par 
\textit{Our solution:} To make this problem Markovian we augment the state-space to include \(Y_t = g(s_t) + \sum_{t' < t}\gamma^{t '}g(S_{t'})\), \(t\geq 0\). The resulting triplet of processes \((Y_t,S_t, A_t)\) is Markovian. We develop a model containing a constraint optimisation on three levels through an iterative update of the Q-function \(\hat{Q}(\cdot,\cdot)\), the environment-dynamics model \(\hat{\mathcal{P}}\), and the model of the cumulative continuation gain \(\hat{Y}_t\), allowing us to incorporate a non-Markovian component into the Q-function approximation (Section \ref{sec:DO-IQS}). \par 

\hypertarget{insight:cemetery_states}{\textbf{Insight 4: Boundary conditions}}. The problem of learning a stopping gain function \(G(\cdot)\) is similar to learning the reward for absorbing states as in \cite{alhafez2023lsiq}. Setting the value of the absorbing states to zero results in termination bias \cite{kostrikov2018discriminatoractorcritic}, when the agent might be more prone to stopping immediately, or never stopping at all, depending on the sign of the continuation reward.\par
\textit{Our solution:}
In OS this problem is naturally addressed by the stopping gain function \(G(\cdot)\), and the action ``stop" sends the agent to the set of cemetery states, where the value is zero, \(V^\pi(s_\Delta)=0\). Hence, the function \(G(\cdot)\) can be thought of as a reward, accumulated during the time spent in the absorbing state. This approach allows us to avoid setting a fixed target as is done in \cite{alhafez2023lsiq}, while still ensuring that there is no termination bias through well-defined boundary conditions.\par 

\section{Algorithm}\label{sec:algorithm}

% \vskip 5pt
  
\subsection{Offline IQ-learning for IOS}\label{sec:offline_iq_learn_for_ios}
There are two ways to estimate the OS region offline. The first approach uses pure IQ-learning where only the current state \(s\) is required: \(\hat{D}_{\epsilon}=\{s|\hat{G}(s)\geq \hat{V}(s)\}\), where the dynamics is learnt explicitly through the Q-function \cite{garg2021}. A more precise estimate for the world dynamics would allow one to reduce the Bellman error of the IQ-Learning algorithm and correspondingly to increase its stability.\par
To address this problem, we present an approach inspired by bi-level constrained optimisation. First, the Q-model is updated in a conservative way to fit the current beliefs about the world dynamics, and then the estimate of the environment dynamics is updated to fit the current model of the Q-function. This allows minimising the covariance shift in the learnt world model, and improving the sharpness of the reward estimates in case of imperfect demonstrations (see, for example, \cite{xu22l,liu2023adv} and \cite{zeng2023when}). Learning an environment model is crucial for estimating stopping region offline from demonstrations as in (\ref{eq:epsilon_stopping_region}). Often a crude estimate of \(\hat{\mathcal{P}}\) that is aware of the value function can be used instead of more complex approaches such as MLE \cite{farahmand2017}. This requires a dynamics loss function \(Loss_{\mathcal{P}}\) that minimises the Bellman error \cite{fujimoto2022i}. Although the original IRL objective in case of \(\chi^2\)-regularisation minimises a squared soft-Bellman error \cite{alhafez2023lsiq}, it still relies on the expert state transitions. Looking at the Bellman error definition: \(\delta_Q(s,a) = Q(s,a) - r(s,a)+\mathbb{E}_{s'\sim\mathcal{P}}[\gamma V(s')]\) we could notice that replacing the true environment dynamics \(\mathcal{P}\) by an estimate \(\hat{\mathcal{P}}\) increases \(\delta_Q\) by \(\gamma\sum_{s'\in\mathcal{S}}\left(\hat{\mathcal{P}}(s'|s,a)-\mathcal{P}(s'|s,a)\right)V(s')\). Although the choice of a particular \(Loss_\mathcal{P}
\) will change the degree of the Bellman error, updating it in a bi-level constrained optimisation way together with \(\hat{Q}\) allows us to incorporate information about the value function into the dynamics model.\par
% To recover the OS region we apply the OS rule (\ref{eq:stopping_region}) to the Q-function estimates: \(D_{\epsilon,\theta}= \{s|G_{\epsilon,\theta}(s)\geq V_{\theta}^{\pi_\epsilon}(s)\}= \{s|Q_{\epsilon,\theta}(s,a=0)\geq Q_{\epsilon,\theta}(s,a=1)\}\).
% \vskip -15pt
\subsection{Oversampling approach to data sparsity in IOS}\label{sec:oversampling_approach}
Sparse and small datasets often pose a challenge for neural networks functional approximation. A conventional way to tackle this problem is to employ various oversampling strategies by simulating artificial stopping and/or continuation points for dataset enrichment. Here we make an assumption about the geometrical properties of the stopping and continuation regions, i.e. the proximity of the states from the respective regions to each other. We propose to view IOS with oversampling as a problem of learning from imperfect demonstrations \cite{wu2019,ning2022}, where we could assign a certain confidence score to each of the observations based on the probability that these points come from the stopping region set. The expert stopping points then will have the highest confidence score, whereas artificial observations are assigned a confidence score. Then the expert occupancy measure becomes a mixture of optimal and non-optimal measures \(\rho_{E_{CS}}(s,a) = \rho_E(s,a) + \alpha(s,a)\rho_{fake}(s,a)\), where \(\alpha(s,a)\in \left[0,1\right]\) is the confidence score and \(\alpha(s,a)=1\) for \(\forall (s,a)\in O_E \). The objective function can then be expressed as \(\mathcal{J}_{CS}(Q) = \mathcal{J}(Q) + \mathcal{J}_{fake}(Q)\), where
\begin{equation}\label{eq:iq_learn_optimisation_fake}
    \begin{aligned}
        \mathcal{J}_{fake}(Q)
        {}&=\mathbb{E}_{\rho_{fake}}\left[\alpha(s,a)\phi\left(Q(s,a)\right.\right.
        \left.\left. -\gamma\mathbb{E}_{s'\sim\mathcal{P}(\cdot|s,a)}\sigma_{a'}^\epsilon \left(Q(s',a')\right)\right)\right]\\
        {}&-(1-\gamma)\mathbb{E}_{\rho_{fake}}\left[\alpha(s,a)\left(V^{\pi}(s) - \gamma V^{\pi}(s')\right)\right].
    \end{aligned}
\end{equation}
(Noting that the second term of (\ref{eq:iq_learn_optimisation}) in an offline learning case can be sufficiently approximated through bootstrapping expert demonstrations with \(\mathbb{E}_{(s,a,s')\sim expert}\left[V^{\pi}(s) - \gamma V^{\pi}(s')\right]\) \cite{garg2021}). 
In this work, we use an SMOTE oversampling technique \cite{Chawla-2002}, and its confidence-score-based modification, which we call CS-SMOTE. We note that the SMOTE oversampling approach has the property of producing artificial minority samples whose expected value converges to the expected value of true minority samples as the sample size
tends to infinity \cite{Kamalov2024AsymptoticBO}.
\begin{algorithm}[!t]
    \centering
    \caption{DO-IQS}
    \label{alg:doiqs}
    \begin{algorithmic}[1]
        \Require{Stopped expert observation history \(O_{T_m}^M=\{s_{T_m},a_{T_m}\}^{m=0:M-1}\), \(\forall 0\leq m\leq M-1\); a  discount factor \(\gamma\in (0,1]\); a number of training epochs \(E>0\); a batch size \(B>0\); a \(softmax\) scaling parameter \(\epsilon>0\); a zero-valued cemetery state \(s_\Delta=\boldsymbol{0}\) }
\Ensure{Approximate Q-function \(\hat{Q}\), continuation gain \(\hat{g}\), and the environment model \(\mathcal{\hat{P}}\)}
\State Initialise a Q-function \(Q_{\epsilon,\theta}\), an environment dynamics model \(\mathcal{P_\theta}\), and a continuation gain function \(g_\phi\)
\State Pre-process \(O_{T_m}^M\) using Algorithm \ref{alg:iqs-preprocess}
\For{epoch \(e\) from \(0\) to \(E-1\)}:
    \State Sample a batch of triples \((s_l,a_l,s'_l)\) of size \(B\)  and their corresponding states history \(h_t^m 
 = \{s_0^m,s_1^m,...s_t^m\}\), where \(s_t^m\) corresponds to \(s_l\) in the stacked process
    \State Do a forward pass for \(g_\phi(s_k^m) \forall s_k^m\in h_{t+1}^m\) and estimate \(y_{l} = \sum_{k=0}^t\gamma^k g_\phi(s_k^m)\), \(y'_l = \sum_{k=0}^{t+1}\gamma^k g_\phi(s_k^m)\) and \(y^{-}_l = \sum_{k=0}^{t-1}\gamma^k g_\phi(s_k^m)\) 
    \State Compose a batch of augmented states \((\tilde{s}_l,a_l,\tilde{s}'_l)\), where \(\tilde{s}_l=(s_l,y_l)\) and \(\tilde{s}'_l=(s'_l,y'_l)\)
    \State (DO-IQS-LB only) Bootstrap augmented states belonging to the stopping region in the batch and add them to the original batch.
    \State For a modified batch \((\tilde{s}_l,a_l,\tilde{s}'_l)\) do a forward pass for \(\mathcal{P}_\theta\) and for \(Q_{\epsilon,\theta}\)
    \State Estimate \(g(\tilde{s}_l) = Q_{\epsilon,\theta}(\tilde{s}_l,a=1)-\gamma\sum_{\tilde{s}_l'\in\tilde{\mathcal{S}}}\mathcal{P}_\theta(\tilde{s}_l'|\tilde{s}_l,a=1) V_{\epsilon,\theta}(\tilde{s}_l') - y_l^{-}\)
    \State Update \(\mathcal{P}_\theta\) using \(Loss_\mathcal{P}\), \(Q_{\epsilon,\theta}\) using (\ref{eq:iq_learn_optimisation}), and \(g_\phi\) using \(Loss_g\)
\EndFor
\end{algorithmic}
\end{algorithm}
% \vskip 5pt
% \end{figure}
% \end{minipage}

% \end{wrapfigure}

\subsection{DO-IQS}\label{sec:DO-IQS}
% \vskip -15pt
In higher dimensions, it might no longer be possible to perform oversampling (or the approximation error for the Q-function becomes significantly bigger). Moreover, many real-world problems have non-Markovian dynamics. We solve this problem by adding an additional process \(Y\) as described in Section \ref{sec:inverse_os}. This requires estimation of the continuation function \(g\) in parallel with the main algorithm, since it now becomes a part of the input into \(Q_{\epsilon,\theta}\) and \(\mathcal{P}_\theta\) (Algorithm \ref{alg:doiqs}) in a form of an augmented state \(\tilde{s}=(s,y); \tilde{s}\in \tilde{S}=\mathcal{S}\times\mathcal{Y}\), where \(y = \sum_{k=0}^{t}g_\phi(s_k)\) and \(y' = y+g_\phi(s_{t+1})\) To avoid updating \(g_\phi\) for all the states in the history, we update it only for the current observation in the batch by minimising \(Loss_g\) (e.g. \(L_2\) loss) between \(g\) recovered through IQ-Learning algorithm \(g(\tilde{s}) = Q_{\epsilon,\theta}(\tilde{s},a=1)-\gamma \sum_{\tilde{s}'\in\mathcal{\tilde{S}}}\mathcal{P}_\theta(\tilde{s}'|s,a=1) V_{\epsilon,\theta}(\tilde{s}')-y^{-}\), where \(y^{-} = \sum_{k=0}^{t-1}g_\phi(s_k)\), and the approximate cumulative gain function \(g_\phi\).  We note that due to the sparsity of stopping decisions, the DO-IQS model still suffers from low convergence. We propose applying local (batch-wise) bootstrapping of the augmented states \(\tilde{s}\) belonging to the stopping set (line 7 in Algorithm \ref{alg:doiqs}).
% \vskip -15pt
\subsection{Quality evaluation}\label{sec:quality_evaluation}
% \vskip -15pt
In the IRL literature, the quality of the recovered Q-function / policy is often assessed by measuring their performance in terms of the true accumulated reward. In practice, querying to the environment is often not available. Hence, a metric to assess the performance of the IRL algorithm is needed, which could also be used during model tuning. We propose to view the IOS as a classification problem, labelling the states as belonging to the stopping or continuation region. This requires a choice of goal-specific metrics suitable for imbalanced data (see Figure \ref{fig:sparse_stopping_region}, right). In case of applying OS to critical events prevention, one could consider a trade-off between a mean time-to-event (m-TTE) and a mean event-miss rate (m-EMR). To represent this trade-off one might use one of the classification metrics which account for imbalanced datasets while not favouring neither True Negatives (continuation actions in case of OS) nor True Positives (stopping actions). We propose using the Balanced Accuracy score as a metric of choice due to its ability to treat positive and negative observations in an equal way.

\setlength{\textfloatsep}{0pt}
\vskip 5pt
\begin{figure*}[!b]
    % \captionstyle{centerlast}
    % \vspace{-5pt}
    \begin{minipage}[l]{0.25\linewidth}
        \centering
        \includegraphics[trim={0 0 0 2.7cm}, clip, width=1\linewidth]{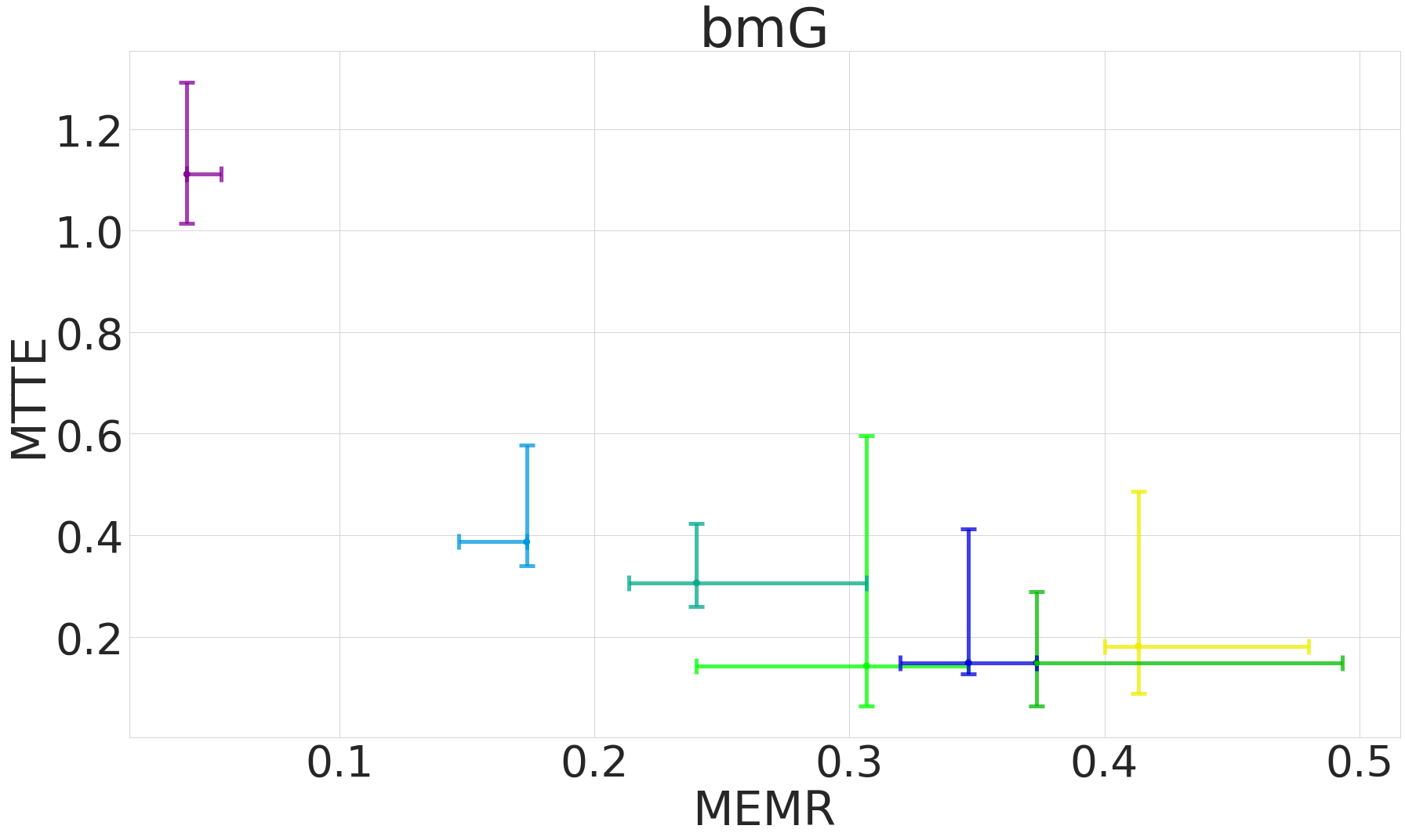}
        \captionsetup{labelformat=empty}
    \end{minipage}\hfill
    \begin{minipage}[l]{0.25\linewidth}
        \centering
        \includegraphics[trim={0 0 0 2.7cm}, clip, width=1\linewidth]{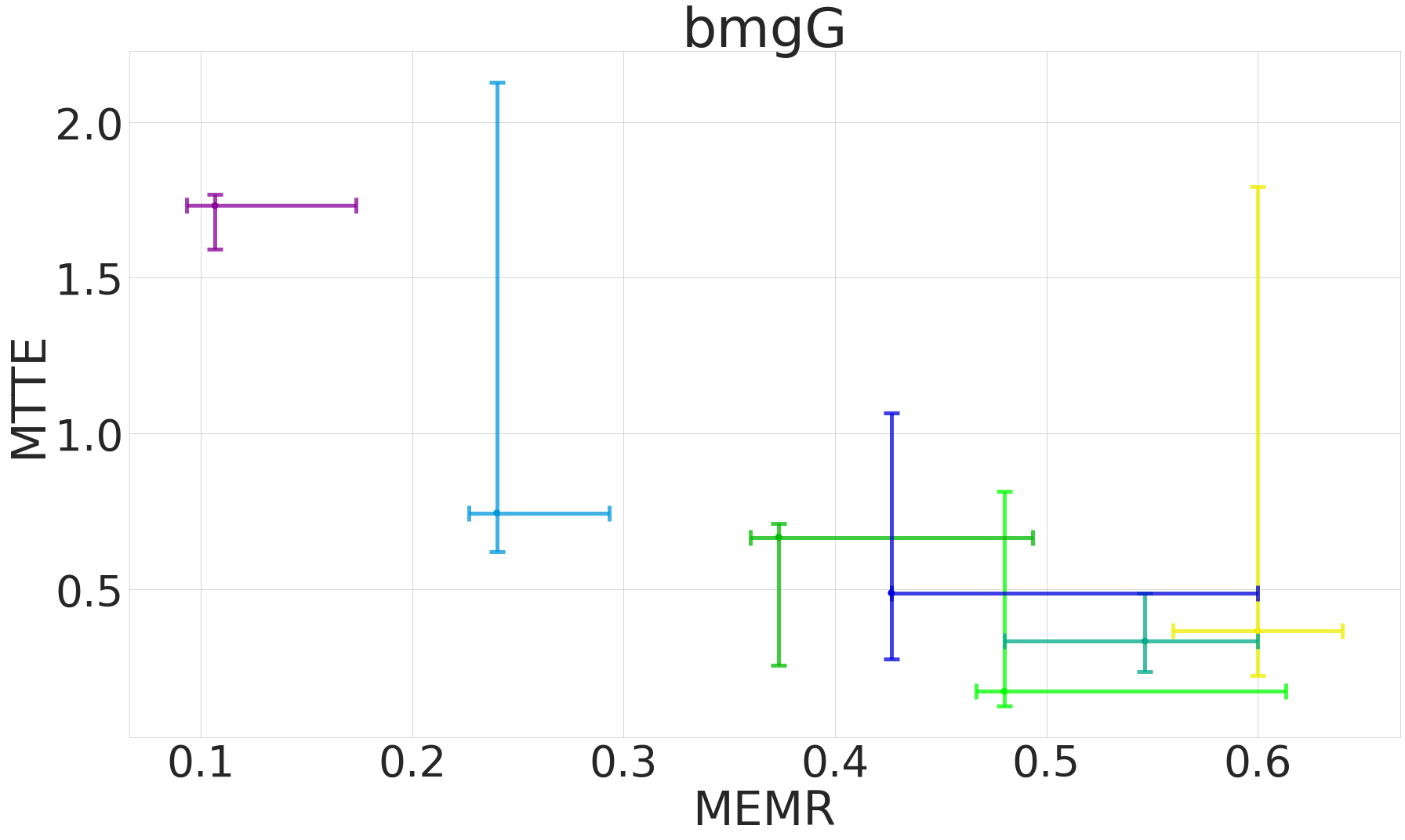}
        \captionsetup{labelformat=empty}
    \end{minipage}\hfill
    \begin{minipage}[l]{0.25\linewidth}
        \centering
        \includegraphics[trim={0 0 0 2.7cm}, clip, width=1\linewidth]{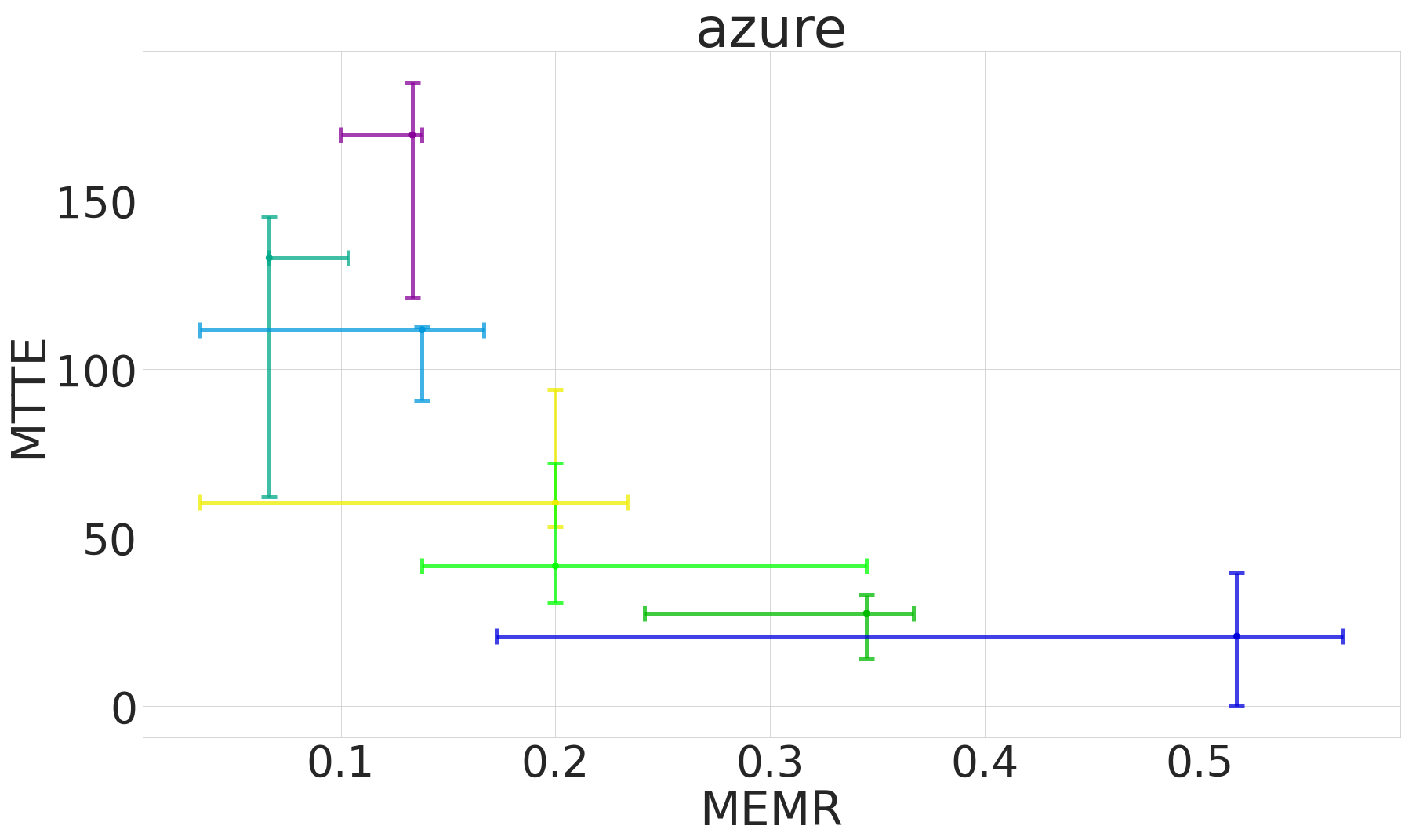}
        \captionsetup{labelformat=empty}
    \end{minipage}\hfill
    \begin{minipage}[l]{0.25\linewidth}
        \centering
        \includegraphics[trim={0 0 0 2.7cm}, clip, width=1\linewidth]{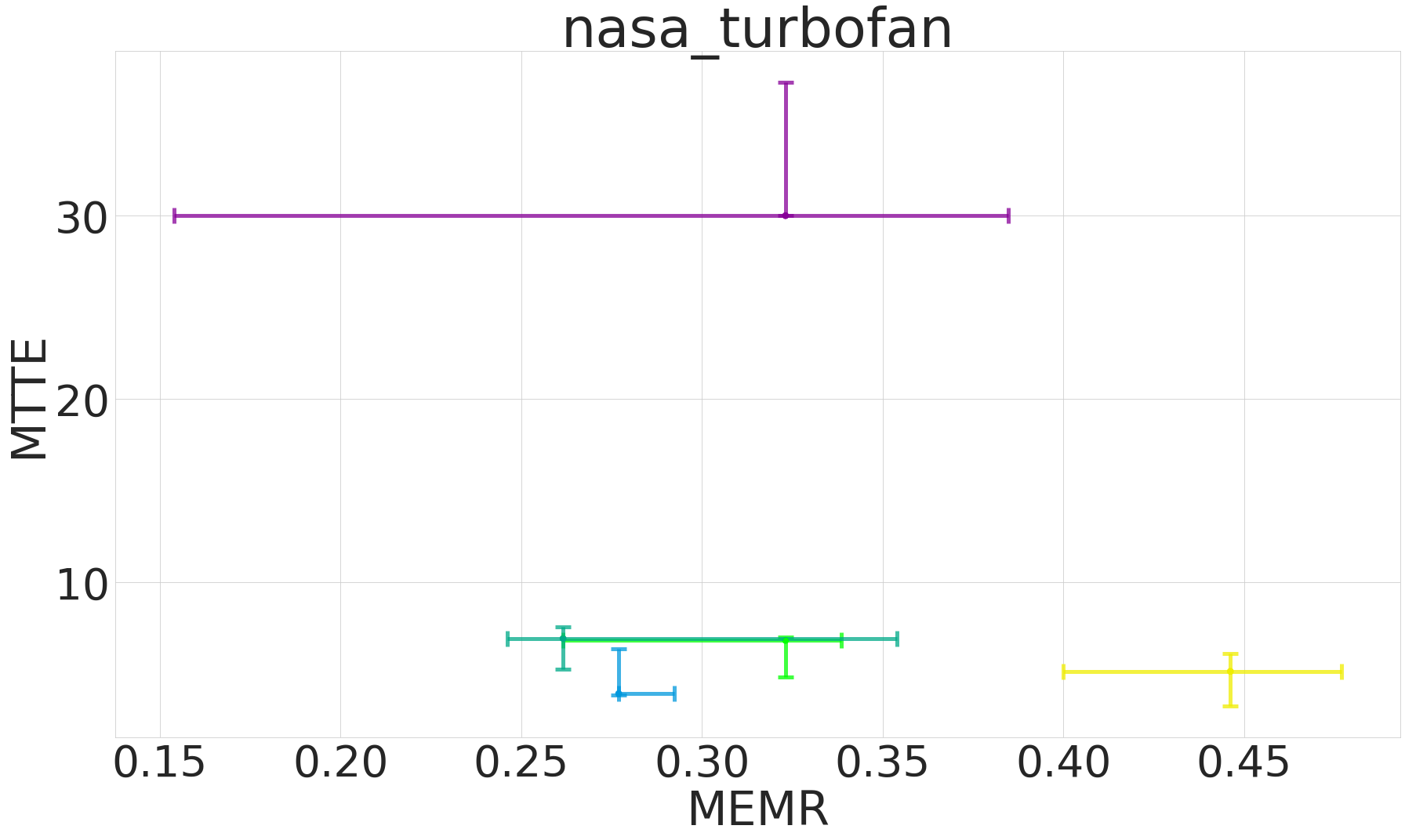}
        \captionsetup{labelformat=empty}
    \end{minipage}\hfill
    
    % \captionstyle{centerlast}
    \begin{minipage}[l]{0.25\linewidth}
        \centering
        \includegraphics[trim={0 13cm 0 0.6cm}, clip, width=1\linewidth]{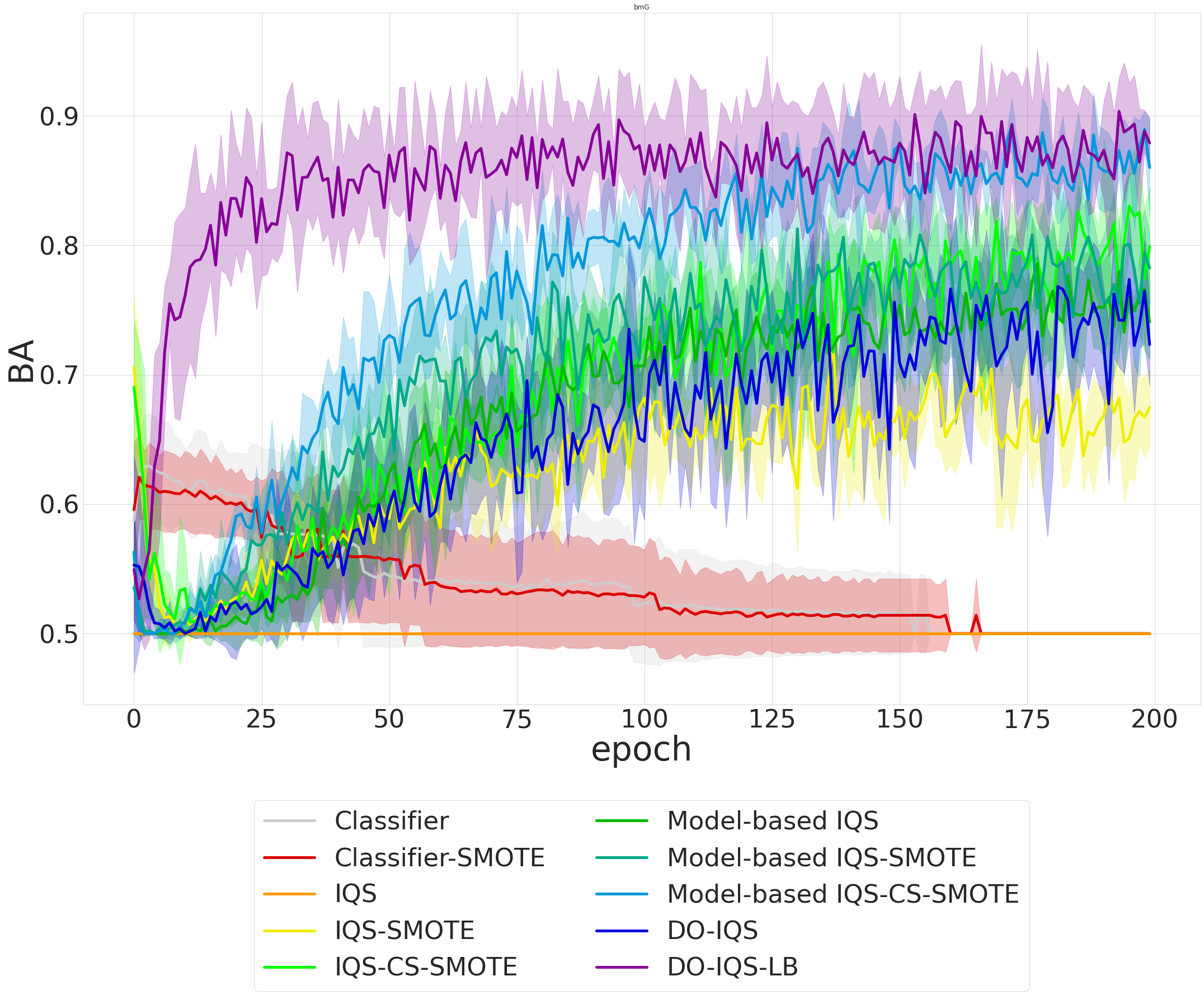}
        \captionsetup{labelformat=empty}
        \caption{bmG}
    \end{minipage}\hfill
    \begin{minipage}[l]{0.25\linewidth}
        \centering
        \includegraphics[trim={0 13cm 0 0.6cm}, clip, width=1\linewidth]{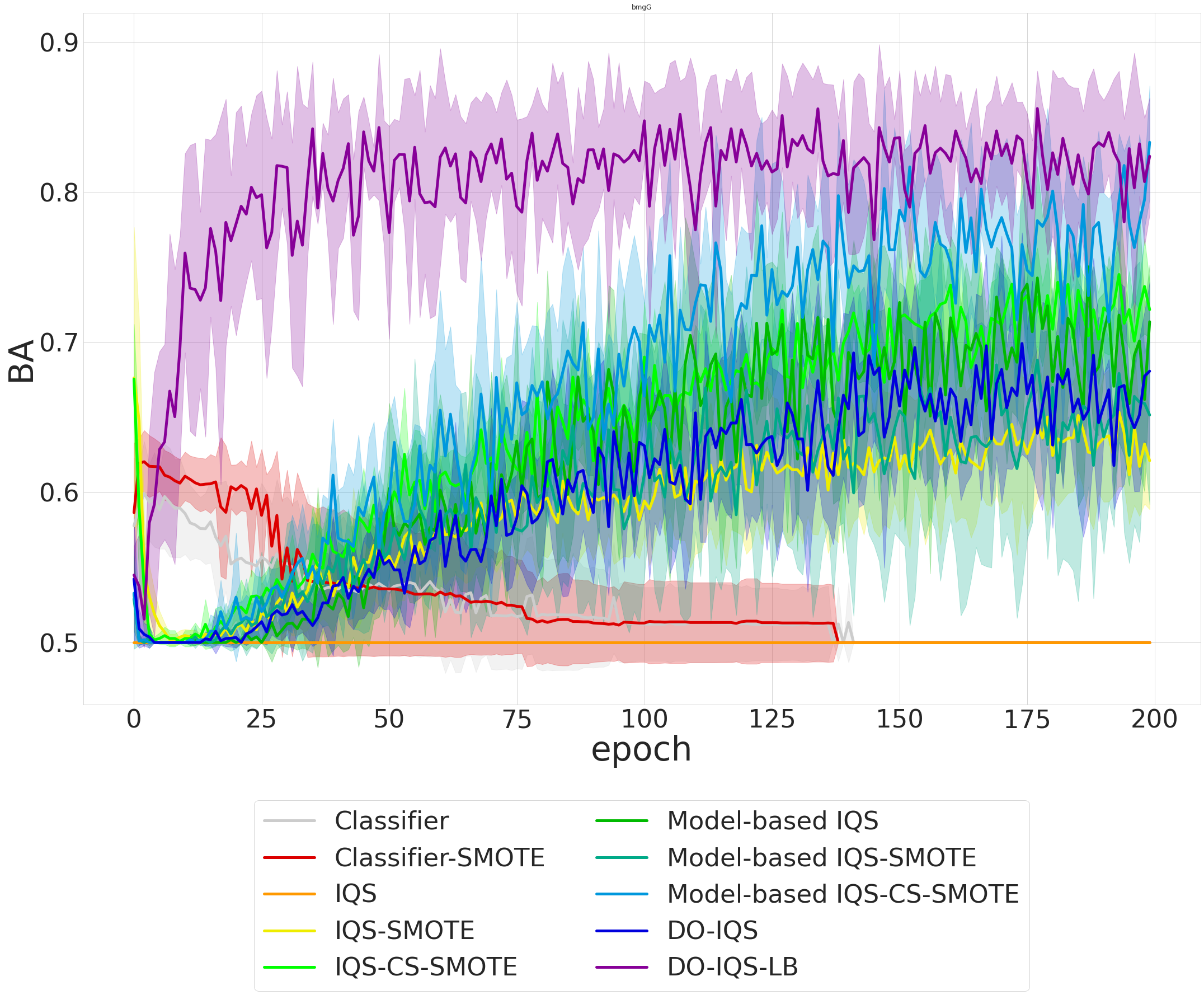}
        \captionsetup{labelformat=empty}
        \caption{bmgG}
    \end{minipage}\hfill
    \begin{minipage}[l]{0.25\linewidth}
        \centering
        \includegraphics[trim={0 13cm 0 0.6cm}, clip, width=1\linewidth]{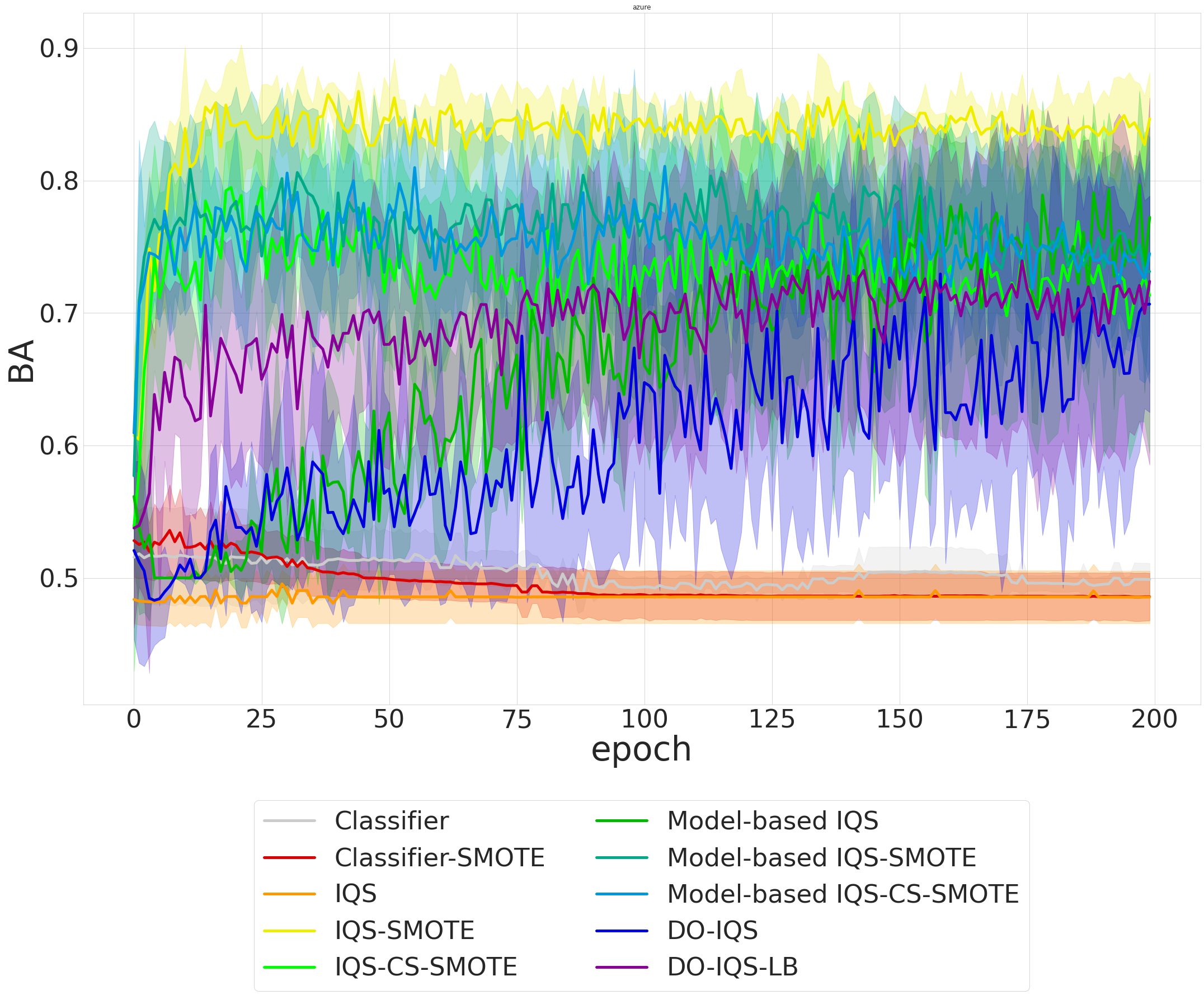}
        \captionsetup{labelformat=empty}
        \caption{Azure}
    \end{minipage}\hfill
    \begin{minipage}[l]{0.25\linewidth}
        \centering
        \includegraphics[trim={0 13cm 0 0.6cm}, clip, width=1\linewidth]{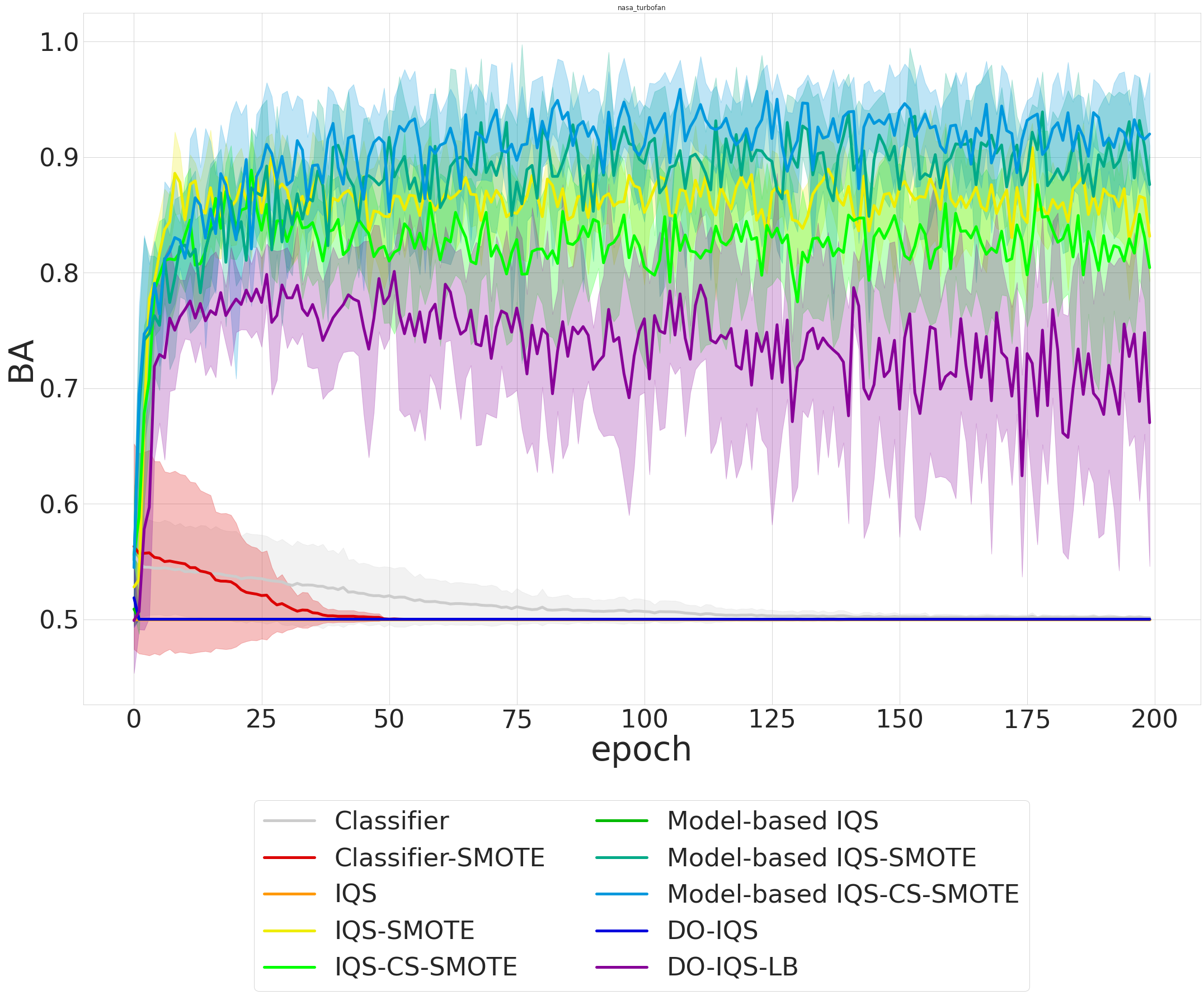}
        \captionsetup{labelformat=empty}
        \caption{Turbofan}
    \end{minipage}\hfill
    \begin{minipage}[l]{1\linewidth}
        \centering
        \includegraphics[trim={0 0 0 0}, clip, width=1\linewidth]{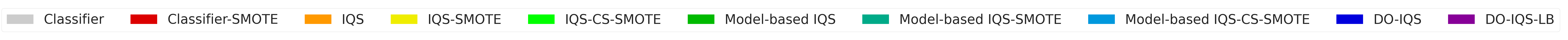}
        \captionsetup{labelformat=empty}
        \caption{}
    \end{minipage}
    \vspace{-20pt}
    \setcounter{figure}{2}
    \caption{\textbf{Top}: Median m-EMR to m-TTE trade-off (median values with 25-75 IQR error bars). \textbf{Bottom}: Average balanced accuracy (solid lines) with a standard deviation (shaded area) over 200 training epoches}
    \label{fig:results}
% \vskip -15pt 

\end{figure*}

\section{Experiments}\label{sec:experiments}
% \subsection{Environments Description and Results}\label{sec:environments}
\textbf{2D Brownian motion with sparse data}\label{sec:2dbm}. We simulate a 2D Brownian motion \(\mathbb{P}\left(s_1'|s_0=s\right) = N\left(\begin{matrix}s'\left[0\right] \\ s'\left[1\right]\end{matrix}|0,\sqrt{\Delta}\begin{pmatrix}1 & 0 \\ 0 & 1\end{pmatrix}\right)ds'\), where \(\Delta t=\frac{1}{50}\), and \(s[i]\) is and the {i}-th element of \(s\). The stopping and continuation gains are set to \(G(s)=s[0]^2+s[1]^2\) and \(g(s)=\left(5\times\mathds{1}_{|s|<1}-400\times\mathds{1}_{|s|\geq 1}\right)\times \Delta t\), where \(|s| = \sqrt{s[0]^2 + s[1]^2}\). We consider two examples. In the first, expert stopping decisions are simulated using the stopping gain function \(G\) only. In the second example, both the stopping gain \(G\) and the cumulative continuation gain \(g\) forming the \(y\) process are used. We simulate 175 paths and hold out \(30\% \) of those for validation and model selection. We also simulate 75 paths for testing. For each simulated path, the expert optimal stopping times are found using a backward induction algorithm for OS. \par 

\textbf{Optimal time to intervention via IOS with real data }\label{sec:oti}. OTI problem naturally translates into the OS problem over a hazard rate process \cite{2022oti}. The (empirical) hazard rate is defined by minimising both the mean residual time to event (m-TTE) and the mean event miss rate (m-EMR) using survival analysis. This approach requires manual tuning of the desired risk-tolerance. Instead, we approximate the risk function using IOS by setting the expert stopping time to be at the time (or as close as required) to the hazardous event. We tested our approach on Azure Predictive Maintenance Guide Data \cite{azuremodel2017} and NASA Turbofan Engine Failure Data \cite{Saxena2008} to show how IOS can be used to solve the corresponding OTI problem. The datasets were down-sampled to every 10 hours. \par
\textbf{Results}. We compare the convergence results of the models presented as well as a binary classifier (with the same ANN structure as the IQ-Learn algorithm used in this paper) with and without SMOTE oversampling as a baseline. Figure \ref{fig:results} summarises the results of the simulation in terms of the m-EMR and m-TTE trade-off, as well as the convergence of the algorithms in terms of balanced accuracy over 200 training epochs. We also report the mean and errors of balanced accuracy in Table \ref{tbl: ba_results}. DO-IQS-LB outperforms all the other algorithms for the 2D BM examples, providing a clear advantage in the case of the example with continuation gain.  For the OTI example, the best performance was achieved by Model-based IQS and Model-based IQS-SMOTE. DO-IQS-LB achieves average performance on these datasets with a slower convergence, while still outperforming the baseline. Our approach achieves similar or better results compared to the results reported in \cite{2022oti}. Further ablation studies are presented in the appendix. \par 
% \begin{figure}
\begin{table}[!t]
\centering
\caption{Balanced accuracy (mean \(\pm\) two standard deviations) with top-3 values in bold}
\vskip 5pt
\begin{tabular}{lcccc}
\hline
                          & bmG         & bmgG        & azure       & turbofan   \\
\hline
 Classifier               & 0.6164±0.05 & 0.5936±0.05 & 0.535±0.09  & 0.5479±0.09 \\
 Classifier-SMOTE         & 0.6175±0.06 & 0.5924±0.04 & 0.5397±0.09 & 0.5616±0.18 \\
 IQS                      & 0.5000±0.00 & 0.5000±0.00 & 0.4917±0.04 & 0.5000±0.00 \\
 IQS-SMOTE                & 0.7610±0.07 & 0.6651±0.05 & \textbf{0.8141±0.05} & 0.8685±0.08 \\
 IQS-CS-SMOTE             & 0.8438±0.07 & 0.7433±0.10 & \textbf{0.8163±0.12} & \textbf{0.9136±0.07}  \\
 Model-based IQS          & 0.7794±0.07 & \textbf{0.7663±0.10} & \textbf{0.8245±0.13} & 0.5000±0.00 \\
 Model-based IQS-SMOTE    & \textbf{0.8586±0.06} & 0.7337±0.09 & 0.7960±0.09 & \textbf{0.9185±0.12}  \\
 Model-based IQS-CS-SMOTE & \textbf{0.8926±0.08} & \textbf{0.8279±0.05} & 0.7783±0.07 & \textbf{0.9055±0.06} \\
 DO-IQS                   & 0.8257±0.04 & 0.7355±0.09 & 0.7681±0.27 & 0.5000±0.00  \\
 DO-IQS-LB                & \textbf{0.9164±0.02} & \textbf{0.8552±0.05} & 0.7391±0.15 & 0.8250±0.10 \\
\hline
\end{tabular}
\label{tbl: ba_results}
\vskip 2pt
\end{table}
% \end{figure}
% \subsection{Computational Setup, Training Times and Runtime Complexity}\label{sec:compute}
We run the experiments on a VM with 2 vCPUs, 4.0 GiB. Each model is trained for 250 epoches with a batch size of 128 and the best model is selected based on the validation set balanced accuracy score. Training each model took an average of 6 min (12.5 min for DO-IQS and DO-IQS-LB), and the inference step took 0.003 sec/observation. See the appendix for computational complexity analysis and additional experiments.
% We report the median values of the scores with the 25-75 IQR error bars over five runs with different random seeds. The results for the basic IQ-Learn algorithm are not presented due to their poor performance.\par 
% The approximate complexity (one forward pass) of the basic IQ-Learn algorithm is \( O(n_s \cdot n + n_l \cdot n^3 + 2n^2 + n_l \cdot n + n + 2)\) and \(O(4n_s\cdot n + 5n + 2n^4+2n^2)\) for the DO-IQS model (where \(n_s\) is the number of features in the input, \(n\) is the dimension of the ANN layers assumed to be equal across layers, and \(n_l\) is the number of layers in the network; we also assume equal structure of the \(g_\phi\) and \(Q_\theta\) networks, taking the sample size 1). See the supplementary materials for details. 

\section{Discussion and outlook}\label{sec:discussion_and_outlook}
In this paper we develop a set of methods to solve the Inverse Optimal Stopping problem by coupling the original problem to a related RL problem in a way, allowing to apply SOTA methods for IRL. While the conceptual IOS challenges are addressed, some more advances methods for confidence scores evaluations (e.g. semi-supervised classification, assigning confidence scores to expert demonstrations, etc.), environment dynamics estimation (e.g. Bayesian methods, GANs, etc.) and a more efficient way to approximate continuation function are needed. We note that making use of the binary structure of the OS problem and applying (``naive") oversampling techniques substantially improved the performance of the baseline IQ-Learning algorithm even in a more complex case of OTI problems on real data. It is possible to use more principled methods of oversampling with corrections for potential non-convex shape of the stopping region and allowing oversampling of diverse types of data (e.g., generative adversarial minority oversampling \cite{mullick2020}). Although the choice of the IQ-Learning algorithm as a basis for the approach was mostly motivated by the easiness of Q-function recovery in an offline setup, it would be an interesting research direction to compare the performance and stability of other IRL algorithms coupled with IOS.

\bibliography{export}
\newpage 
\appendix
\section{Appendix A: Propositions and proofs}\label{appendix:propositions_and_proofs}
In this section, we present supplementary definitions, propositions, and proofs.
\subsection{Proof of Proposition 1}

\begin{proof}
To show this we define the first hitting time of the stopping region \(D_{\Delta}\) as 
\begin{equation}\label{eq:opthittingtime}
  \tau_{D_{\Delta}} = \inf\{t\geq 0|s_t\in D_{\Delta}\}. 
\end{equation}
with the stopping and continuation regions associated to the original OS problem as follows:\par
\begin{equation}
    \begin{aligned}
        {}&C = \{s_t | G(s_t) \leq V^{\star}(s_t)\},\\
        {}&D_{\Delta} = S\setminus C \cup \{\Delta\},
    \end{aligned}
\end{equation}
and \(D_{\Delta}\) is the stopping region augmented with the set of ``cemetery'' states.\par
We first can show that the first hitting time of the stopping region \(D_{\Delta}\) defined with (\ref{eq:opthittingtime}) above is also an optimal stopping time using the Debut theorem.\par

\begin{theorem}[Debut Theorem]
Let X be an adapted right-continuous stochastic process defined on a filtered probability space. If D is a stopping set which is Borel-measurable, then:\par
\begin{equation}
\tau_{D} = \inf\{t\geq 0| X_t\in D \}
\end{equation}
is a stopping time.
\end{theorem}

Here we assume that the MC evolves with jumps (allowing us to formulate it as a discrete MDP). So if \(\tau^{\star} = \tau_{D_{\Delta}}\) is the first hitting time of \(D_{\Delta}\), we assume that there exists a hard (instantaneous) killing time \(
\tau_{\Delta} = \tau_{D_{\Delta}}+1\) at which the process enters the set of cemetery states, i.e. \(S_t\in\{\Delta\},\forall t\geq \tau_{\Delta}\).

Hence the OS problem can be expressed as
\begin{equation*}\label{equation:opt_st_rl_explicit}
\begin{aligned}
    {}&V^{\star}(s_0) =  \mathbb{E}_{s_0}\left[\sum_{t=0}^{\tau_{D_{\Delta}} - 1}\gamma^t g(s_t)\mathds{1}_{\{\tau_{D_{\Delta}} \geq 1\}} + \gamma^{\tau_{D_{\Delta}}} G(s_{\tau_{D_{\Delta}}}) \right].
\end{aligned}
\end{equation*}

We now assume that \(s_0\in C\) and \(\tilde{s}_0=s_0 \in C\) with \(C=S\setminus D\). We also define an optimal policy as
    \begin{equation}
        \pi^{\star} (\cdot|s_t) =
        \begin{cases}
          \delta_1(\cdot):s_t\in C \\
          \delta_0(\cdot): s_t\in D_{\Delta}\\
        \end{cases}\
\end{equation}
and express the sum in (\ref{equation:opt_st_rl_explicit}) as:
\begin{equation*}
    \begin{aligned}
        {}&V^{\pi^\star}(s_0) = \mathbb{E}_{s_0}^{\pi^{\star}}\left[\sum_{t=0}^{\tau_{D_{\Delta}} - 1}\gamma^t r(s_t, a_t)\mathds{1}_{\{\tau_{D_{\Delta}} \geq 1\}} + \gamma^{\tau_{D_{\Delta}}}r(s_{\tau_{D_{\Delta}}}, a_{\tau_{D_{\Delta}}}) + \sum_{t= \tau_{D_{\Delta}}+1}^{\infty} \gamma^t r(s_t,a_t) \right]\\
        {}&=\mathbb{E}_{s_0}^{\pi^{\star}}\left[\sum_{t=0}^{\tau_{D_{\Delta}} - 1}\gamma^t r(s_t, a_t)\mathds{1}_{\{\tau_{D_{\Delta}} \geq 1\}}\right]  + \mathbb{E}_{s_0}^{\pi^{\star}}\left[\gamma^{\tau_{D_{\Delta}}}r(s_{\tau_{D_{\Delta}}}, a_{\tau_{D_{\Delta}}})\right] + \mathbb{E}_{s_0}^{\pi^{\star}}\left[\sum_{t= \tau_{D_{\Delta}}+1}^{\infty} \gamma^t r(s_t,a_t) \right].
    \end{aligned}
\end{equation*}

We note that 
\begin{equation*}
    \begin{aligned}
        {}&\mathbb{E}_{s_0}^{\pi^{\star}}\left[r(s_t,a_t)\right]\\
        {}&= p_0(s_0)\sum_{a_0}\pi^{\star}(a_0|s_0) \sum_{s_1}\mathbb{P}_{s_0,a_0}(s_1) \sum_{a_1}\pi^{\star}(a_1|s_1) \sum_{s_2}\mathbb{P}_{s_1,a_1}(s_2)\\
        {}&... \sum_{s_t}\mathbb{P}_{s_{t-1},a_{t-1}}(s_t) \sum_{a_t}\pi^{\star}(a_t|s_t)r(s_t,a_t)\\
        {}&= p_0(s_0)\sum_{a_0}\pi^{\star}(a_0|s_0)\left(\prod_{k=1}^t \sum_{s_{k}}\mathbb{P}_{s_{k-1},a_{k-1}}(s_{k}) \sum_{a_k}\pi^{\star}(a_k|s_k)\right) r(s_t,a_t).
    \end{aligned}
\end{equation*}
We now consider the following tree cases:
\textbf{1}: If \(t<\tau_{D_{\Delta}}\), then \(s_t \in C\) and \(\pi^{\star}(a_t|s_t)=\delta_1(a_t)\). Then \(\pi^{\star}(a_t=1|s_t)=1, \pi^{\star}(a_t=0|s_t)=0\) and \(\mathbb{E}_{s_0}^{\pi^{\star}}\left[r(s_t,a_t)\right]=\mathbb{E}_{s_0}^{\pi^{\star}}\left[g(s_t)\right]\), \(\forall t<\tau_{D_{\Delta}}\);
    
\textbf{2}: If \(t=\tau_{D_{\Delta}}\), then
    \begin{equation*}
        \begin{cases}
            s_{k}\in C, \pi^{\star}(a_k|s_k)=\delta_1(a_k) \text{ and } \pi^{\star}(a_k=1|s_k)=1, \pi^{\star}(a_k=0|s_k)=0 \forall k<\tau_{D_{\Delta}};\\
            s_t \in D_{\Delta}, \pi^{\star}(a_t|s_t)=\delta_0(a_t)\text{ and } \pi^{\star}(a_t=1|s_t)=0, \pi^{\star}(a_t=0|s_t)=1 for t=\tau_{D_{\Delta}}.
        \end{cases}
    \end{equation*}
    
    Then  \(\mathbb{E}_{s_0}^{\pi^{\star}}\left[r(s_t,a_t)\right]=\mathbb{E}_{s_0}^{\pi^{\star}}\left[G(s_t)\right]\), for \(t=\tau_{D_{\Delta}}\);
    
\textbf{3}: Likewise since \(t>\tau_{D_{\Delta}}\) is equivalent to \(t\geq \tau_{\Delta}\) we have that \(s_t\in\{\Delta\}\) and \(\pi^{\star}(a_t|s_t)=\delta_0(a_t)\). We also know that \(g(s_t) = G(s_t) = 0\) for \(s_t\in \{\Delta\}\). Then \(\pi^{\star}(a_t=1|s_t)=0, \pi^{\star}(a_t=0|s_t)=1\) and \(\mathbb{E}_{s_0}^{\pi^{\star}}\left[r(s_t,a_t)\right]=\mathbb{E}_{s_0}^{\pi^{\star}}\left[G(\Delta)\right] = 0\), for \(t>\tau_{D_{\Delta}}\).

Putting it all together we obtain that \(V^{\tau^{\star}}\) and \(V^{\pi^\star}\) are equivalent:
\begin{equation*}
    \begin{aligned}
        {}&{V}^{\pi^\star}(s_0) = \mathbb{E}_{s_0}\left[\sum_{t=0}^{\tau_{D_{\Delta}} - 1}\gamma^t g(s_t)\mathds{1}_{\{\tau_{D_{\Delta}} \geq 1\}} + \gamma^{\tau_{D_{\Delta}}}G(s_{\tau_{D_{\Delta}}}) + \sum_{t= \tau_{D_{\Delta}}+1}^{\infty} \gamma^t G(\Delta) \right]\\
        {}& = \mathbb{E}_{s_0}\left[\sum_{t=0}^{\tau_{D_{\Delta}} - 1}\gamma^t g(s_t)\mathds{1}_{\{\tau_{D_{\Delta}} \geq 1\}} + \gamma^{\tau_{D_{\Delta}}}G(s_{\tau_{D_{\Delta}}}) \right]\\
        {}&= V^{\star}(s_0).
    \end{aligned}
\end{equation*}
\end{proof}

\subsection{Proposition: deterministic stopping policy}
\begin{proposition}[Deterministic stopping policy]\label{proposition:deterministic_stopping_policy}
The optimal stopping policy defined by a set of Delta-functions conditioned on the continuation and stopping region is equivalent to a greedy policy with respect to the \(Q\)-function, i.e.
\begin{equation}
\begin{aligned}
    \pi^{\star}(a|s) = \begin{cases}
     \delta_1(a):s\in C^{\star} \\
     \delta_0(a):s\in D^{\star}_{\Delta}\\
    \end{cases} =\delta\left(a=\argmax_a Q^{\pi^{\star}}(s,a)\right).
\end{aligned}
\end{equation}
\end{proposition}

\begin{proof}
Assume that there exists an optimal deterministic greedy policy defined by
    \begin{equation*}
        \delta\left(a=\argmax_a Q^{\pi^{\star}}(s,a)\right).
    \end{equation*}
In case of a binary action-space we can write the above in terms of two Delta-functions
\begin{equation*}
    \begin{aligned}
    {}&\pi^{\star}(a|s)=
    \begin{cases}
        \delta(a=0 \text{ if } Q^{\pi^{\star}}(s,0)=\max_a Q^{\pi^{\star}}(s,a))\\
        \delta(a=1 \text{ if } Q^{\pi^{\star}}(s,1)=\max_a Q^{\pi^{\star}}(s,a))\\
    \end{cases}\\
    {}&=\begin{cases}
        \delta_0(a):s\in\{s|r(s,0)+\gamma V^{\star}(s') = V^{\star}(s)\}\\
        \delta_1(a):s\in\{s|r(s,1)+\gamma V^{\star}(s')=V^{\star}(s)\}\\
    \end{cases}\\
    {}&=\begin{cases}
        \delta_0(a):s\in\{s|G(s)= \max\{g(s)+\gamma V^{\star}(s'), G(s)\}\}\\
        \delta_1(a):s\in\{s|g(s)+\gamma V^{\star}(s')=\max\{g(s)+\gamma V^{\star}(s'), G(s)\}\}\\
    \end{cases}\\
    {}&=\begin{cases}
        \delta_0(a):s\in\{s|G(s)\leq g(s)+\gamma V^{\star}(s')\}\\
        \delta_1(a):s\in\{s|G(s)>g(s)+\gamma V^{\star}(s')\}\\
    \end{cases}\\
    {}&=\begin{cases}
        \delta_0(a):s\in\{s|G(s)\leq V^{\star}(s)\}\\
        \delta_1(a):s\in\{s|G(s)>V^{\star}(s)\}\\
    \end{cases}\\
    {}&= \begin{cases}
     \delta_1(a):s\in C^{\star} \\
     \delta_0(a):s\in D^{\star}_{\Delta},\\
    \end{cases}
\end{aligned}
\end{equation*}
where we used the definitions of the \(Q\)- and \(V\)-functions and stopping and continuation regions and the fact that for a deterministic policy \(V^{\pi}(s)=Q^{\pi}(s,\pi(s))\).
\end{proof}
\subsection{Convergence of a stochastic policy}
\begin{proposition}[Convergence of the Boltzmann distribution to the Delta distribution]
 Define Boltzmann distribution as 
 \begin{equation*}
    s_{\epsilon}(f(x)) = \frac{\exp\{f(x)/\epsilon\}}{\sum_{x'\in\mathcal{X}}\exp\{f(x')/\epsilon\}}.
\end{equation*}
 In the limit 
 \begin{equation*}
    s_{\epsilon}(f(x))\rightarrow\delta(x=\argmax_{x'}f(x'))
 \end{equation*}
 as \(\epsilon\rightarrow 0\).
\end{proposition}
\begin{proof}
Assume that the function \(f(x)\) has a unique supremum and define \(f(x^{\star}) = \sup_{x\in X}f(x)\). We first divide \(s_{\epsilon}(f(x))\) by \(\frac{\exp\{-f(x^{\star})/\epsilon\}}{\exp\{-f(x^{\star})/\epsilon\}}\) to obtain

\begin{equation*}
    \lim_{\epsilon\rightarrow 0}s_{\epsilon}(f(x)) = \lim_{\epsilon\rightarrow 0}\frac{\exp\{f(x)/\epsilon\}}{\sum_{x'\in\mathcal{X}}\exp\{f(x')/\epsilon\}} \frac{\exp\{-f(x^{\star})/\epsilon\}}{\exp\{-f(x^{\star})/\epsilon\}} = \lim_{\epsilon\rightarrow 0}\frac{\exp\{\frac{f(x)-f(x^{\star})}{\epsilon}\}}{\sum_{x'\in\mathcal{X}}\exp\{\frac{f(x')-f(x^{\star})}{\epsilon}\}}.
\end{equation*}
Define \(\Delta(x) = f(x)-f(x^{\star}) <0\) if \(f(x)\neq f(x^{\star})\) and \(\Delta^{\star} (x)= f(x)-f(x^{\star})=0\) if \(f(x)=f(x^{\star})\). Then
\begin{equation*}
    \lim_{\epsilon\rightarrow 0} s_{\epsilon}(f(x))=
    \begin{cases}
        \lim_{\epsilon \rightarrow 0} \frac{\exp\{\Delta (x)/\epsilon \}}{\sum_{x'\neq \mathcal{X}}\exp\{\Delta(x')/\epsilon\}+1} \text{ if \(f(x)\neq f(x^{\star})\)}\\
        \lim_{\epsilon \rightarrow 0} \frac{1}{\sum_{x'\neq \mathcal{X}}\exp\{\Delta(x')/\epsilon\}+1} \text{ if \(f(x)=f(x^{\star})\)}.
    \end{cases}
\end{equation*}
Then
\begin{equation*}
\begin{aligned}
    {}&\lim_{\epsilon\rightarrow 0} s_{\epsilon}(f(x))=
    \begin{cases}
        \frac{0}{0+1} \text{ if \(f(x)\neq f(x^{\star})\)}\\
        \frac{1}{0+1} \text{ if \(f(x)=f(x^{\star})\)}
    \end{cases}
    =\begin{cases}
        0 \text{ if \(f(x)\neq f(x^{\star})\)}\\
        1 \text{ if \(f(x)=f(x^{\star})\)}
    \end{cases}\\
    {}&=\delta\left(x=\argmax_{x'}f(x')\right).
\end{aligned}
\end{equation*}
In other words, the softargmax function \(s_\epsilon(\cdot)\) converges to the argmax function \(\argmax(\cdot)\) as the scaling parameter goes to zero.
\end{proof}

% \subsection{Connection to POMDPs}
% We note a connection of the OS problem with the extra $Y$ variable to Partially Observable Markov Decision Processes (POMDP). To show that we present a Figure \ref{fig:pomdp}. Here the observable part of the process consists of the environment states \(s_t\), while the expert makes decision to stop or continue based on full observations containing a component \(y_t\) providing them with the full information on the underlying dynamics.
% \begin{figure}\label{fig:pomdp}
%     \centering
%     \includegraphics[width=1\linewidth]{images/POMDP.png}
%     \caption{OS with non-Markovian component as a POMDP problem}
%     \label{fig:pomdp}
% \end{figure}
\subsection{Stopping states and cemetery states}

For practical implementation, proper treatment of stopping and cemetery states is crucial. \par 
\begin{remark}
  The OS value function \(V^{\pi_{stopping}}\) with the reward \(r_{stopping}(s,a)=g(s)a+G(s)(1-a), \forall s\neq s_\Delta\) and \(r_{stopping}(s_\Delta)=0\) is equivalent to the value function for the MDP with the absorbing states \(V^{\pi_{absorbing}}\) (with the rewards \(r(s,a)\) and \(r(s_\Delta,\cdot) \neq 0\)).   
\end{remark}
To see that consider the following form of the value function for the problem with absorbing states:
\begin{equation*}
\begin{aligned}
    {}&V^{\pi_{absorbing}}(s) \\
    {}&=\mathbb{E}_s\left[\sum_{t=0}^T \gamma^t r(s_t,a_t) + \sum_{t=T+1}^\infty\gamma^t r(s_\Delta,\cdot)\right]\\
    {}&= \mathbb{E}_s\left[\sum_{t=0}^{T-1} \gamma^t r(s_t,a_t) +\gamma^Tr(s_T, a_T)+ \sum_{t=T+1}^\infty\gamma^t r(s_\Delta,\cdot)\right]\\
    {}& =\mathbb{E}_s\left[\sum_{t=0}^{T-1} \gamma^t r(s_t,a_t) + \gamma^Tr_\Delta(s_T,a_T)\right],
\end{aligned}
    \end{equation*}
where \(r_\Delta(s_T,a_T)=r(s_T,a_T)+const_\Delta = r(s_T,a_T)+\sum_{t=T+1}^\infty\gamma^tr(s_\Delta,\cdot)\), since \(r(s_\Delta,\cdot)\) and \(\gamma\) are constants. This is our stopping reward \(G(\cdot)\).\par 
\begin{figure}
\centering
\includegraphics[width=0.5\linewidth]{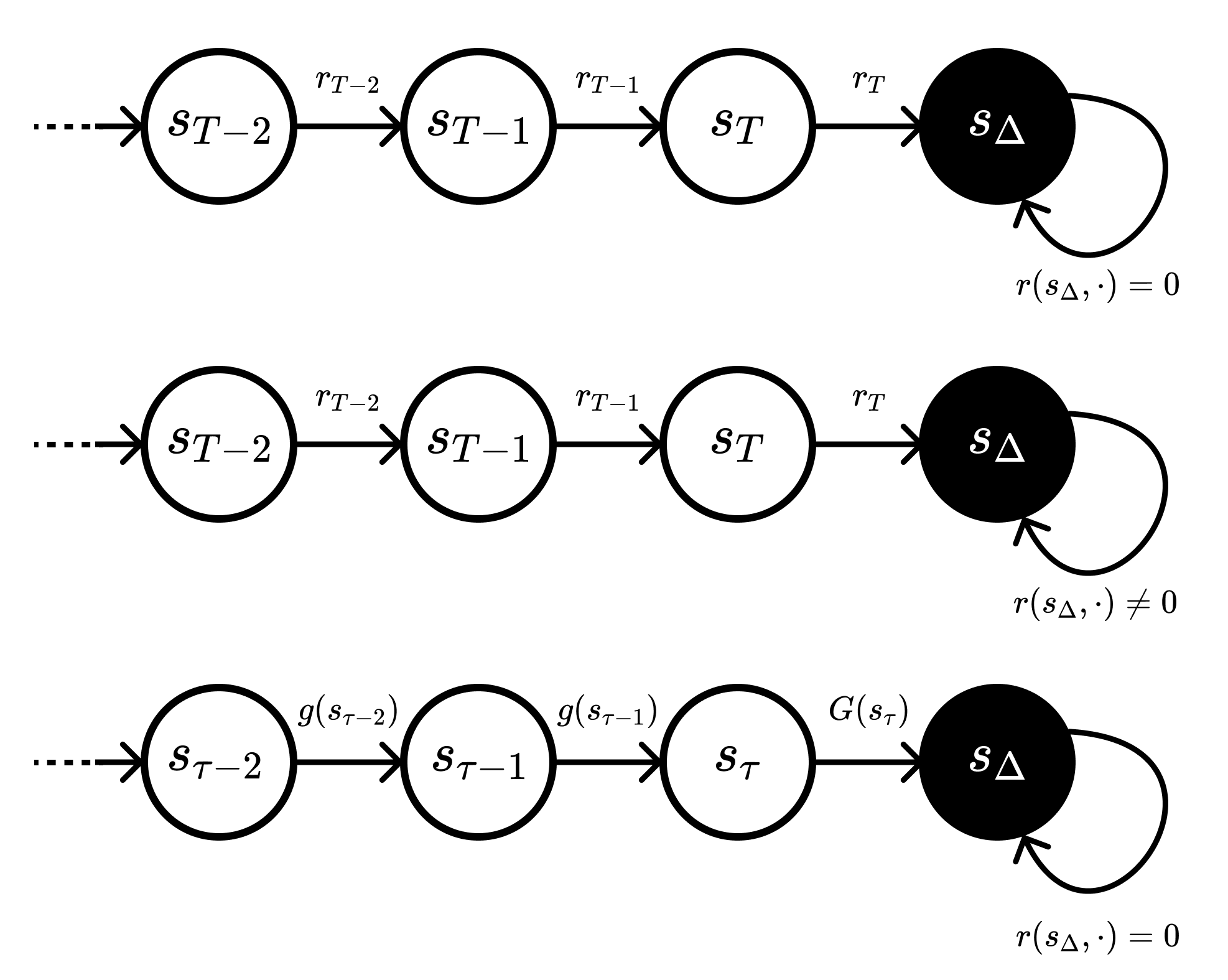}
  \caption{Top-to-bottom: a standard MDP with zero reward in absorbing states; MDP with non-zero reward in absorbing states; SMDP with stopping and continuation gains and zero-reward absorbing states.}
\label{fig:non-stationary mdp and optimal stopping}
\end{figure}
In case of the optimal stopping problem, we also deal with episodic MDPs, where the last observed state at a stopping time \(\tau\) is followed by a transition into the absorbing state \(s_\Delta\) with zero rewards.
\begin{equation*}
    \begin{aligned}
        V^{\pi_{absorbing}}(s){}& =\mathbb{E}_s\left[\sum_{t=0}^{T-1} \gamma^t r(s_t,a_t) + \gamma^Tr_\Delta(s_T,a_T)\right] \\
        {}& =\mathbb{E}_s\left[\sum_{t=0}^{\tau-1} \gamma^t r(s_t,a_t) + \gamma^\tau r_\Delta(s_\tau,a_\tau)\right] \\
        {}& =\mathbb{E}_s\left[\sum_{t=0}^{\tau-1} \gamma^t g(s_t) + \gamma^\tau G(s_\tau)\right]\\
        {}&=V^{\pi_{stopping}}(s)
    \end{aligned}
\end{equation*}

To address the practical implementation of the IOS algorithm, we consider the following boundary conditions: \textbf{(1)} The stopping action $a=0$ sends the agent into a set of cemetery states with probability one, which can be thought of as entering an absorbing state. \textbf{(2)} The last state in a path is followed by cemetery states, i.e., once the agent reaches the final state of the path, they enter the set of cemetery states with probability one. \textbf{(3)} Cemetery states are absorbing. Once an agent enters a cemetery state, they stay there forever, i.e. any action performed from a cemetery state will transition the agent into another cemetery state with probability one. \textbf{(4)} Both rewards and, hence, the V- and Q-function in the cemetery states are zero, that is, $(Q(s_\Delta, a)=0$, $V^\pi(s_\Delta)=0$ and $r(s_\Delta,a_\Delta)=0)$. This is reflected in the pre-processing stage as described in Algorithm \ref{alg:iqs-preprocess}.

\section{Appendix B: Model and implementation details}
This section provides a more detailed description of the model architecture and implementation details. 

\subsection{Model implementation}
A more detailed structure of the Neural Network used to implement the DO-IQS algorithm is presented in Figure \ref{fig:fcnn_structure}.
\setlength{\textfloatsep}{0pt}
% \begin{wrapfigure}{C}{1\textwidth}
% \vskip -18pt
    % \begin{minipage}{1\textwidth}
\begin{algorithm}[!tbp]
    % \small
    \caption{IQS data pre-processing procedure}
    \label{alg:iqs-preprocess}
    \begin{algorithmic}[1]
    \Require{Stopped expert observation history \(O_{T_m}^M=\{s_{T_m},a_{T_m}\}^{m=0:M-1}\), \(\forall 0\leq m\leq M-1\); a zero-valued cemetery state \(s_\Delta=\boldsymbol{0}\) }
\Ensure{Expert training dataset}
\State Compose two processes by stacking \(M\) paths together \(S_{T_m}^M = (s_t^m)_{0\leq t\leq T_m-1}^{0\leq m \leq M-1}\), and \(A_{T_m}^M = (a_t^m)_{0\leq t\leq T_m-1}^{0\leq m \leq M-1}\) and set \(L\) to be its length
\State Initialise a shifted state process \(S'_L\) with zeros
\For{\(l\) from \(0\) to \(L-1\)}
\State Set \(s'_l=s_\Delta\) if \(a_l=0\), else set \(s'_l = s_{l+1}\)
\EndFor
\end{algorithmic}
\end{algorithm}
The algorithm \ref{alg:st_region_via_bi_iqs} shows how to recover the OS region using the approximate Q-function. \par 
In practice the IOS framework is implemented as follows. Assume that an expert has access to a reward \(r_E\) and follows their (optimal or sub-optimal) policy \(\pi_E\), corresponding to the reward maximisation problem: \(\pi_E=\argmax_{\pi}V^\pi_{r_E}(\cdot)\), where \(V_{\chi}^\pi(s)=\mathbb{E}_s[\sum_{t=0}^\infty \gamma^t \chi(s_t)]\). We assume that for any \(\pi_E\) we could specify a reward function \(r_E\), s.t. the optimality principle holds. This assumption is necessary for working with real world data, where it is often not possible to know the structure of the expert's decision making process. We note that this assumption is not particularly restrictive and can be viewed as matching a stopping region to the expert's behaviour, and matching a reward function to the stopping region.\par 
The expert also has access to the environment to perform a certain number of independent trials by following their policy \(\pi_E\). This produces sequences of states \((s_0,s_1,s_2,...,s_{T_k})_{k=0,...,K}\), where K is the number of trials and \(T_k\) is the duration of the k-th trial. Since we are working with OS problems, that is, we know that the only two types of admissible actions are \(a=0\) (stop) and \(a=1\) (continue). Hence, the expert's trajectories can be accompanied by corresponding action sequences: \((a_0,a_1,...,a_{T_k})_{k=0,...,T}\). \par 
Noting the structure of the OS problem one can identify the type of these actions exactly, i.e. we know that the expert during the trial \(k\) transitioned for \(T_k\) time steps with the state \(s_{T_k}\) being the last observed. This means that all the actions of the expert up to state \(s_{T_k}\) (i.e. \((s_0,s_1,...,s_{T_k-1})\) were continuation actions, and the stopping action was chosen in state \(s_{T_k}\), after which the cemetery state set \(\{s_\Delta\}\) was entered.\par 
In this way we obtain a more familiar representation of the expert dataset  used for IRL procedure:
\begin{equation*}
\begin{aligned}
    D_E = {}& (s_0,a_0=1,s_1,a_1=1,...,\\
    {}& s_{T_K-1},a_{T_K-1}=1, s_{T_k}, a_{T_k}=0, s_{\Delta}, a_{\Delta}=0, s_{\Delta}, a_{\Delta}=0,...).
\end{aligned}
\end{equation*}
The next step is to transform the dataset to be suitable for training IOS models. For that, we compliment the original \(D_E\) with its shifted version pairing current state-action pairs with their one-step-lookahead counterparts. Note that we are not concerned with the values of the cemetery states \(s_\Delta\), since in the practical implementation of the algorithm they disappear in the Bellman equation through the use of an indicator function checking if the action at the current state is \(a=0\): \(\mathds{1}_{a=0}=(1-a)\). This trick ensures that the value of the cemetery states is always zero.
% \begin{figure}\label{alg:bi-iqs}
%     \centering
%     \includegraphics[width=1\linewidth]{images/2nets (2).jpg}
%     \caption{DO-IQS model structure}
%     \label{fig:DO-IQS_structure}
% \end{figure}
\begin{algorithm}[!htbp]
    % \small
    \caption{Stopping region recovery using Model-Based IQS}
    \label{alg:st_region_via_bi_iqs}
    \begin{algorithmic}[1]
    \Require{Unstopped state process \(S_T^M\); an approximation of the stopping Q-function \(Q_{\epsilon,\theta}\) and the environment model \(\mathcal{P}_\theta\); a  discount factor \(\gamma\in (0,1]\);}
\Ensure{A stopping region \(D_{\epsilon,\theta}=\{s|G(s)\geq g(s) + \gamma \mathcal{P}_\theta(s'|s,a=1)V_{\epsilon,\theta}(s')\}\)}
\For{state \(s_t^m\) in \(S_T^M\)}
\If{\(Q_{\epsilon,\theta}(s_t^m,a=0)\geq Q_{\epsilon,\theta}(s_t^m,a=1)\)}
\State Add \(s_t^m\) to \(D_{\epsilon,\theta}\)
\EndIf
\EndFor
\end{algorithmic}
\vskip 5pt
\end{algorithm}
\subsection{Hyper-parameters selection details}
We use the following hyper-parameter setup throughout the models and examples:
\begin{itemize}
    \item The learning rates for the Q-function, environment dynamics \(P\) and the continuation gain \(g\) were set to  0.01 decaying exponentially after each epoch with the factor 0.9999. We do not fine tune learning rates specifically for each model to avoid overfitting to the validation set;
    \item \(\epsilon=0.1\) with the multiplier factor of 0.9999 applied at each epoch to decrease the value of the temperature parameter over the course of training;
    \item N-neighbours for the SMOTE oversampling algorithm: 12;
    \item The initial confidence score is set to be \(\alpha=0.99\) and then decreased after each epoch exponentially through multiplying it by a decay factor 0.95;
    \item Discount factor \(\gamma=0.99\);
    \item Batch size: 128;
    \item Number of training epochs: 200.
\end{itemize}
In this work, we are taking a naive approach to dynamics modelling by predicting the next state \(s'\) (and not the distribution over the possible next states) to show how integrating even a simple environment model results in a substantial improvement in the modelling results. \par 
An example of the neural network structure used to approximate the dynamics of states together with the Q-function in the model-based IQS approach is presented in Figure \ref{fig:fcnn_structure}.
\begin{figure}\label{fig:ann_doiqs}
    \centering
    \includegraphics[width=0.75\linewidth]{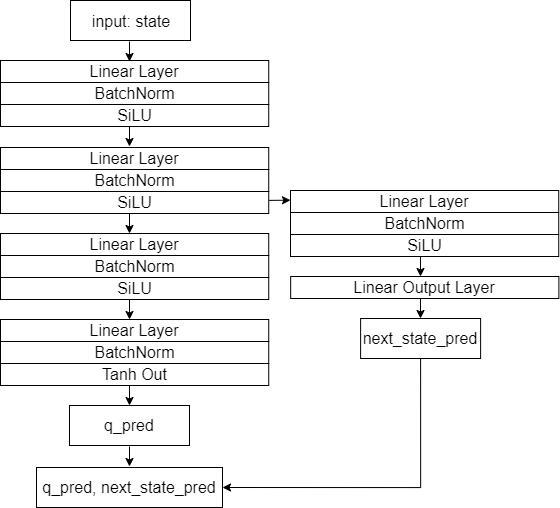}
    \caption{ANN structure for the Model-based IQS}
    \label{fig:fcnn_structure}
\end{figure}

% \subsection{Computational complexity analysis}
% To evaluate how our modifications to the baseline algorithm influence the runtime, we conduct a brief time complexity analysis.\par 
% We first note that since functional approximation in the current work is conducted using Artificial Neural Networks (ANN), time complexity will greatly depend on the number of parameters in the ANNs. The following presents some heuristics about each model runtime for a generic feed-forward network
\subsection{Computational complexity analysis}
The time complexity of any machine learning algorithm depends highly on a few factors, including the training dataset size, the complexity of the functional approximator (e.g. depth and width in case of the neural networks), hardware specifications, etc. Here we give some rough estimates of the runtime complexity in terms of the big-\(O\) notation. \par 
For simplicity we assume the time complexity of a forward-backward pass for a generic feed-forward neural network with a two-dimensional output to be \(O(n_{NN})\). This is exactly the time complexity of the baseline \textbf{classification} model as well as the standard \textbf{IQS} algorithm. For the \textbf{model-based IQS} we just need to add an additional output layer (assuming that this adds a time complexity of \(O(n_{p})\)) resulting in a total complexity of \(O(n_{NN}+n_{p})\). \textbf{IQS-(CS-)SMOTE} increases the runtime through increasing the batch size (\(O(n_{B_{SMOTE}})\), where \(n_{B_{SMOTE}}\) is the number of SMOTE-generated samples in the batch) and the SMOTE procedure itself (used only once at the beginning of training) is \(O(n_{SMOTE} M_{min})\), where \(n_{SMOTE}\) is the number of synthetic instances and \(M_{min}\) is the number of minority samples in the whole dataset. Note that adding SMOTE oversampling does not increase the time complexity during the inference stage, since it is only applied to the training samples. In case of the \textbf{DO-IQS} we deal with two networks: \(Q_\theta\) (with \(P_\theta\)) and \(g_\phi\) (with forward-backward pass complexity of \(O(n_g)\)) approximating the continuation gain. Using the same assumptions as before and noting that the input in \(Q_\theta\) increases in size by one (denoting the new time complexity by \(O(n_{NN,g})\)), we get the following complexity: \(O(n_{NN,g}+n_{p} + n_g)\). Applying the local bootstrap to the batch in case of the \textbf{DO-IQS-LB} adds the complexity of \(O(M_{B_{min}})\), where \(M_{B_{min}}\) is the number of minority samples in the batch.

\section{Appendix C: Experiments and ablation studies}
This section elaborates on the implementation of environments and the experimental results that were not included in the main paper due to the size constraint. It also contains additional experiments and ablation studies to show the performance of the proposed algorithm and highlights some insights revealed through experiments. 
The first three examples consider a problem of change-point detection with and without an autoregressive component. 

\begin{example}[CP1]\label{ex: CP1}
    The following example defines a change-point detection problem. First, a path of 51 time steps is generated using \(x^0\). Then, a change-point is chosen uniformly at random from the last \(70\%\) to \(90\%\) of the path. From the change-point and to the end of the path, the process is described by \(x^1\). The first observation is then chosen uniformly at random in \(0\%\) to \(50\%\) of the path at random, and all previous observations are removed. This is similar to the construction of the change-point example in \citet{2022oti} (see supplementary material). In this example, the mean of the noise component changes after the change-point.
    \begin{equation}
    x^0(t) = sin(\omega t) +\epsilon^{\mu^0,\sigma},
    \end{equation}
    \begin{equation}
    x^1(t) = sin(\omega t) +\epsilon^{\mu^1,\sigma},
    \end{equation}
    where \(x^0(t)\) and \(x^1(t)\) are the values of the process at some time \(t\) before and after the change-point occurs; \(\omega\) is the sinusoidal frequency; \(\epsilon^{\mu,\sigma}\sim N(\mu, \sigma)\) is the normal noise component with the mean \(\mu\) and standard deviation \(\sigma\); \(\mu^0\) and \(\mu^1\) are the means of the noise component before and after the change-point. We take \(\omega=1\), \(\mu^0=0.5\), \(\mu^1=5\) and \(\sigma=1\). For training purposes the change-point is set to be two time steps after the actual change-point occurs. The scatter plot of the stopping (red) and continuation (green) actions is presented in Figure \ref{fig:cp1}.
\end{example}

\begin{figure}[htb]
\centering
  \includegraphics[width=0.75\linewidth]{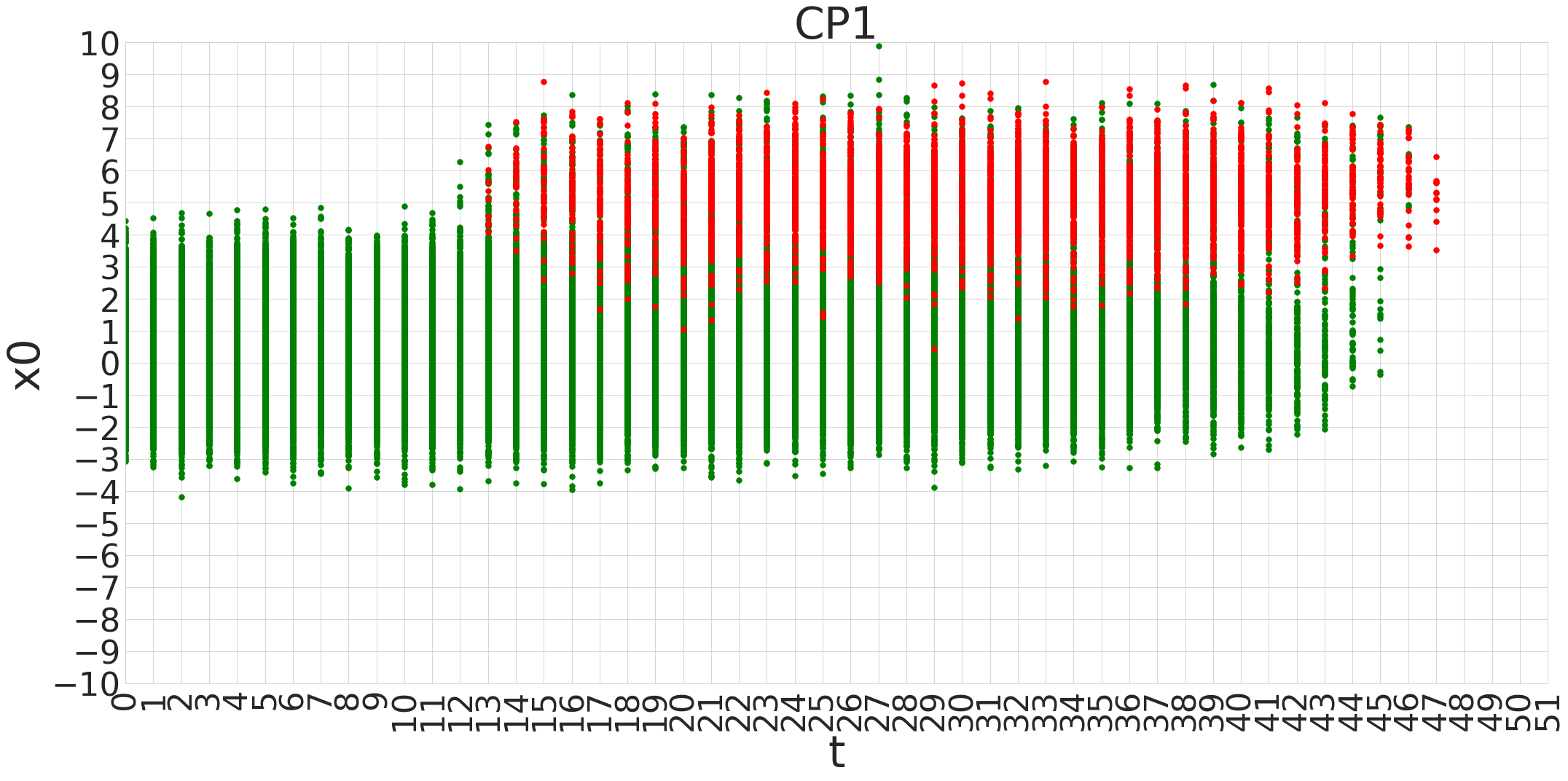}
  \caption{CP1: change-point detection example with the mean change. The plot of 5000 trajectories (green: continuation decisions; red: stopping decisions)}
\label{fig:cp1}
\end{figure}

\begin{example}[CP2]\label{ex: CP2}
    This has the same setup as the Example \ref{ex: CP1}, but the dynamics of the process has got autoregressive components of the second order:
    \begin{equation}
    x^0(t) = sin(\omega t) + \beta^0_0 x(t-1) + \beta^0_1 x(t-2)+\epsilon^{\mu,\sigma},
    \end{equation}
    \begin{equation}
    x^1(t) = sin(\omega t) + \beta^1_0 x(t-1) + \beta^1_1 x(t-2)+\epsilon^{\mu,\sigma},
    \end{equation}
    where \(\beta_0^0, \beta^0_1\) and \(\beta_1^0, \beta^1_1\) are the autoregressive coefficients of the first and second order for the process before and after the change-point respectively. We set \(\omega=1\), \(\mu=0.5\), \(\sigma=1\), \(\beta_0^0=0.25\), \(\beta^0_1=0.05\), \(\beta_1^0=0.75\), \(\beta^1_1=0.5\). The scatter plot of the stopping (red) and continuation (green) actions is presented in Figure \ref{fig:cp2}.
\end{example}

\begin{figure}[htb]
\centering
  \includegraphics[width=0.75\linewidth]{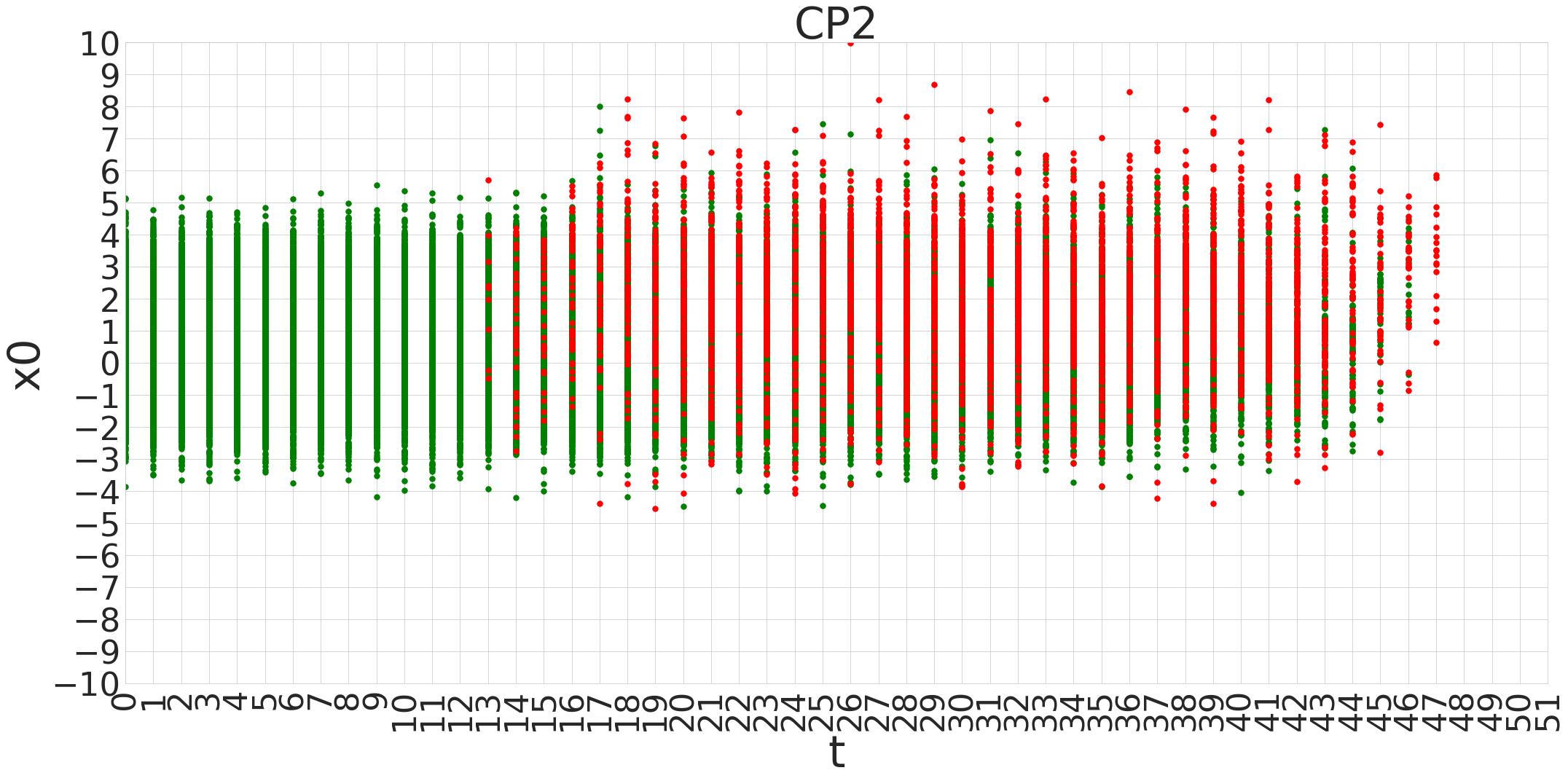}
  \caption{CP2: change-point detection example with the change in AR(2) coefficients. The plot of 5000 trajectories (green: continuation decisions; red: stopping decisions)}
\label{fig:cp2}
\end{figure}

\begin{example}[CP3]\label{ex: CP3}
    This has the same setup as the Example \ref{ex: CP1}, but instead of the mean, the variance of the noise component changes.
    
    \begin{equation}
    x^0(t) = sin(\omega t) +\epsilon^{\mu,\sigma^0},
    \end{equation}
    \begin{equation}
    x^1(t) = sin(\omega t) +\epsilon^{\mu,\sigma^1},
    \end{equation}
    where \(\sigma^0\) and \(\sigma^1\) are the standard deviations of the noise component before and after the change-point. We set \(\omega=1\), \(\mu=0.5\), \(\sigma^0=1\), \(\sigma^1=5\). The scatter plot of the stopping (red) and continuation (green) actions is presented in Figure \ref{fig:cp3}.
\end{example}

\begin{figure}[htb]
\centering
  \includegraphics[width=0.75\linewidth]{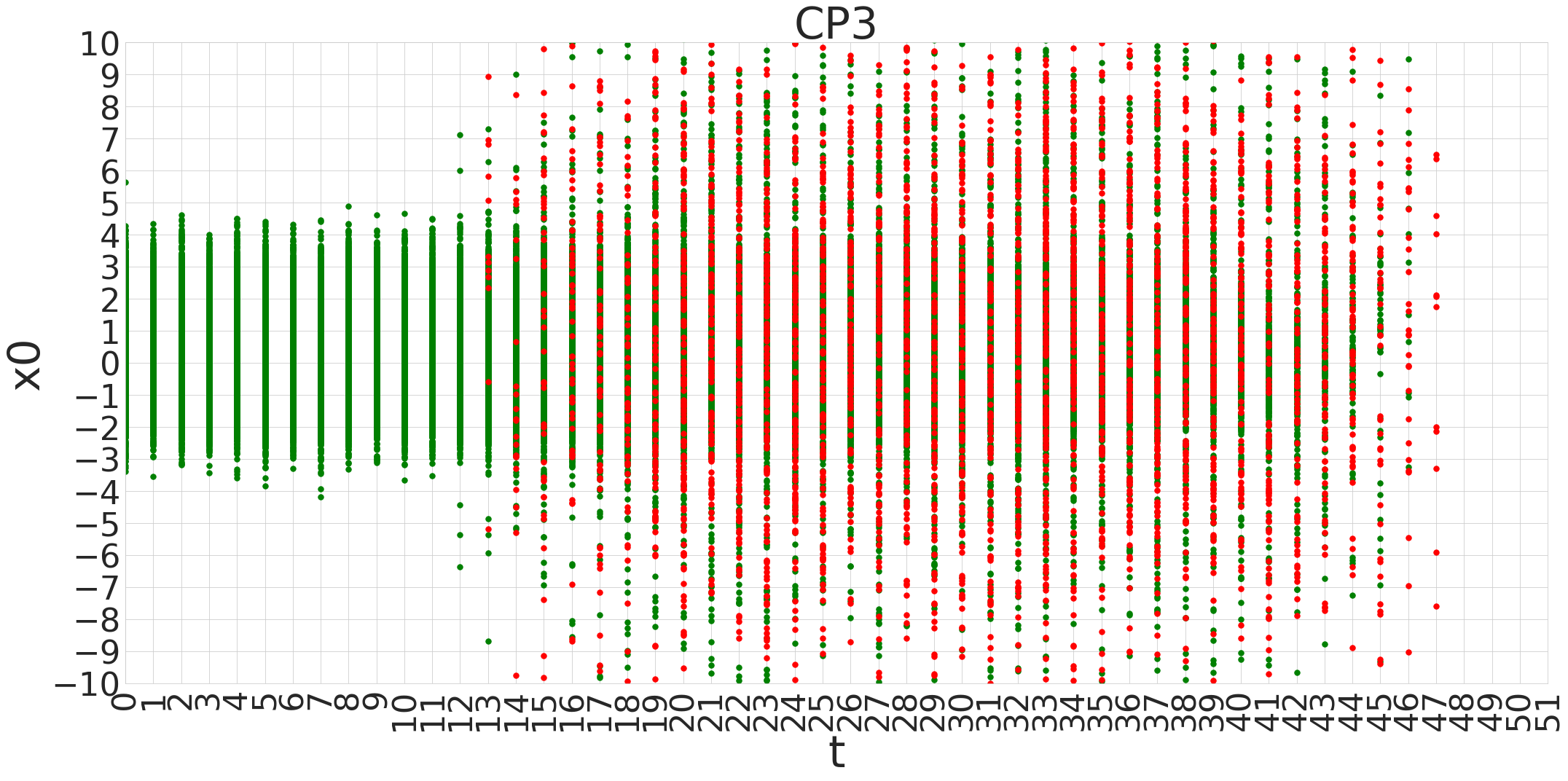}
\caption{CP3: change-point detection example with the variance change. The plot of 5000 trajectories (green: continuation decisions; red: stopping decisions)}
\label{fig:cp3}
\end{figure}

The next two examples provide a convenient visual representation of the stopping regions, both the true and the recovered one. 
\begin{example}[Radial stopping]\label{ex: RADIAL}
An N-dimensional standard Brownian motion starting at the origin is simulated path-wise. Each path has a length of a maximum of 50 time steps. The paths are stopped if they hit the boundary of a circular continuation region \(C_{radial}[t]\) (or exit it) unique for each time \(t\). The continuation region at time \(t\) is defined by a disc of radius \(r[t] = 0.5+0.05t\). The plot representing the stopping boundaries (left) together with the scatter plot of the stopping actions (right) are presented in Figure \ref{fig:radial_ex}.
\end{example}

\begin{figure}[htb]
\centering
  \includegraphics[width=0.45\textwidth]{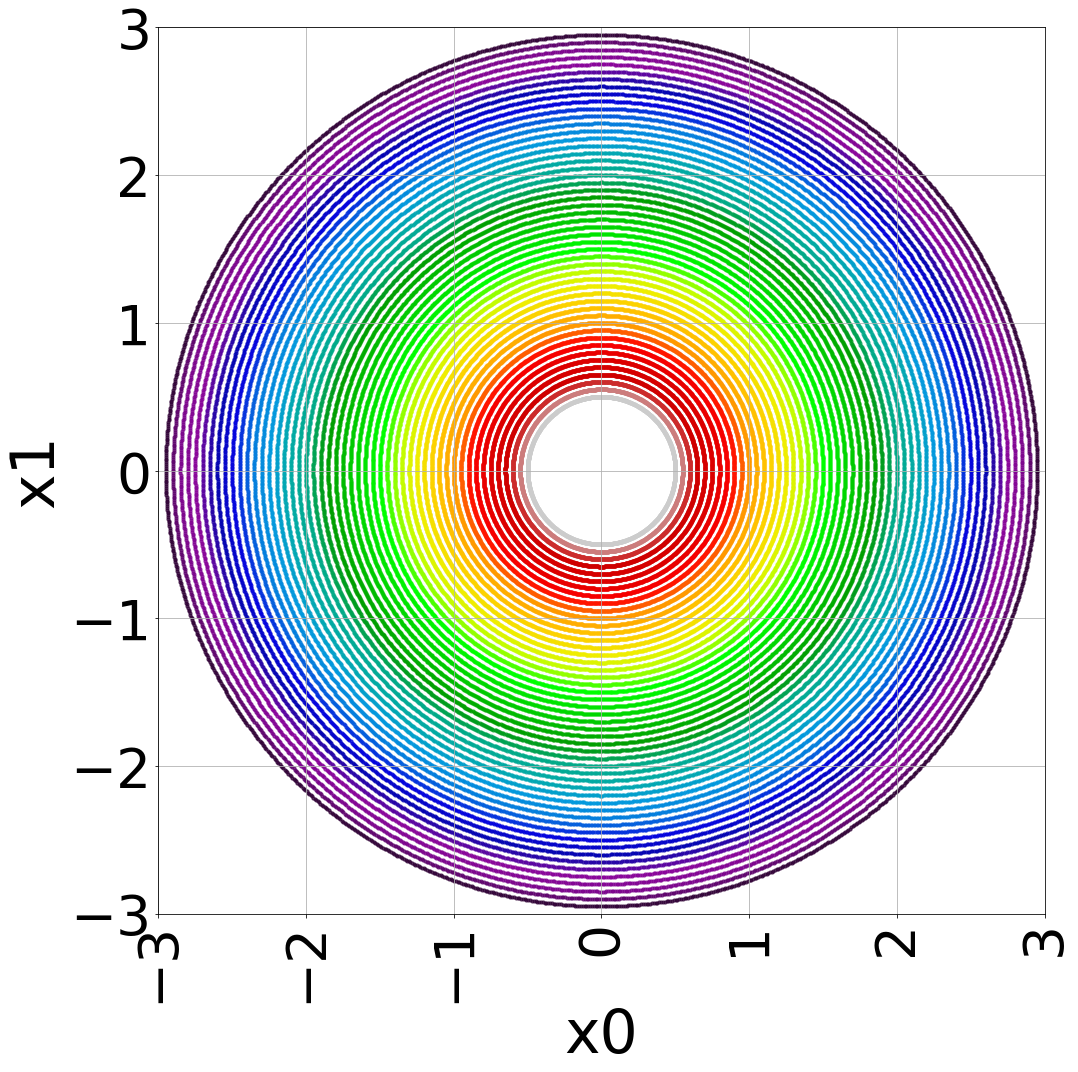}
  \includegraphics[width=0.45\textwidth]{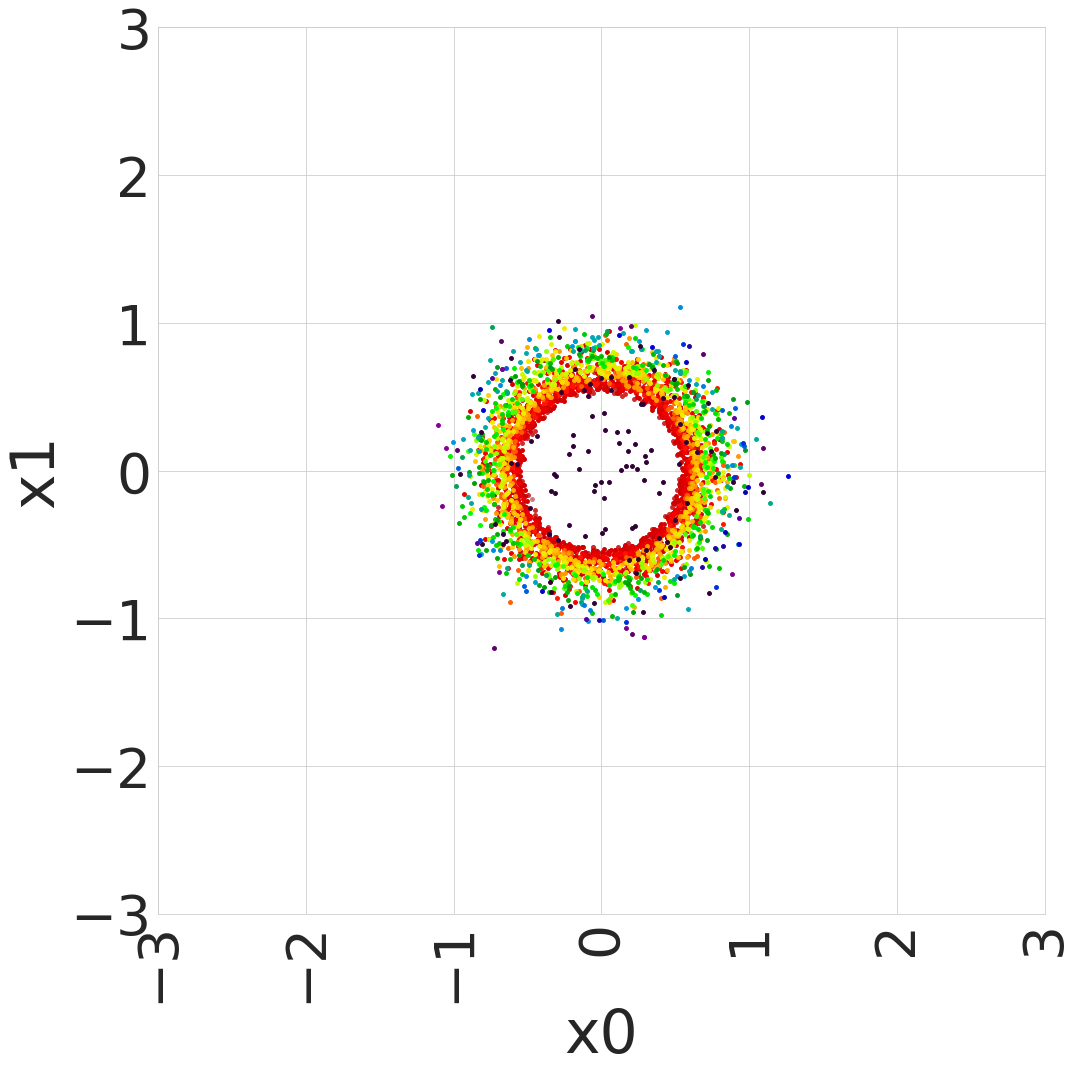}
\caption{Radial stopping example. (a) Time-dependent stopping boundaries; (b) Stopping actions. Different colours define different times.}
\label{fig:radial_ex}
\end{figure}

The above example has a simplistic form of the stopping region (although changing in time), which makes it easier for the algorithms to pick on the shape of it. Moreover, some of the approaches we present in this work rely on local interpolation, which might provide ground for mistakes in examples with more complex shape. The following example aims to cover these types of scenarios.
\begin{example}[STAR-stopping]\label{ex: STAR}
The example has a similar setup to the radial Example \ref{ex: RADIAL} described above. An N-dimensional Brownian motion, starting at the origin, is simulated. At each time \(t\), the path is stopped if it leaves a star-shaped region \(C_{star}[t]\). The star-shaped continuation region at time \(t\) is defined by an outer radius \(r_{outer}\), an inner radius \(r_{inner}\) and the number of angles \(n_{angles}\). The corner points are then interpolated to create an outline of the star shape which changes its radius over time. The inner and outer radiuses at time \(t\) are defined as \( r_{outer}[t] = 0.5+0.05t\) and \( r_{inner}[t] = 0.5r_{outer}[t]\). The plot representing the stopping boundaries (left) together with the scatter plot of the stopping actions (right) are presented in Figure \ref{fig:star_ex}.
% The rotation angle is defined as \(\alpha_{rotation}[t]=0.01t\)
\end{example}

\begin{figure}[htb]
\centering
  \includegraphics[width=0.45\textwidth]{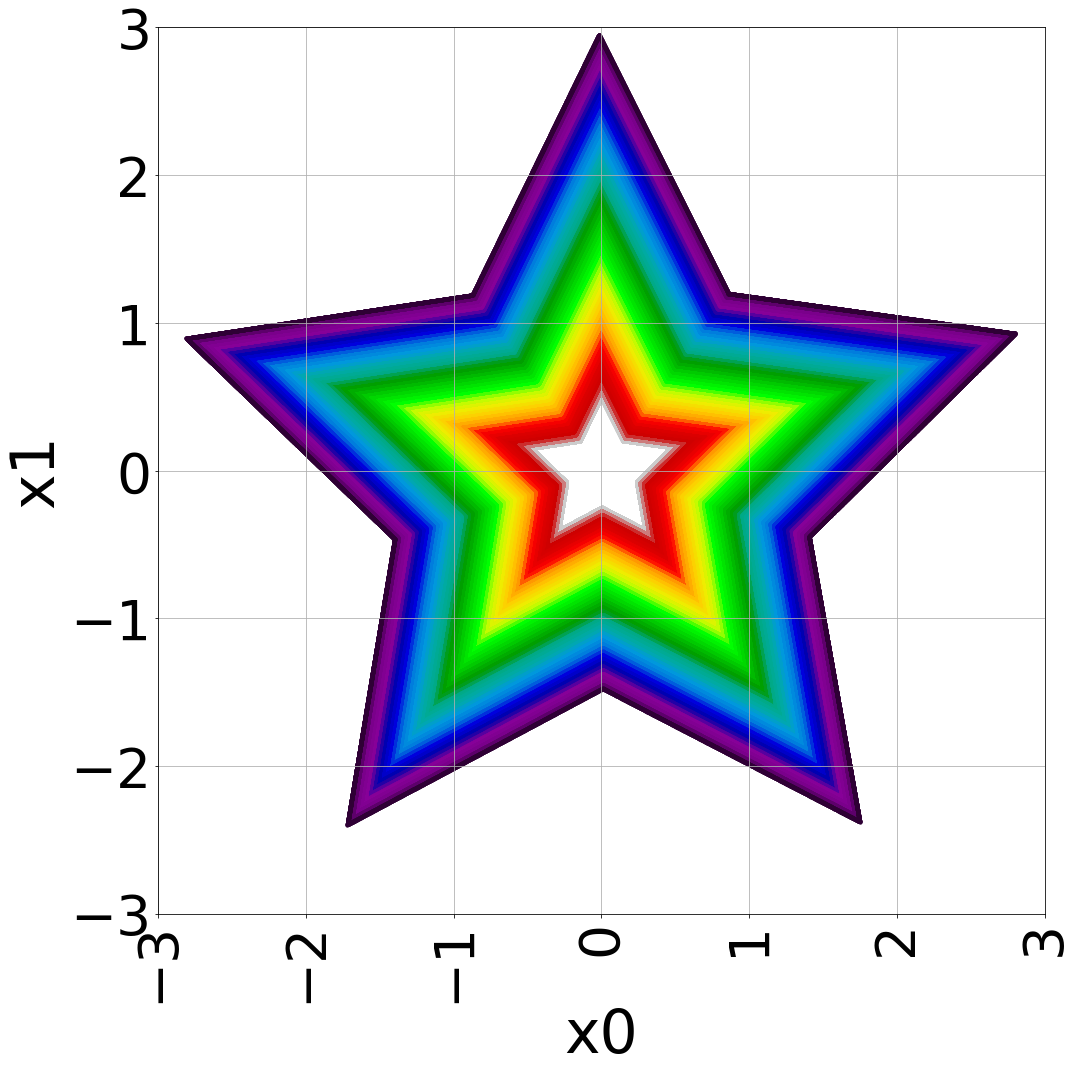}
  \includegraphics[width=0.45\textwidth]{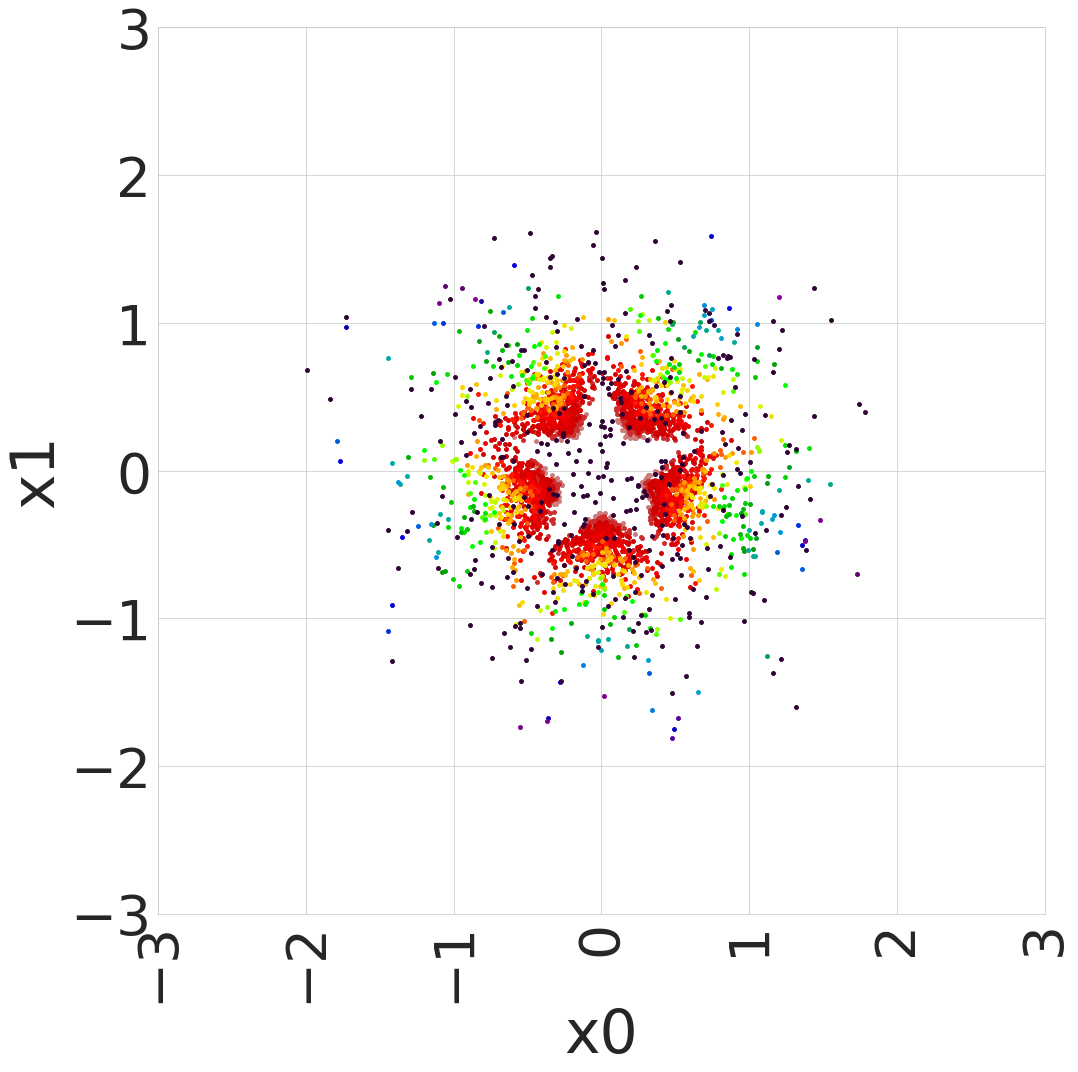}
\caption{STAR stopping example. (a) Time-dependent stopping boundaries; (b) Stopping actions. Different colours define different times.}
\label{fig:star_ex}
\end{figure}

\subsection{Results}

For each of the above examples, we simulate 175 training paths (split into training and validation paths in the 70/30\% ratio) and 75 testing paths for which we report the results in Tables \ref{tbl: cp_results} and \ref{tbl: rad_star_results}. Note how D-IQS-LB outperforms all the other models for the change-point examples with a clear marginal advantage for the change-point detection problem with an autoregressive component.\par To showcase convergence, we also present evolution of balanced accuracy and IQ-loss in Figures \ref{fig:results_cp} and \ref{fig:results_rad_star}. Balanced accuracy values are evaluated on the validation paths. We further note that the convergence in terms of the loss function does not necessarily correspond to an increase in the balanced accuracy score, further supporting the need for a different, task-specific evaluation and convergence metrics.

\begin{figure*}[ht]
    \begin{minipage}[l]{0.33\linewidth}
        \centering
        \includegraphics[trim={0 0 0 2.7cm}, clip, width=1\linewidth]{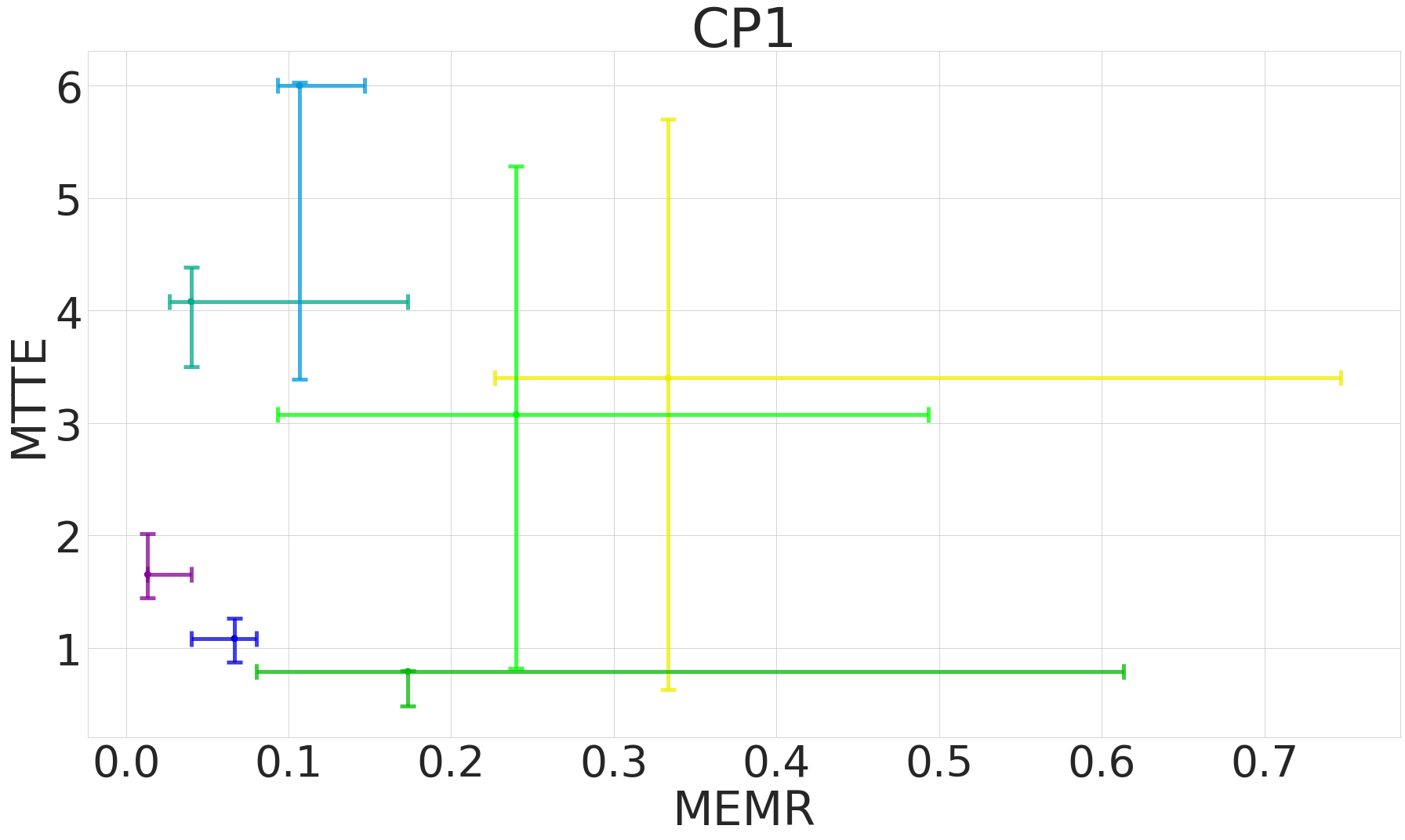}
        \captionsetup{labelformat=empty}
    \end{minipage}\hfill
    \begin{minipage}[l]{0.33\linewidth}
        \centering
        \includegraphics[trim={0 0 0 2.7cm}, clip, width=1\linewidth]{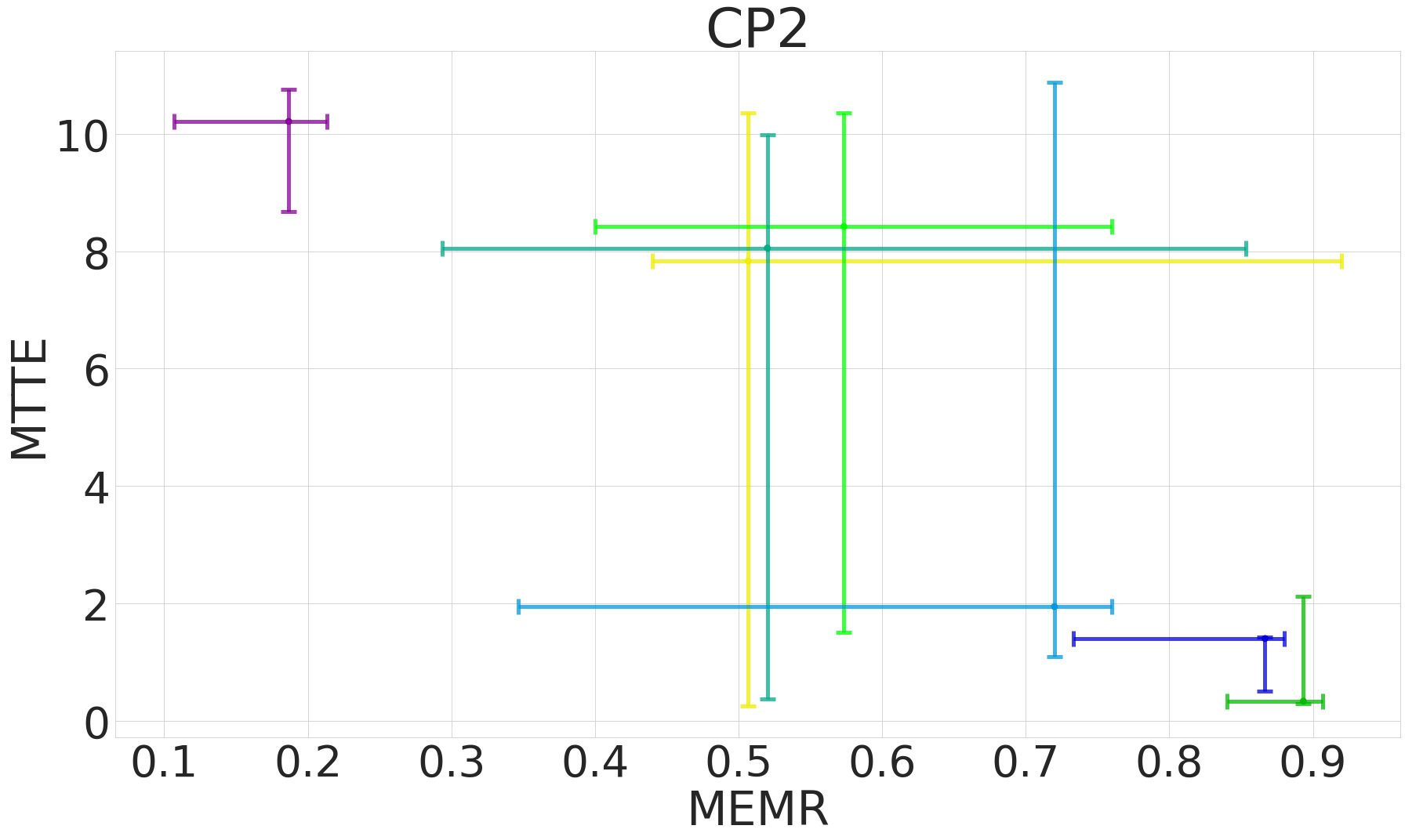}
        \captionsetup{labelformat=empty}
    \end{minipage}\hfill
    \begin{minipage}[l]{0.33\linewidth}
        \centering
        \includegraphics[trim={0 0 0 2.7cm}, clip, width=1\linewidth]{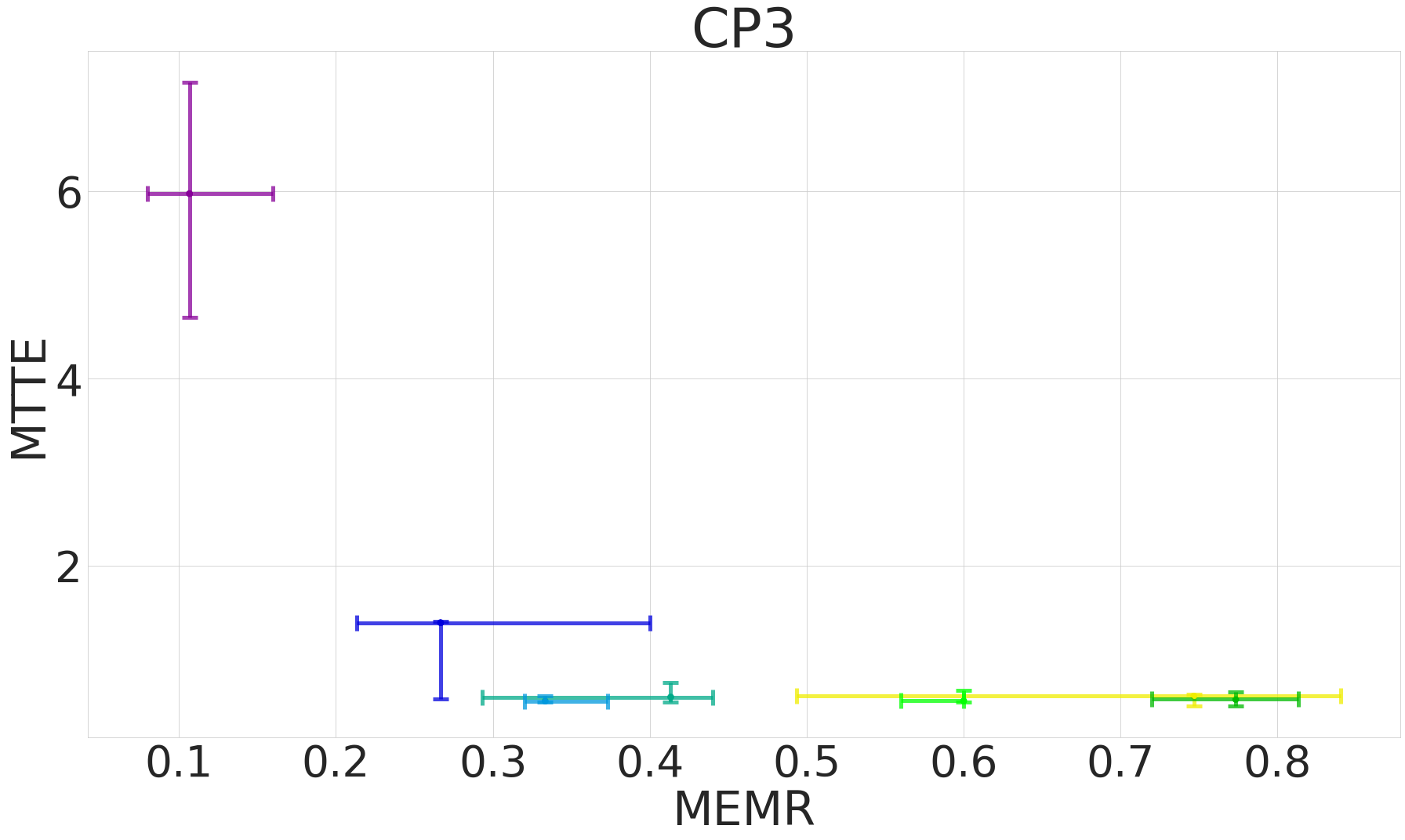}
        \captionsetup{labelformat=empty}
    \end{minipage}\hfill

    \begin{minipage}[l]{0.33\linewidth}
        \centering
        \includegraphics[trim={0 13cm 0 2.5cm}, clip, width=1\linewidth]{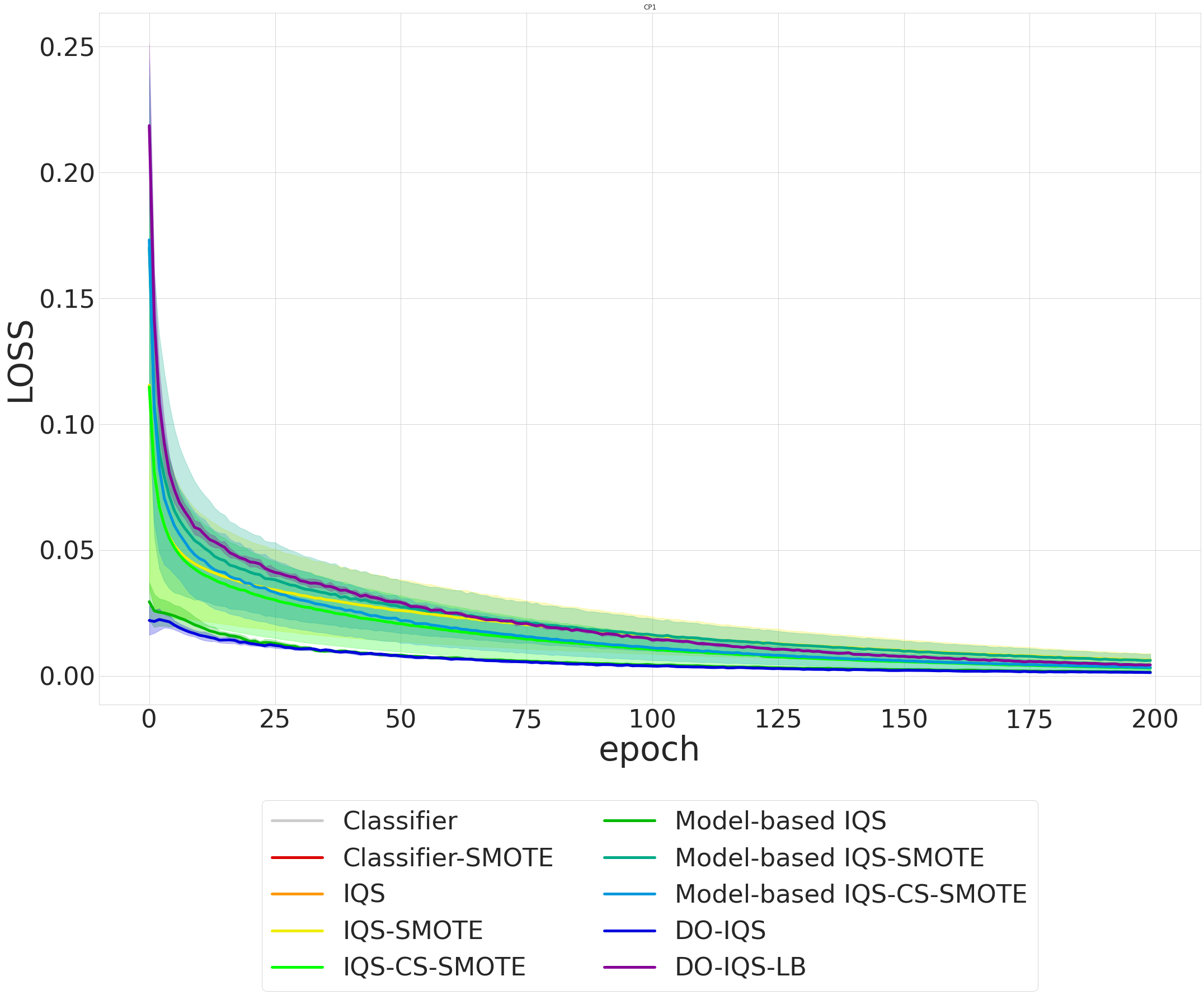}
        \captionsetup{labelformat=empty}
    \end{minipage}\hfill
    \begin{minipage}[l]{0.33\linewidth}
        \centering
        \includegraphics[trim={0 13cm 0 2.5cm}, clip, width=1\linewidth]{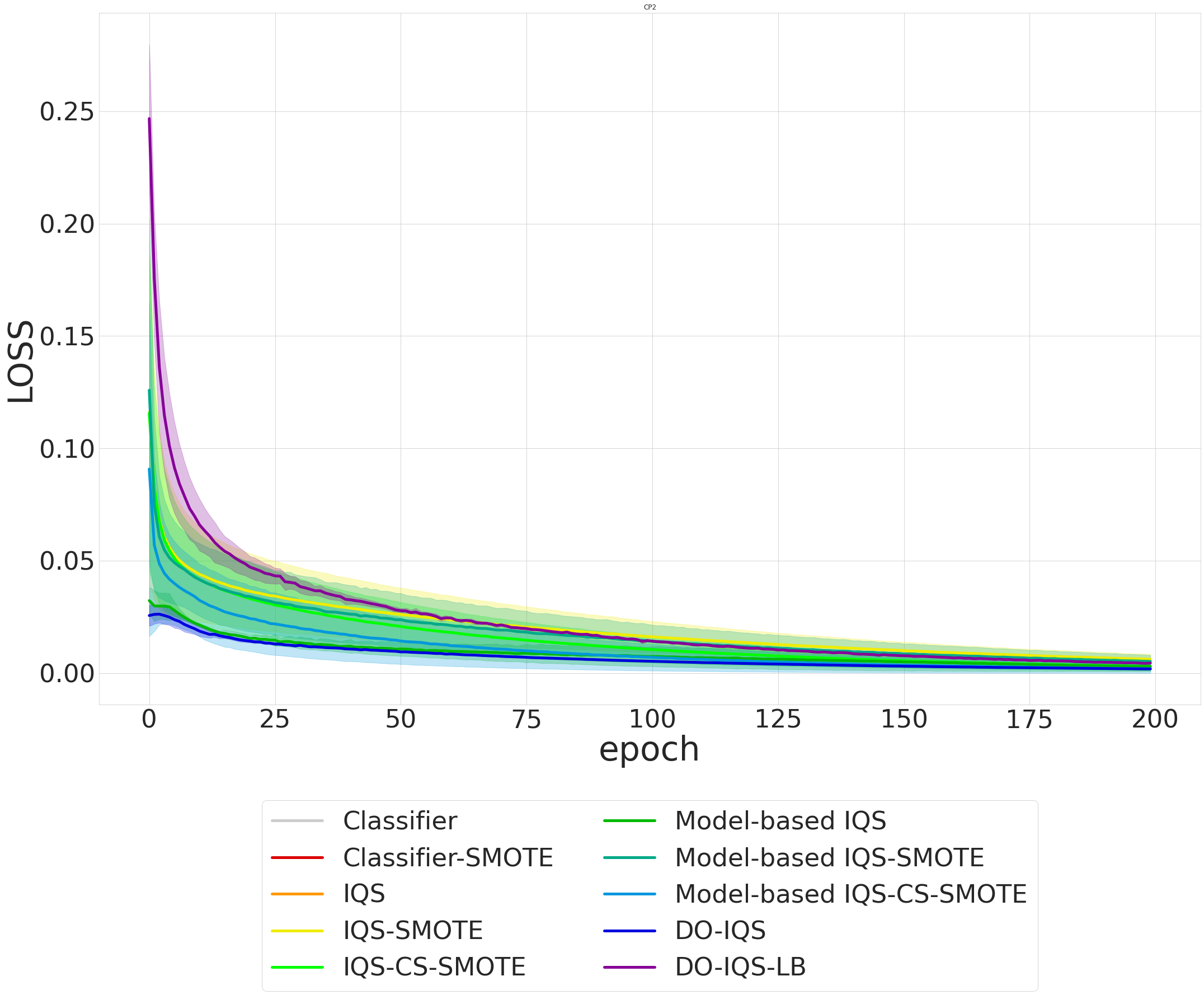}
        \captionsetup{labelformat=empty}
    \end{minipage}\hfill
    \begin{minipage}[l]{0.33\linewidth}
        \centering
        \includegraphics[trim={0 13cm 0 2.5cm}, clip, width=1\linewidth]{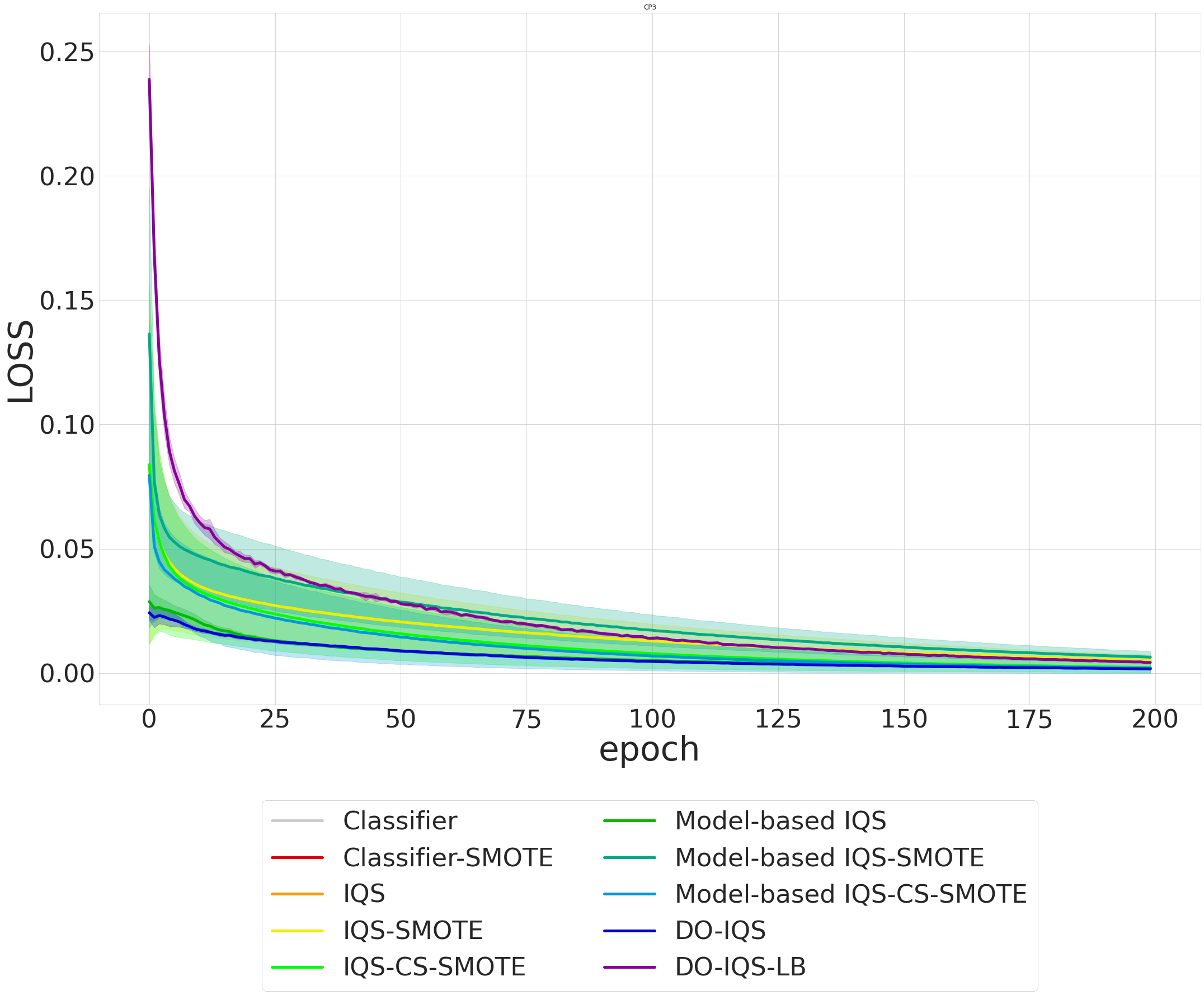}
        \captionsetup{labelformat=empty}
    \end{minipage}\hfill
    
    % \captionstyle{centerlast}
    \begin{minipage}[l]{0.33\linewidth}
        \centering
        \includegraphics[trim={0 13cm 0 0.6cm}, clip, width=1\linewidth]{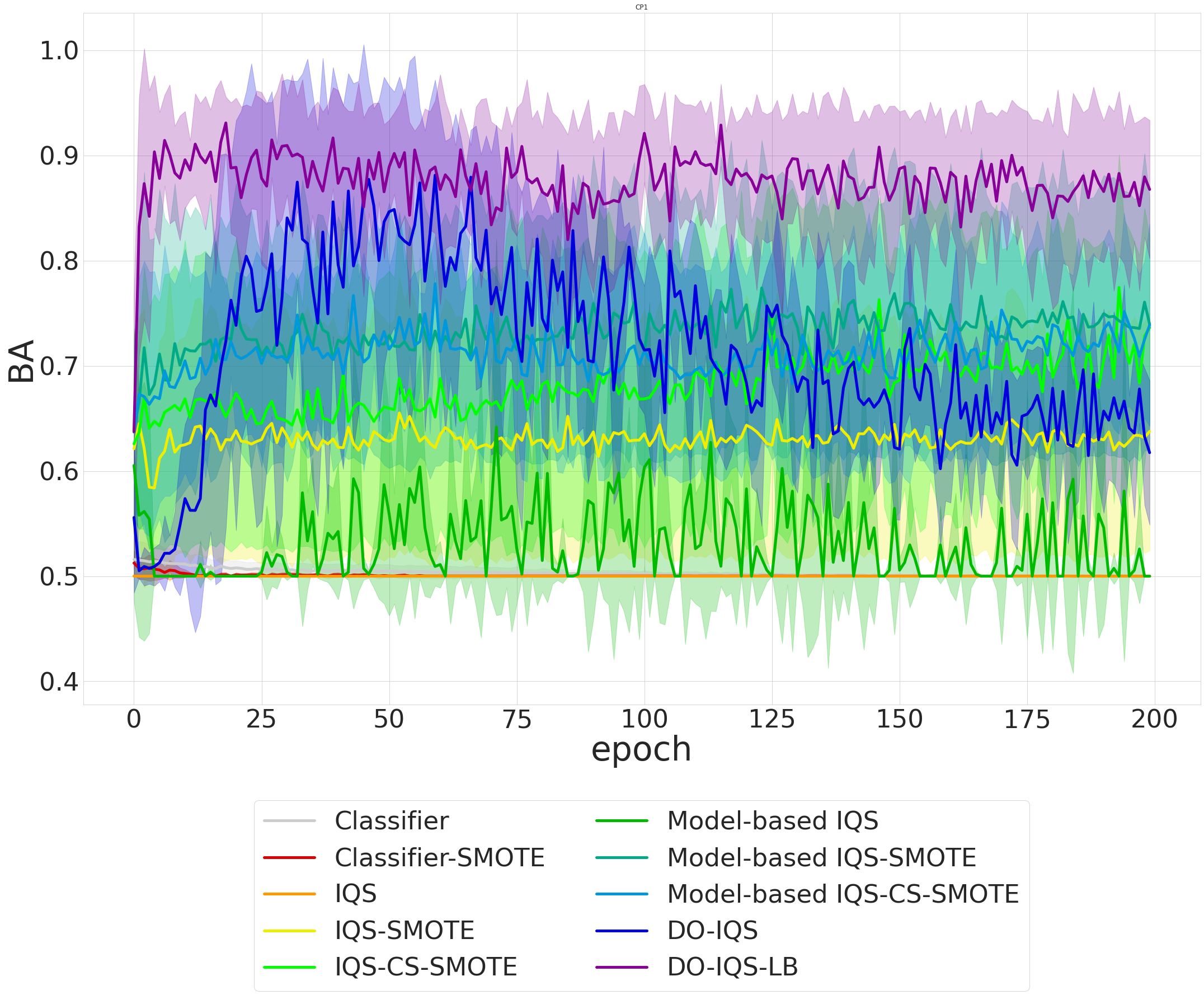}
        \captionsetup{labelformat=empty}
        \caption{CP1}
    \end{minipage}\hfill
    \begin{minipage}[l]{0.33\linewidth}
        \centering
        \includegraphics[trim={0 13cm 0 0.6cm}, clip, width=1\linewidth]{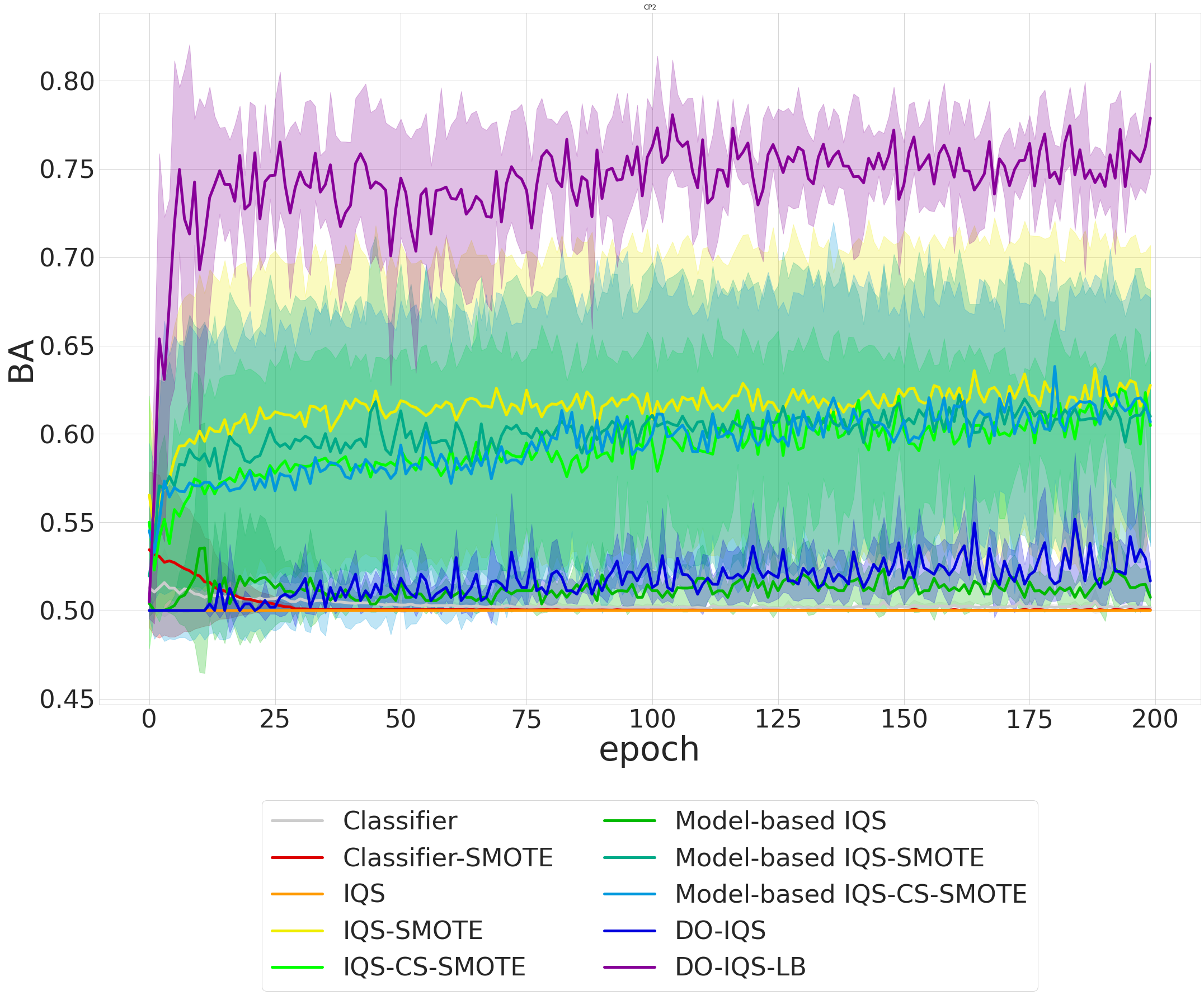}
        \captionsetup{labelformat=empty}
        \caption{CP2}
    \end{minipage}\hfill
    \begin{minipage}[l]{0.33\linewidth}
        \centering
        \includegraphics[trim={0 13cm 0 0.6cm}, clip, width=1\linewidth]{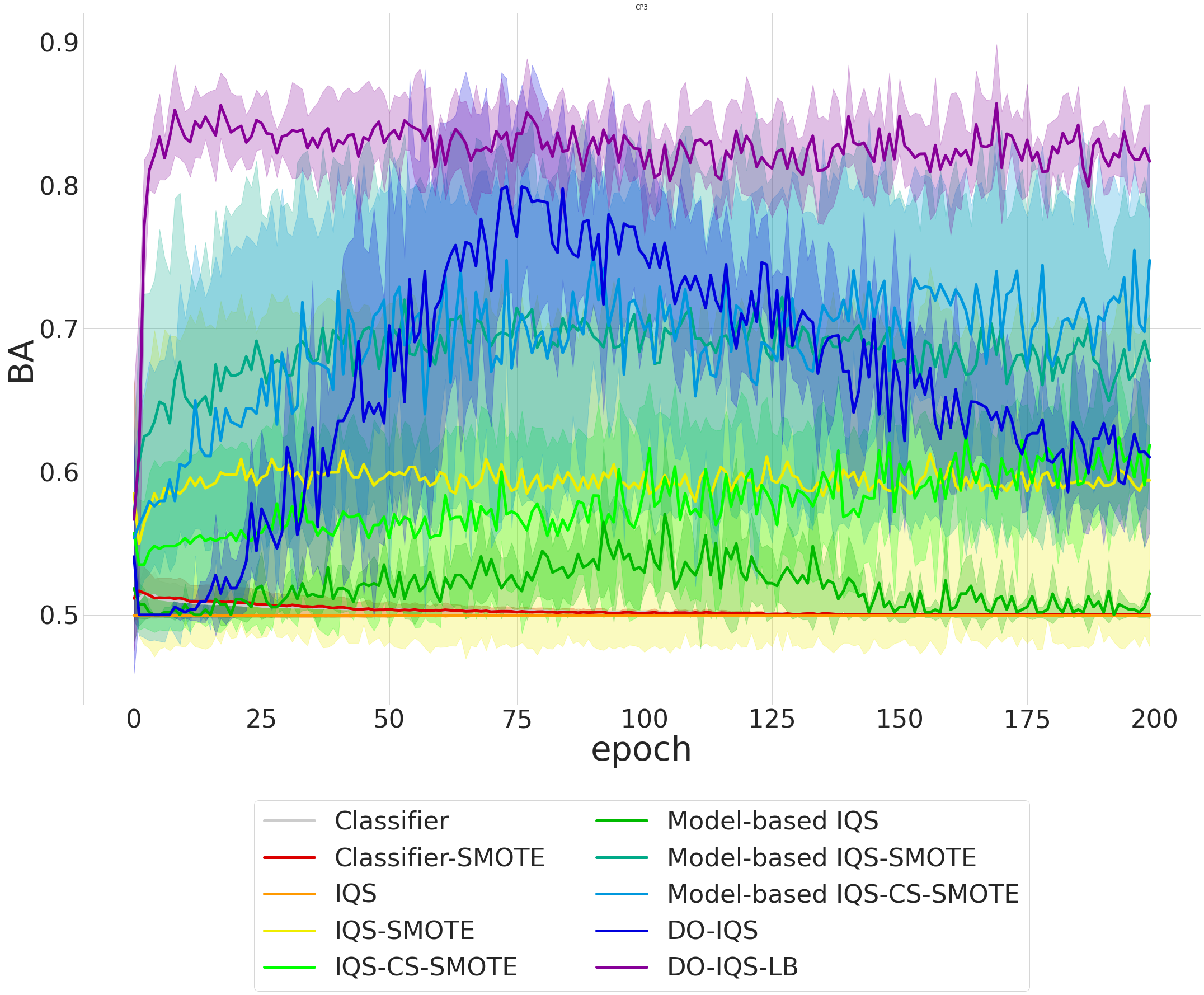}
        \captionsetup{labelformat=empty}
        \caption{CP3}
    \end{minipage}\hfill
    \begin{minipage}[l]{1\linewidth}
        \centering
        \includegraphics[trim={0 0 0 0}, clip, width=1\linewidth]{images/download.png}
        \captionsetup{labelformat=empty}
        \caption{}
    \end{minipage}
    \setcounter{figure}{10}
    \caption{\textbf{Top to bottom}: Median m-EMR to m-TTE trade-off (median values with 25-75 IQR error bars); average IQ-loss (solid lines) with a standard deviation (shaded area) over 200 training epoches; average balanced accuracy (solid lines) with a standard deviation (shaded area) over 200 training epoches.}
    \label{fig:results_cp}
\end{figure*}

\begin{table}[ht]
\centering
\caption{Balanced accuracy (mean \(\pm\) two standard deviations) for the change-point examples. Bold numbers indicate top-3 values}
\vskip 5pt
\begin{tabular}{lcccc}
\hline
                          & CP1         & CP2         & CP3         \\
\hline
 Classifier               & 0.5067±0.02 & 0.5044±0.04 & 0.5022±0.04 \\
 Classifier-SMOTE         & 0.5034±0.02 & 0.5034±0.10 & 0.5009±0.04 \\
 IQS                      & 0.5000±0.00 & 0.5000±0.00 & 0.5000±0.00  \\
 IQS-SMOTE                & 0.7533±0.21 & \textbf{0.6575±0.18} & 0.6240±0.19  \\
 IQS-CS-SMOTE             & 0.8128±0.18 & 0.6403±0.10 & 0.6965±0.04 \\
 Model-based IQS          & \textbf{0.9022±0.33} & 0.5496±0.12 & 0.6109±0.11 \\
 Model-based IQS-SMOTE    & 0.8669±0.27 & \textbf{0.6766±0.17} & 0.7884±0.26 \\
 Model-based IQS-CS-SMOTE & 0.8458±0.12 & 0.6347±0.11 & \textbf{0.8275±0.09} \\
 DO-IQS                   & \textbf{0.9538±0.06} & 0.5660±0.09 & \textbf{0.8479±0.10}  \\
 DO-IQS-LB                & \textbf{0.9567±0.02} & \textbf{0.7748±0.03} & \textbf{0.8596±0.05} \\
\hline
\end{tabular}
\label{tbl: cp_results}
\end{table}

\begin{figure*}[ht]
    \begin{minipage}[l]{0.5\linewidth}
        \centering
        \includegraphics[trim={0 0 0 2.1cm}, clip, width=1\linewidth]{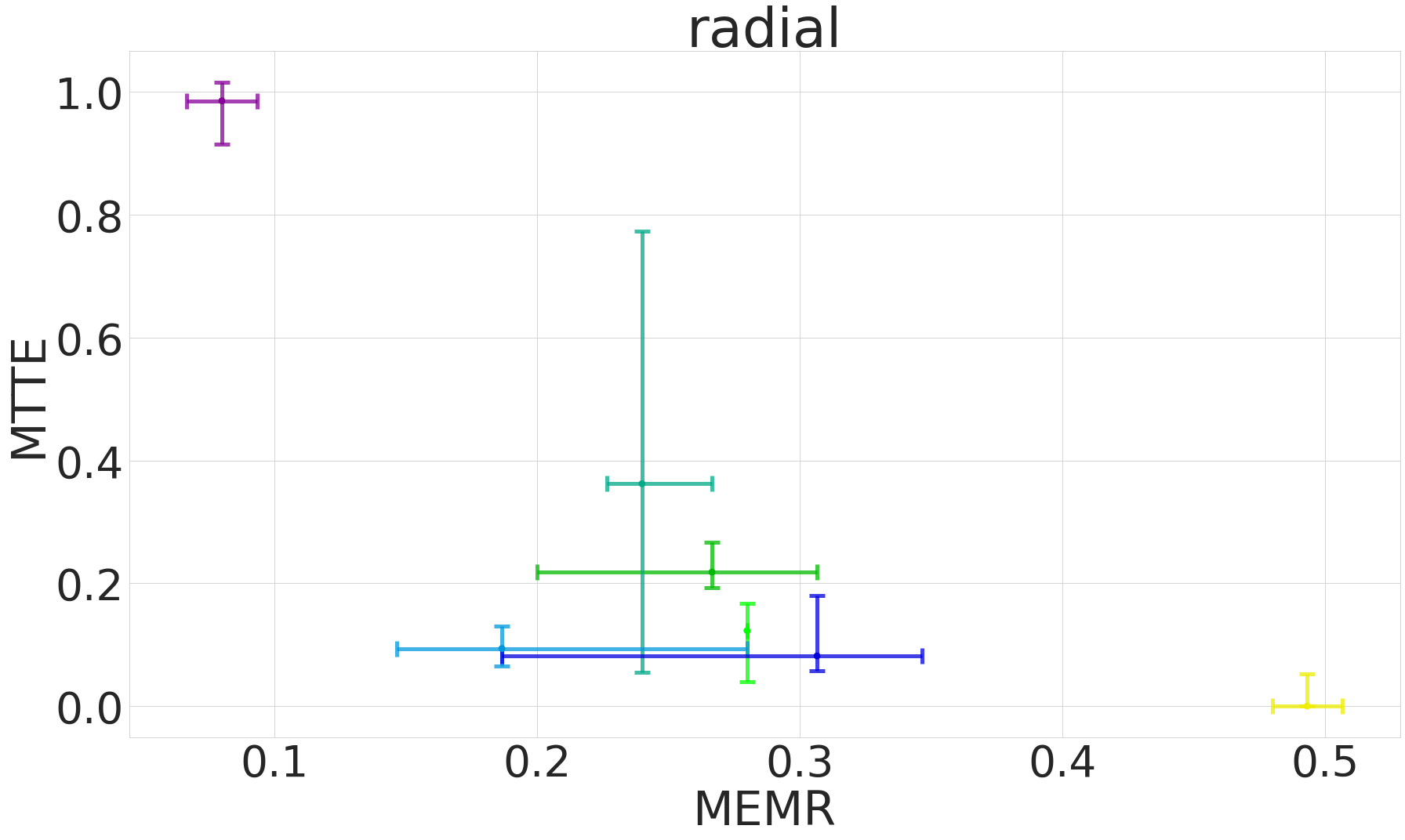}
        \captionsetup{labelformat=empty}
    \end{minipage}\hfill
    \begin{minipage}[l]{0.5\linewidth}
        \centering
        \includegraphics[trim={0 0 0 2.1cm}, clip, width=1\linewidth]{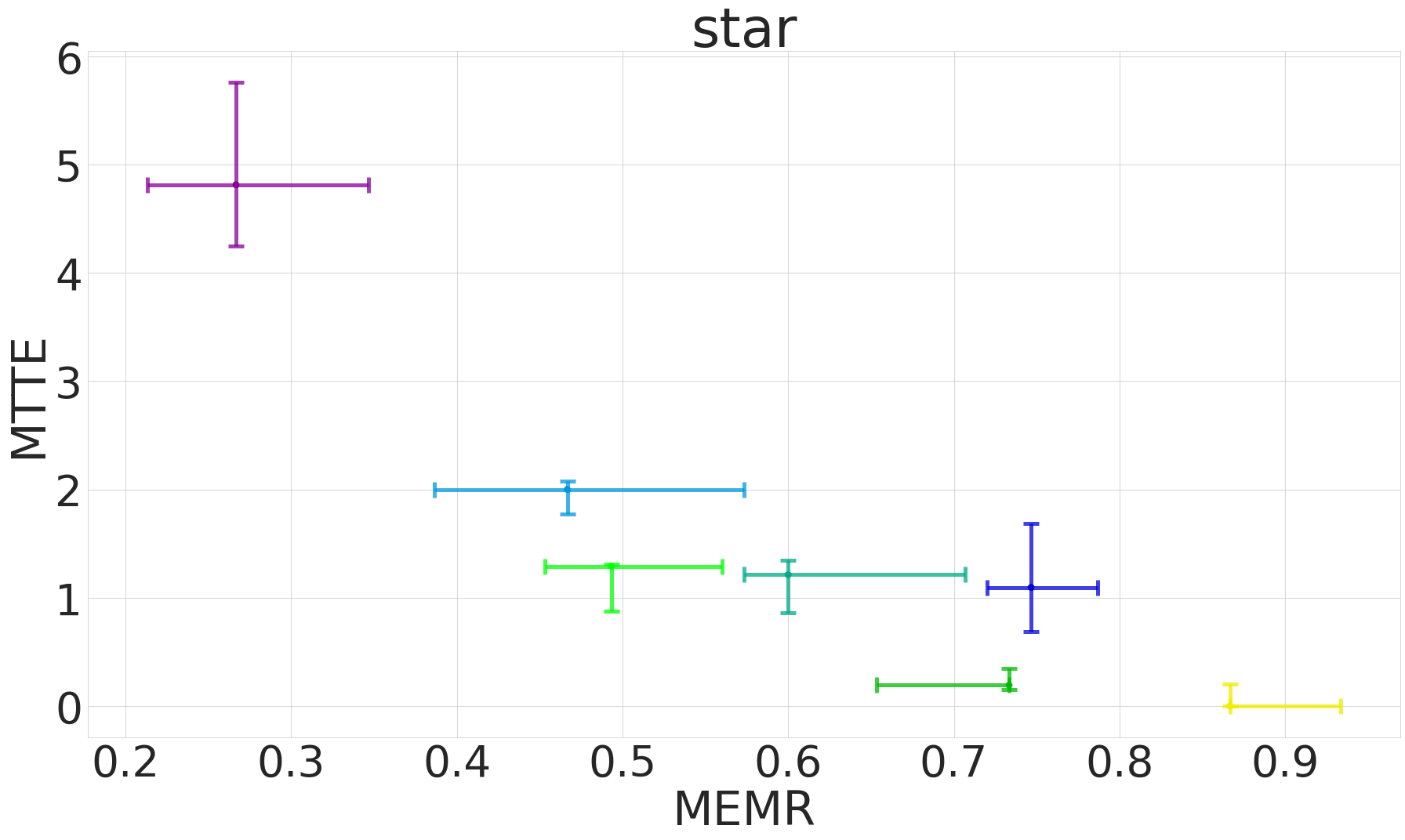}
        \captionsetup{labelformat=empty}
    \end{minipage}\hfill
    \begin{minipage}[l]{0.5\linewidth}
        \centering
        \includegraphics[trim={0 13cm 0 1.2cm}, clip, width=1\linewidth]{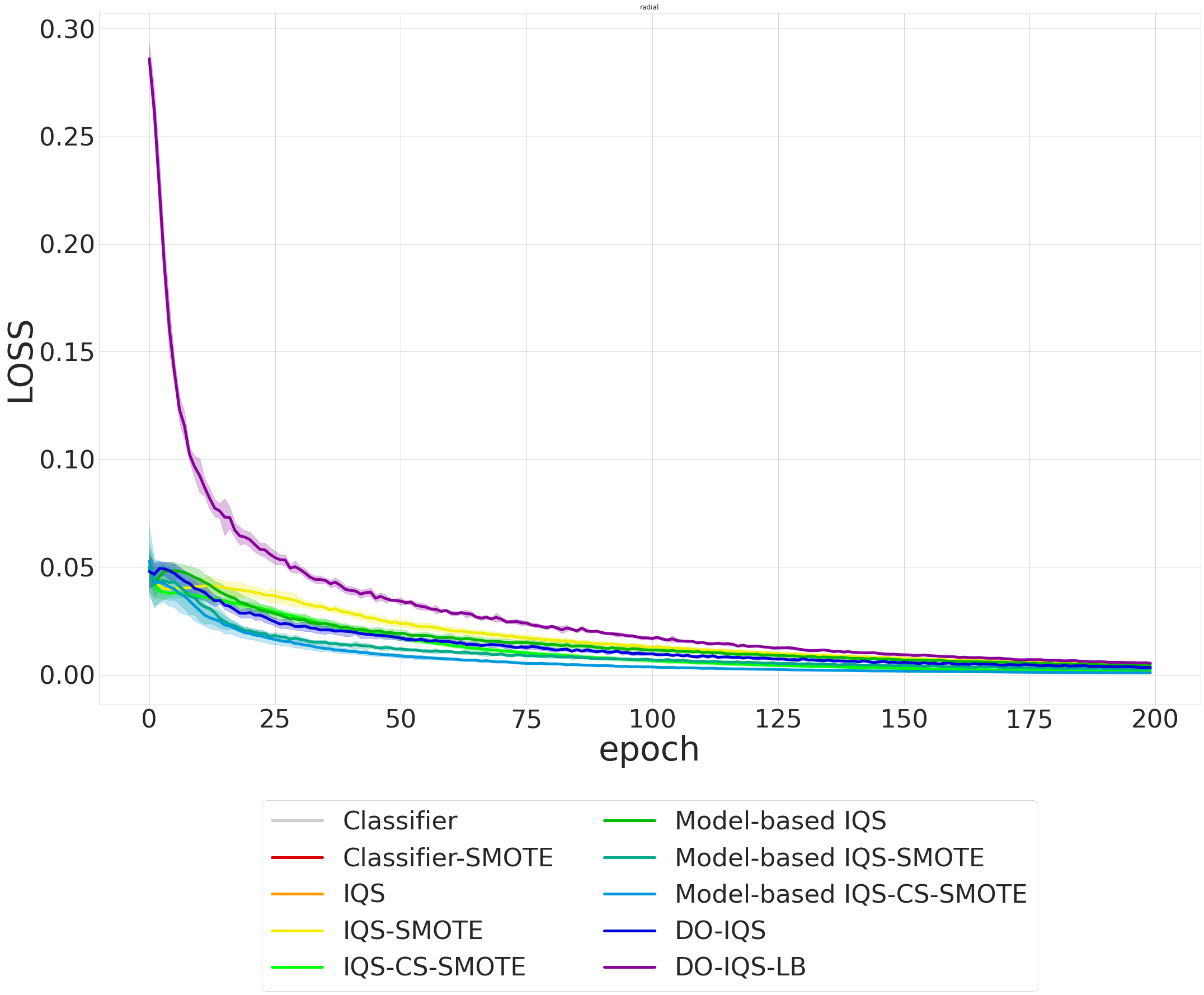}
        \captionsetup{labelformat=empty}
    \end{minipage}\hfill
    \begin{minipage}[l]{0.5\linewidth}
        \centering
        \includegraphics[trim={0 13cm 0 1.2cm}, clip, width=1\linewidth]{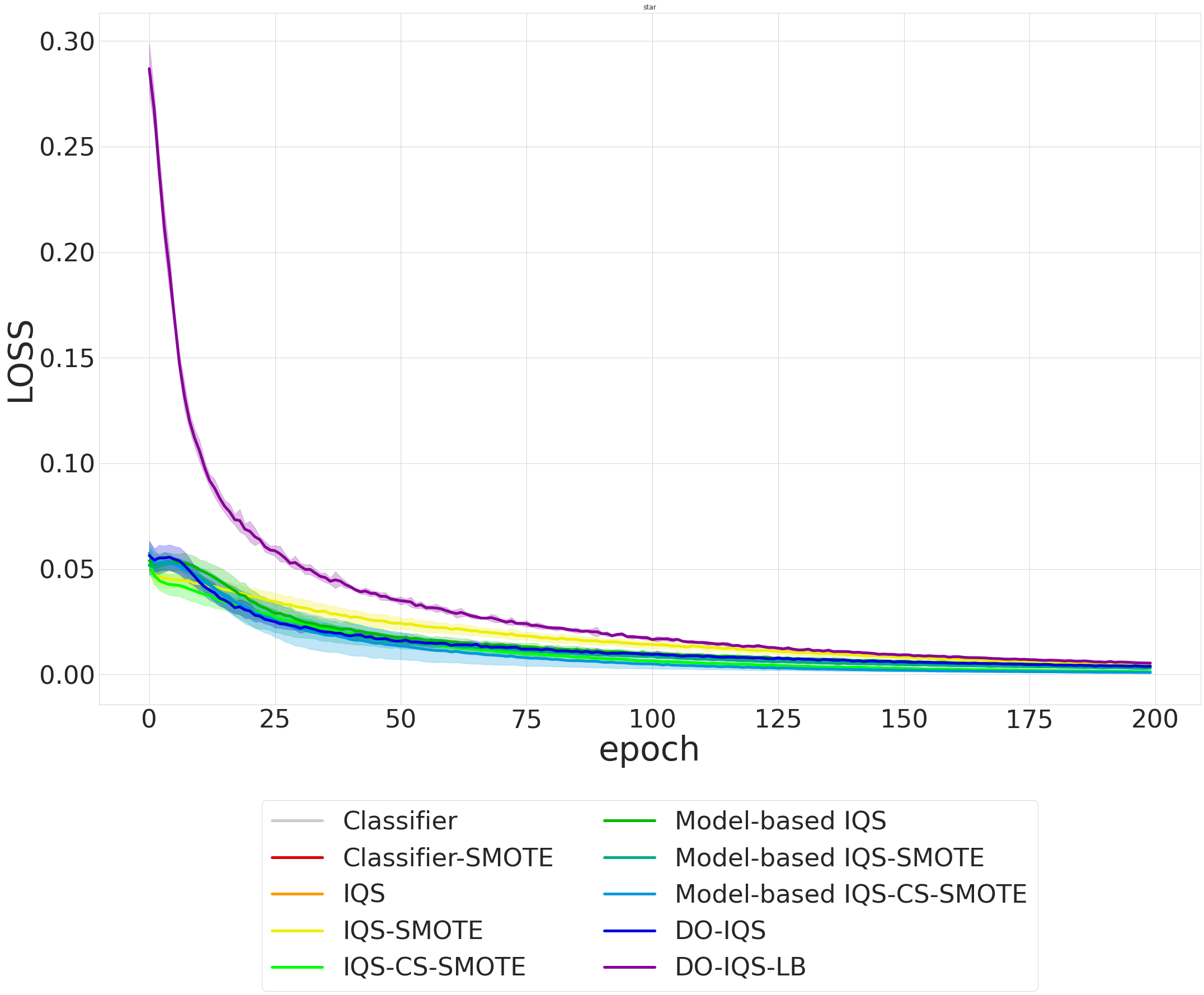}
        \captionsetup{labelformat=empty}
    \end{minipage}\hfill
    % \captionstyle{centerlast}
    \begin{minipage}[l]{0.5\linewidth}
        \centering
        \includegraphics[trim={0 13cm 0 0.6cm}, clip, width=1\linewidth]{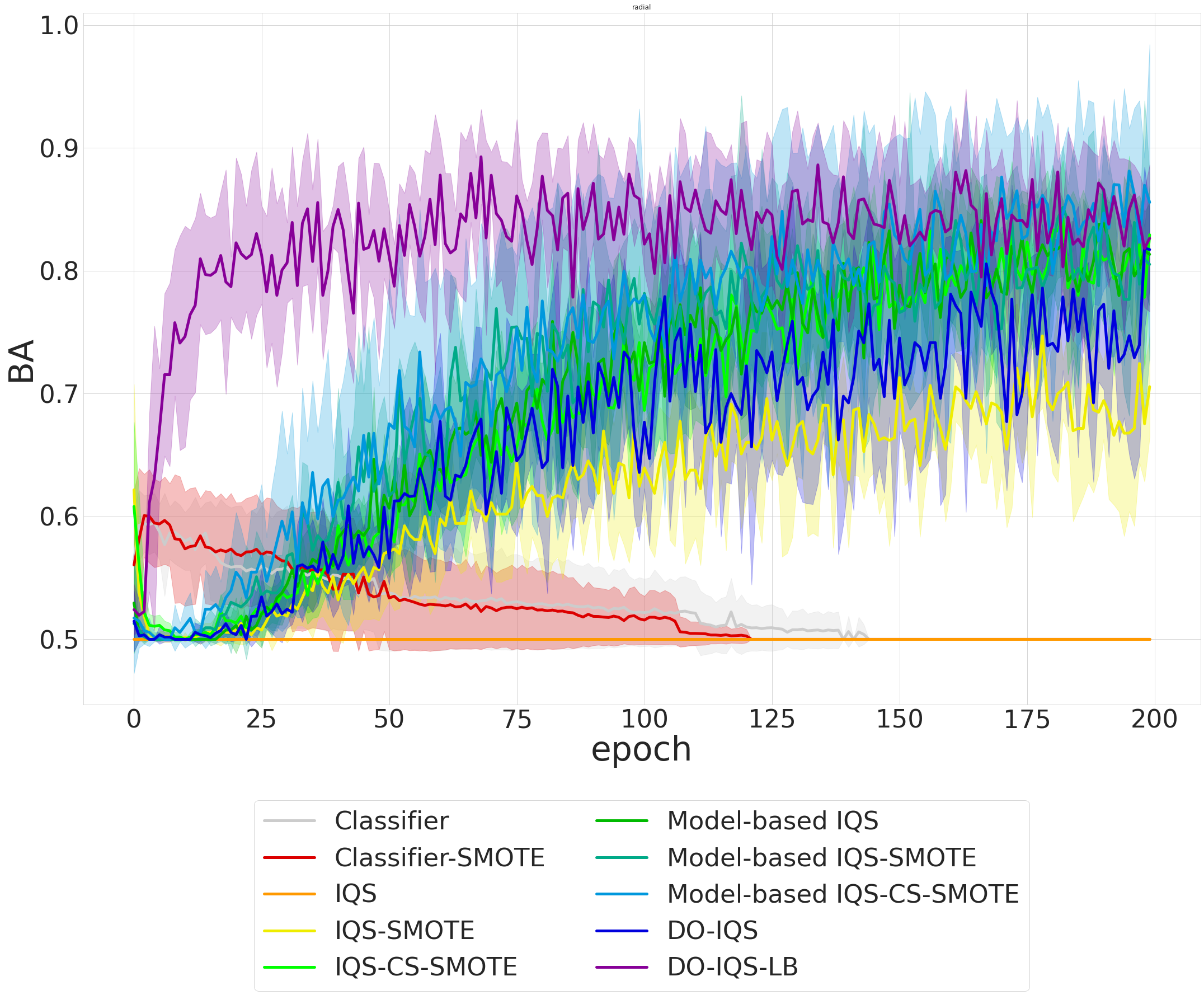}
        \captionsetup{labelformat=empty}
        \caption{Radial}
    \end{minipage}\hfill
    \begin{minipage}[l]{0.5\linewidth}
        \centering
        \includegraphics[trim={0 13cm 0 0.6cm}, clip, width=1\linewidth]{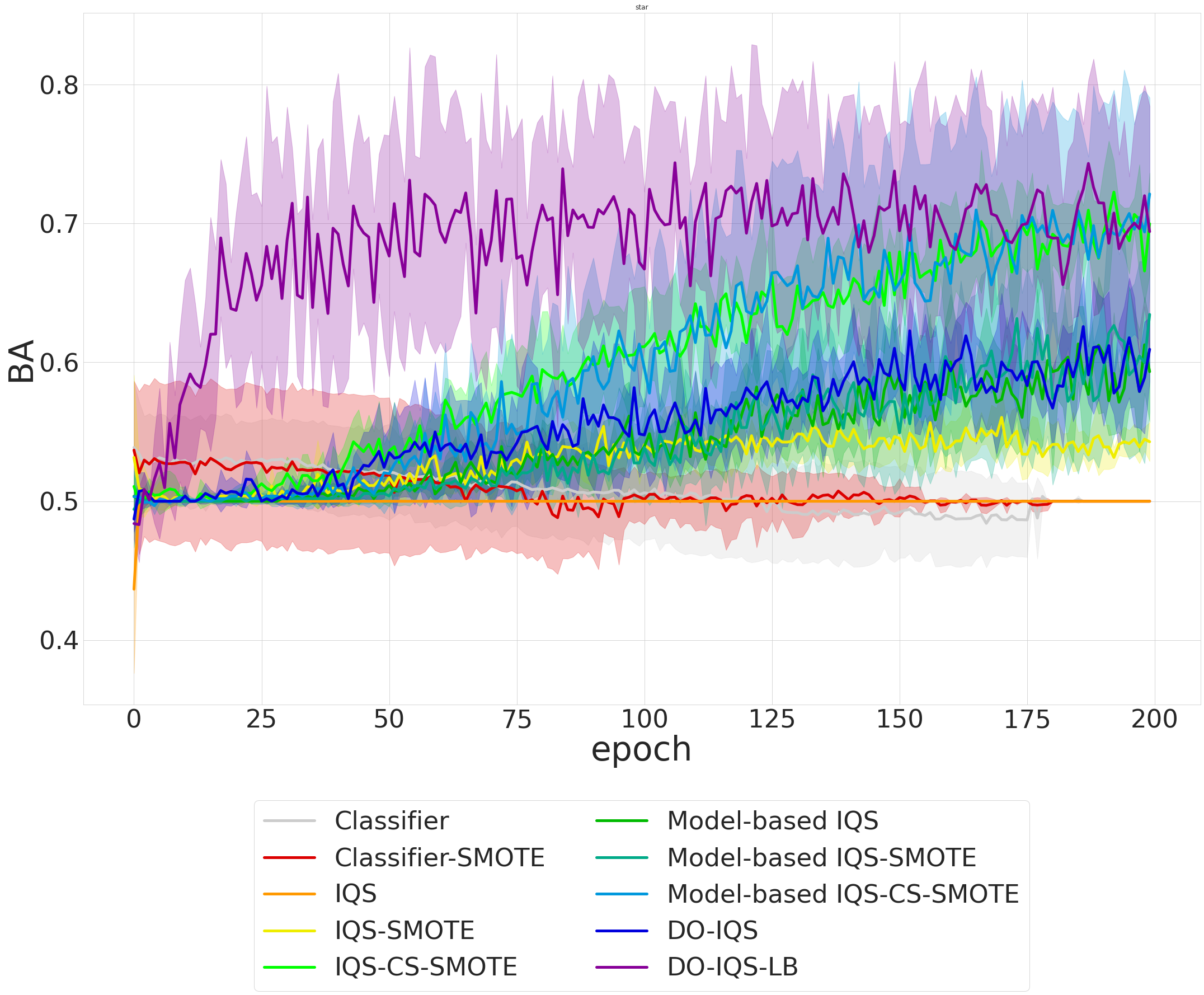}
        \captionsetup{labelformat=empty}
        \caption{Star}
    \end{minipage}\hfill
    \begin{minipage}[l]{1\linewidth}
        \centering
        \includegraphics[trim={0 0 0 0}, clip, width=1\linewidth]{images/download.png}
        \captionsetup{labelformat=empty}
        \caption{}
    \end{minipage}
    \vskip -15pt
    \setcounter{figure}{11}
    \caption{\textbf{Top to bottom}: Median m-EMR to m-TTE trade-off (median values with 25-75 IQR error bars); average IQ-loss (solid lines) with a standard deviation (shaded area) over 200 training epoches; average balanced accuracy (solid lines) with a standard deviation (shaded area) over 200 training epoches.}
    \label{fig:results_rad_star}
    \vskip 5pt
\end{figure*}

\begin{table}[ht]
\centering
\caption{Balanced accuracy (mean \(\pm\) two standard deviations) for the radial and star examples. Bold numbers indicate top-3 values}
\vskip 5pt
 \begin{tabular}{lcc}

\hline
                          & radial      & star        \\
\hline
 Classifier               & 0.5766±0.06 & 0.5088±0.05 \\
 Classifier-SMOTE         & 0.5751±0.07 & 0.5000±0.07    \\
 IQS                      & 0.5000±0.00 & 0.5000±0.00     \\
 IQS-SMOTE                & 0.7523±0.06 & 0.5655±0.06 \\
 IQS-CS-SMOTE             & 0.8543±0.02 & \textbf{0.7252±0.08} \\
 Model-based IQS          & 0.8589±0.05 & 0.6287±0.09 \\
 Model-based IQS-SMOTE    & \textbf{0.8734±0.04} & 0.6886±0.09 \\
 Model-based IQS-CS-SMOTE & \textbf{0.9047±0.08} & \textbf{0.7260±0.09}  \\
 DO-IQS                   & 0.8448±0.07 & 0.6194±0.06 \\
 DO-IQS-LB                & \textbf{0.9202±0.04} & \textbf{0.7170±0.07}  \\
\hline
\end{tabular}
\label{tbl: rad_star_results}
\end{table}

\subsection{Q-function and stopping boundaries approximation results}
We further compare the recovered Q-functions and stopping boundaries by presenting the heat maps of \(Q(s,a=0)\), \(Q(s,a=1)\) and \(Q(s,a=0)-Q(s,a=1)\) evaluated at time\(t=25\) as well as the resulting stopping boundaries per time-step. The results for the Radial and Star example are presented in Figures \ref{fig:radial_heatmaps} and \ref{fig:star_heatmaps}, respectively. The model was trained for 200 epoches with a 490 training and 210 validation paths, and the Q-values were normalised with the centred normalisation (centre at 0 to preserve the stopping rule) for ease of the heatmap comparison to aid visualisation.

\begin{figure*}[ht]
    \begin{minipage}[l]{0.2\linewidth}
        \centering
        \includegraphics[trim={0 0 10cm 0}, clip, width=1\linewidth]{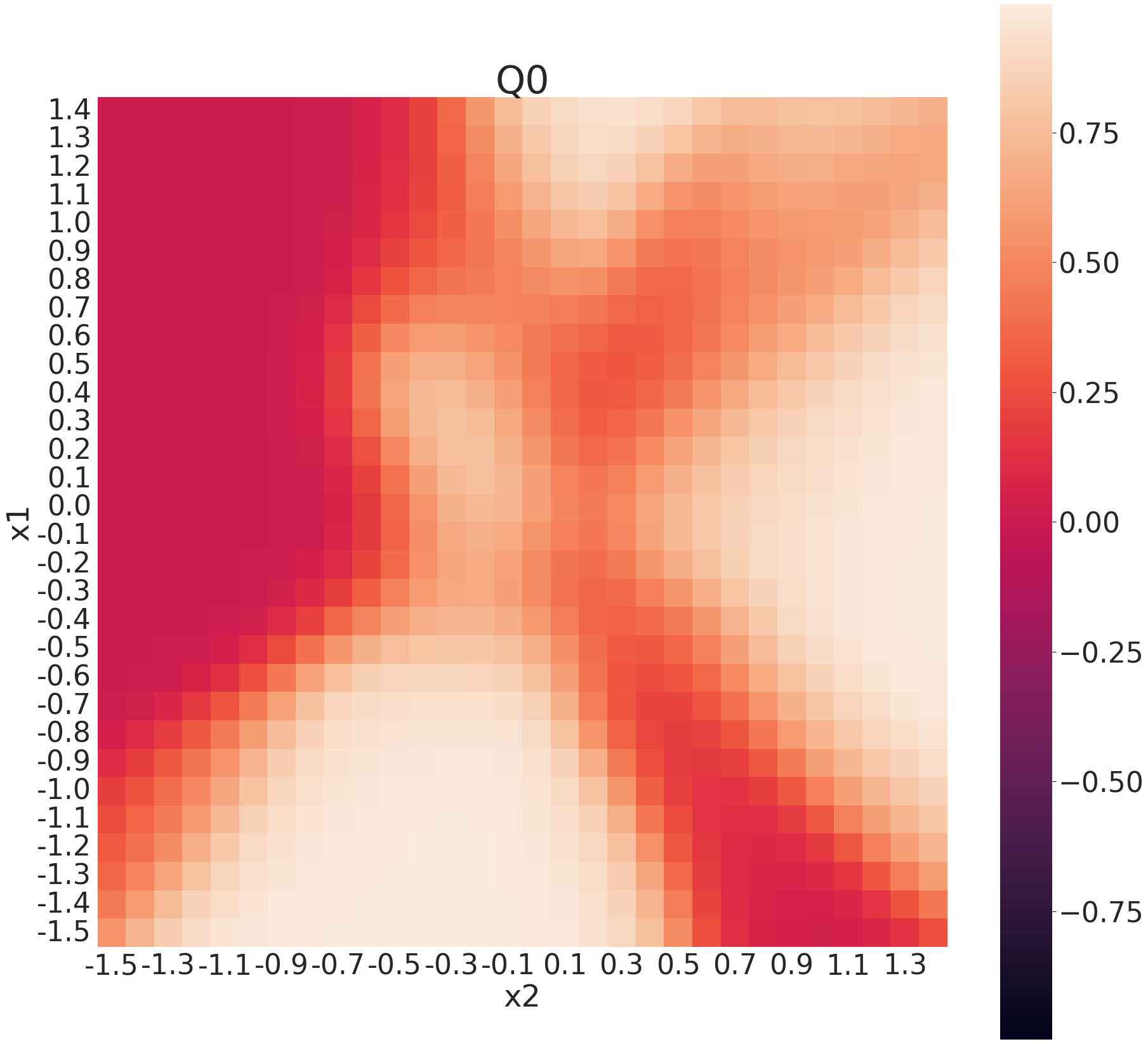}
        \captionsetup{labelformat=empty}
    \end{minipage}\hfill
    \begin{minipage}[l]{0.2\linewidth}
        \centering
        \includegraphics[trim={0 0 10cm 0}, clip, width=1\linewidth]{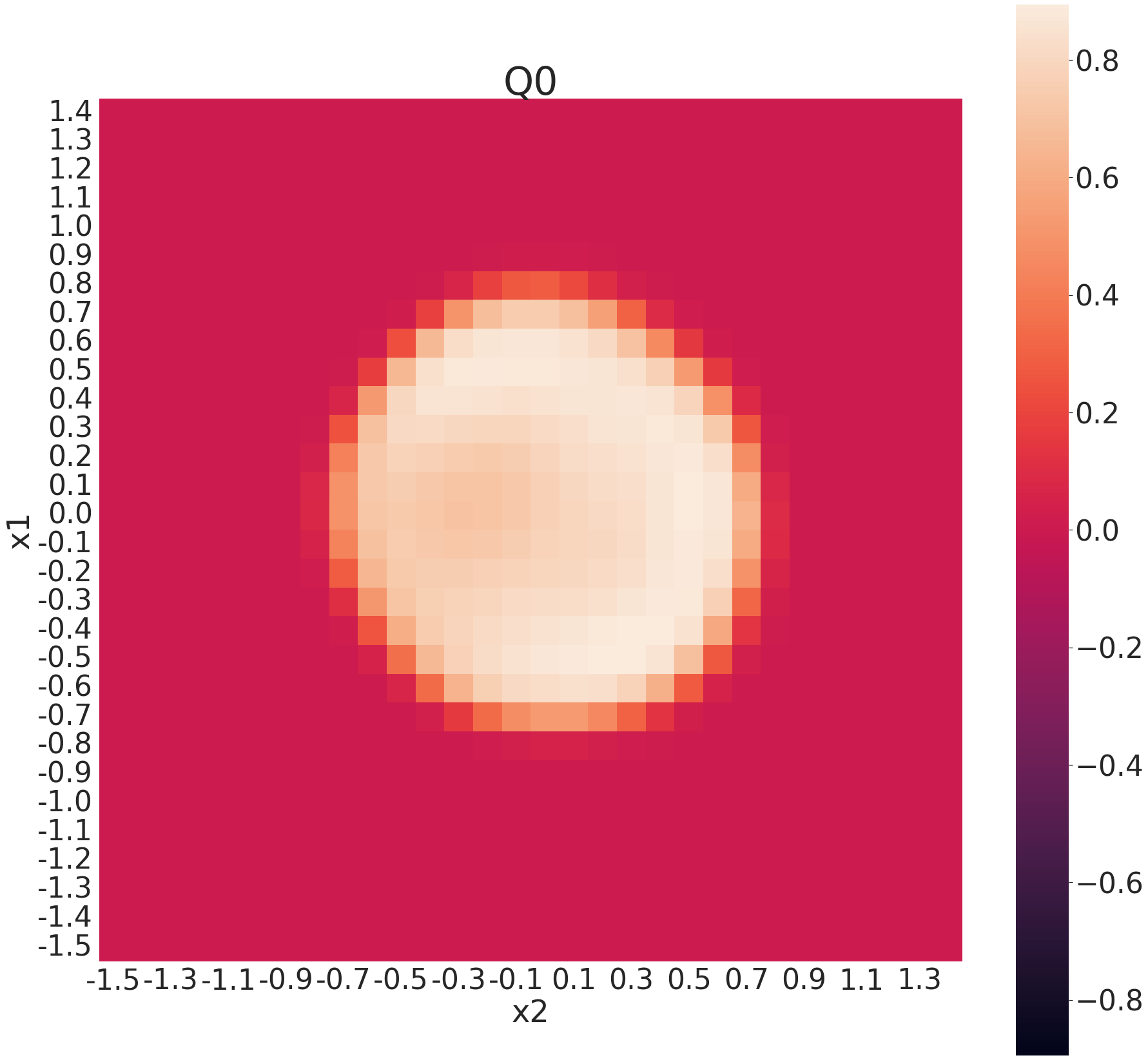}
        \captionsetup{labelformat=empty}
    \end{minipage}\hfill
    \begin{minipage}[l]{0.2\linewidth}
        \centering
        \includegraphics[trim={0 0 10cm 0}, clip, width=1\linewidth]{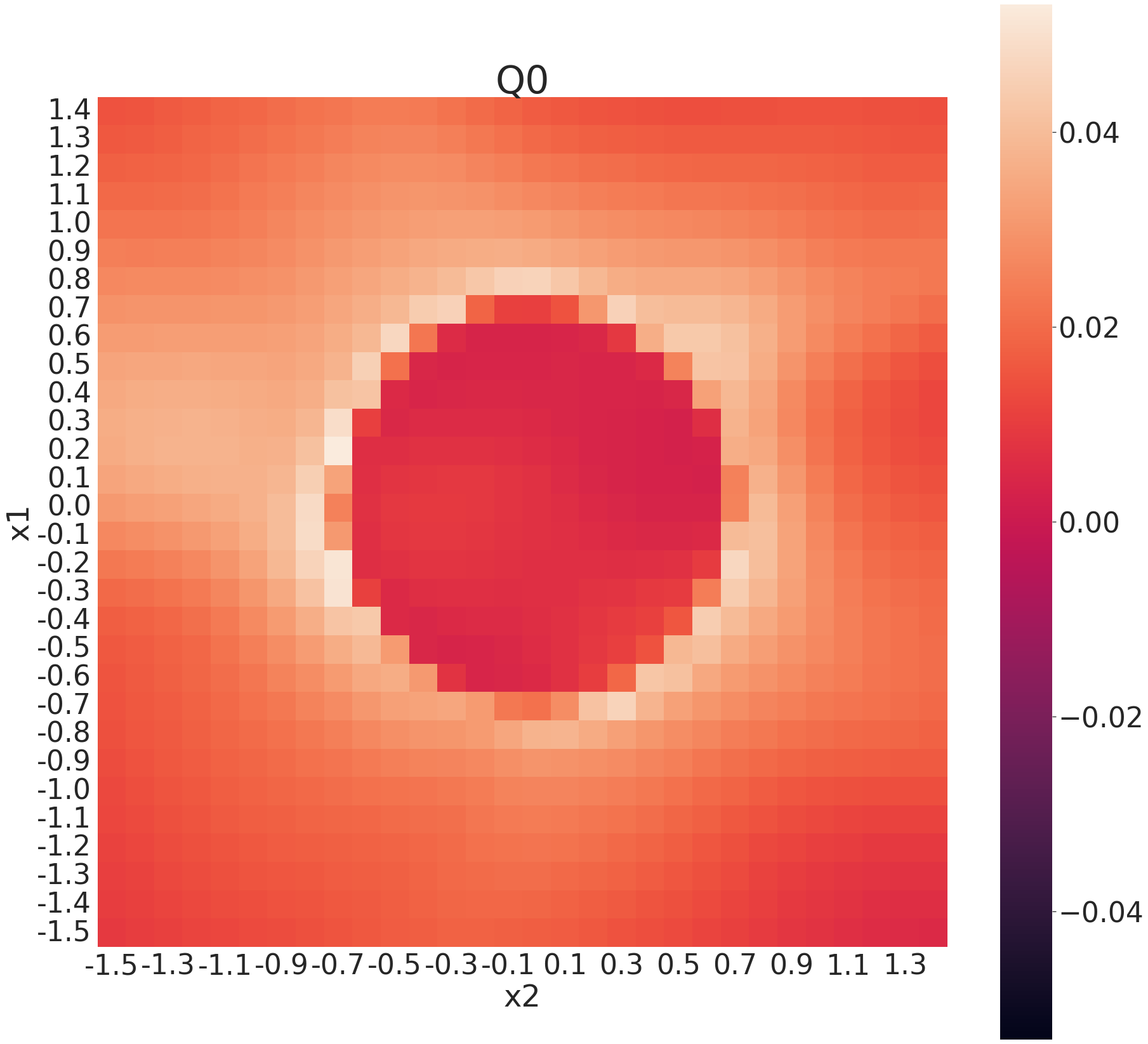}
        \captionsetup{labelformat=empty}
    \end{minipage}\hfill
    \begin{minipage}[l]{0.2\linewidth}
        \centering
        \includegraphics[trim={0 0 10cm 0}, clip, width=1\linewidth]{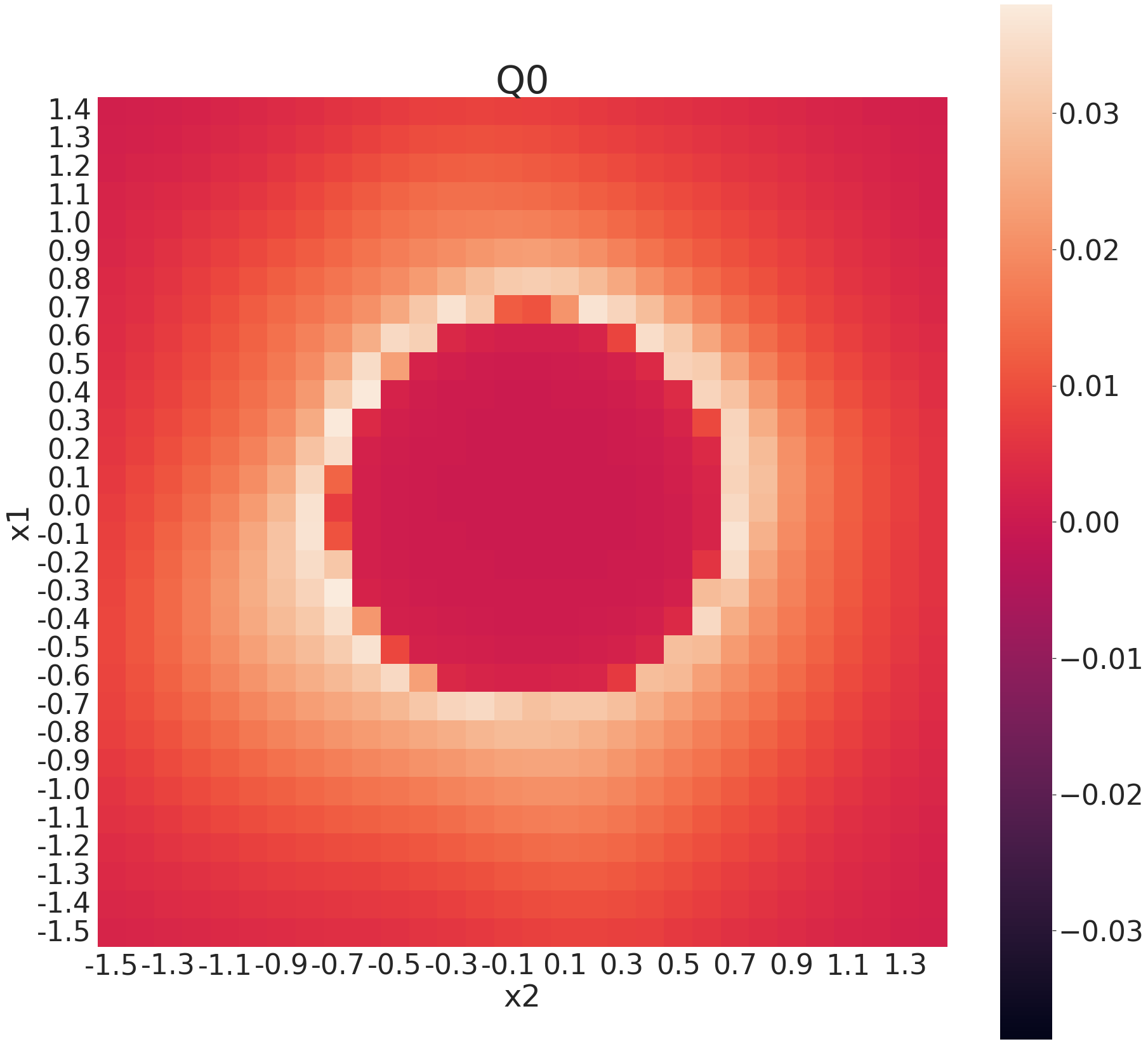}
        \captionsetup{labelformat=empty}
    \end{minipage}\hfill
    \begin{minipage}[l]{0.2\linewidth}
        \centering
        \includegraphics[trim={0 0 10cm 0}, clip, width=1\linewidth]{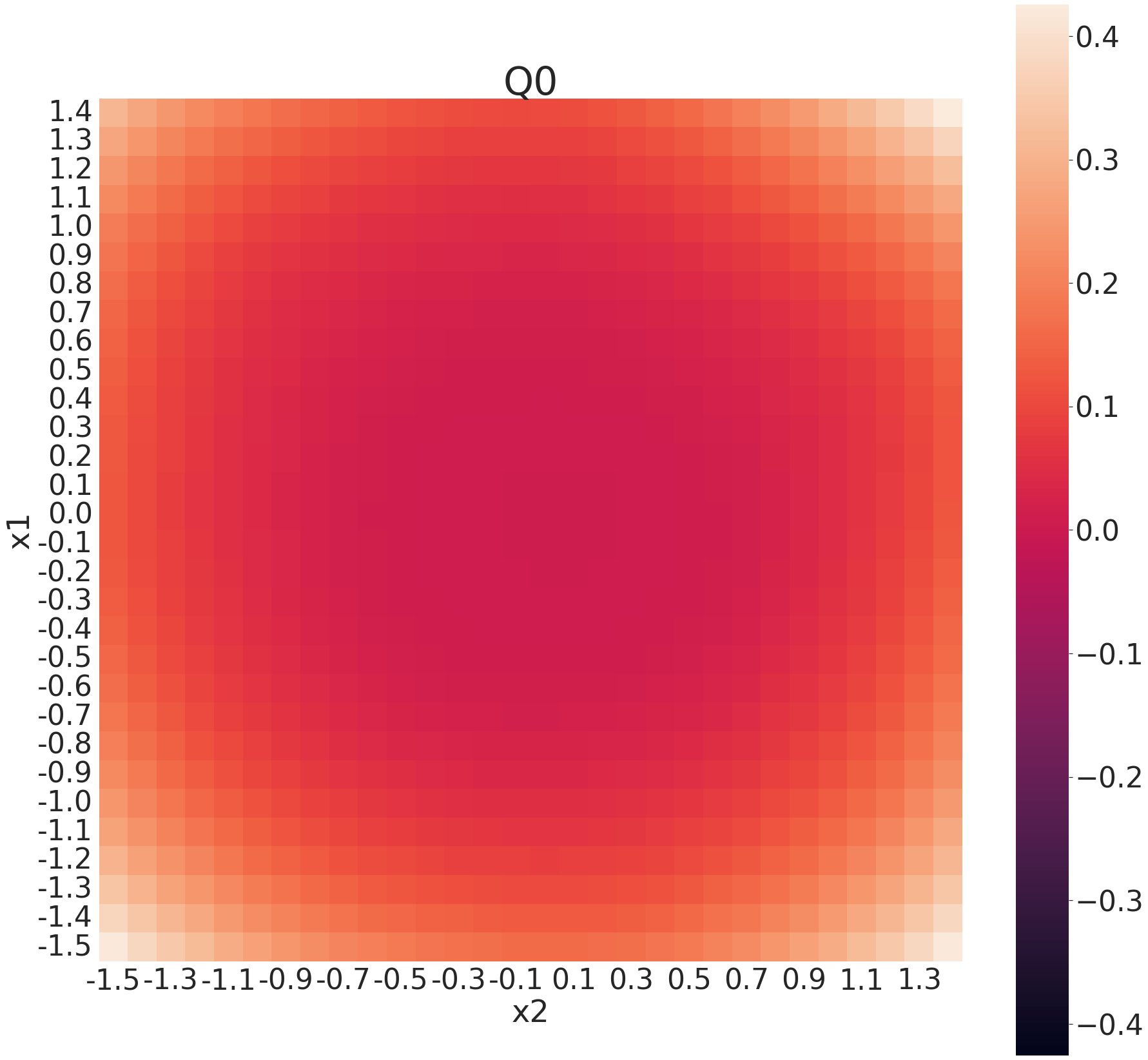}
        \captionsetup{labelformat=empty}
    \end{minipage}\hfill

    \begin{minipage}[l]{0.2\linewidth}
        \centering
        \includegraphics[trim={0 0 10cm 0}, clip, width=1\linewidth]{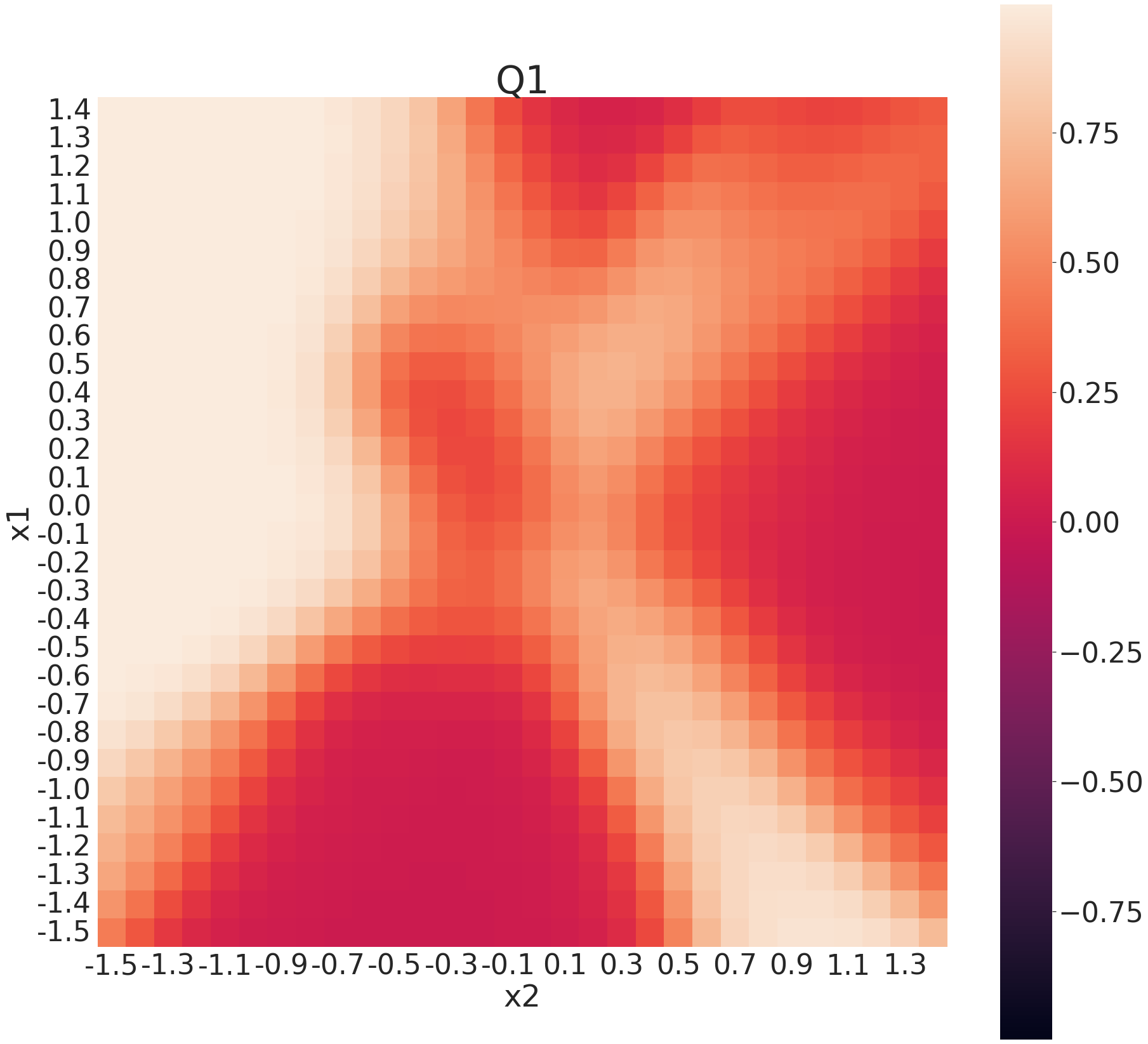}
        \captionsetup{labelformat=empty}
    \end{minipage}\hfill
    \begin{minipage}[l]{0.2\linewidth}
        \centering
        \includegraphics[trim={0 0 10cm 0}, clip, width=1\linewidth]{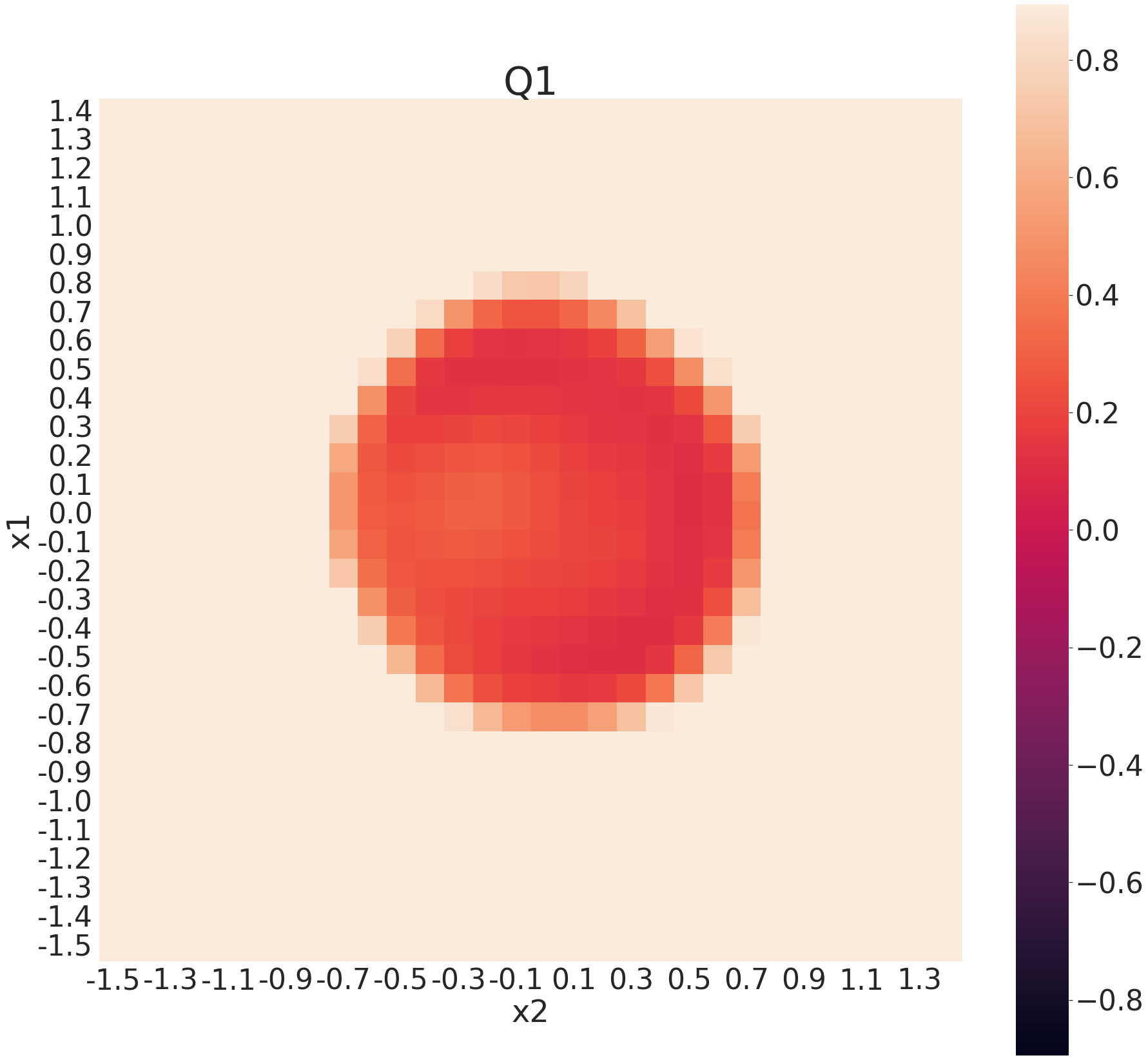}
        \captionsetup{labelformat=empty}
    \end{minipage}\hfill
    \begin{minipage}[l]{0.2\linewidth}
        \centering
        \includegraphics[trim={0 0 10cm 0}, clip, width=1\linewidth]{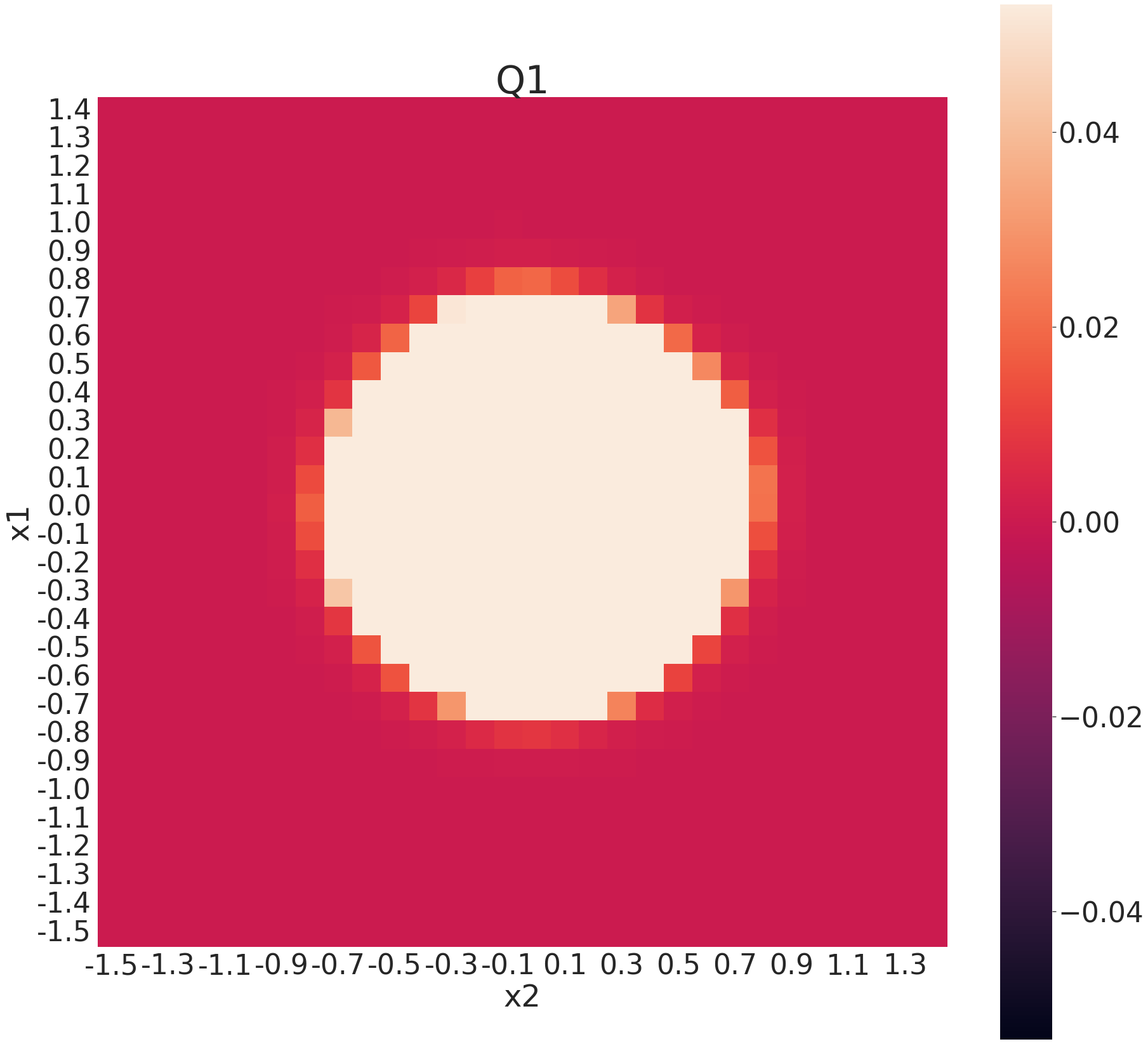}
        \captionsetup{labelformat=empty}
    \end{minipage}\hfill
    \begin{minipage}[l]{0.2\linewidth}
        \centering
        \includegraphics[trim={0 0 10cm 0}, clip, width=1\linewidth]{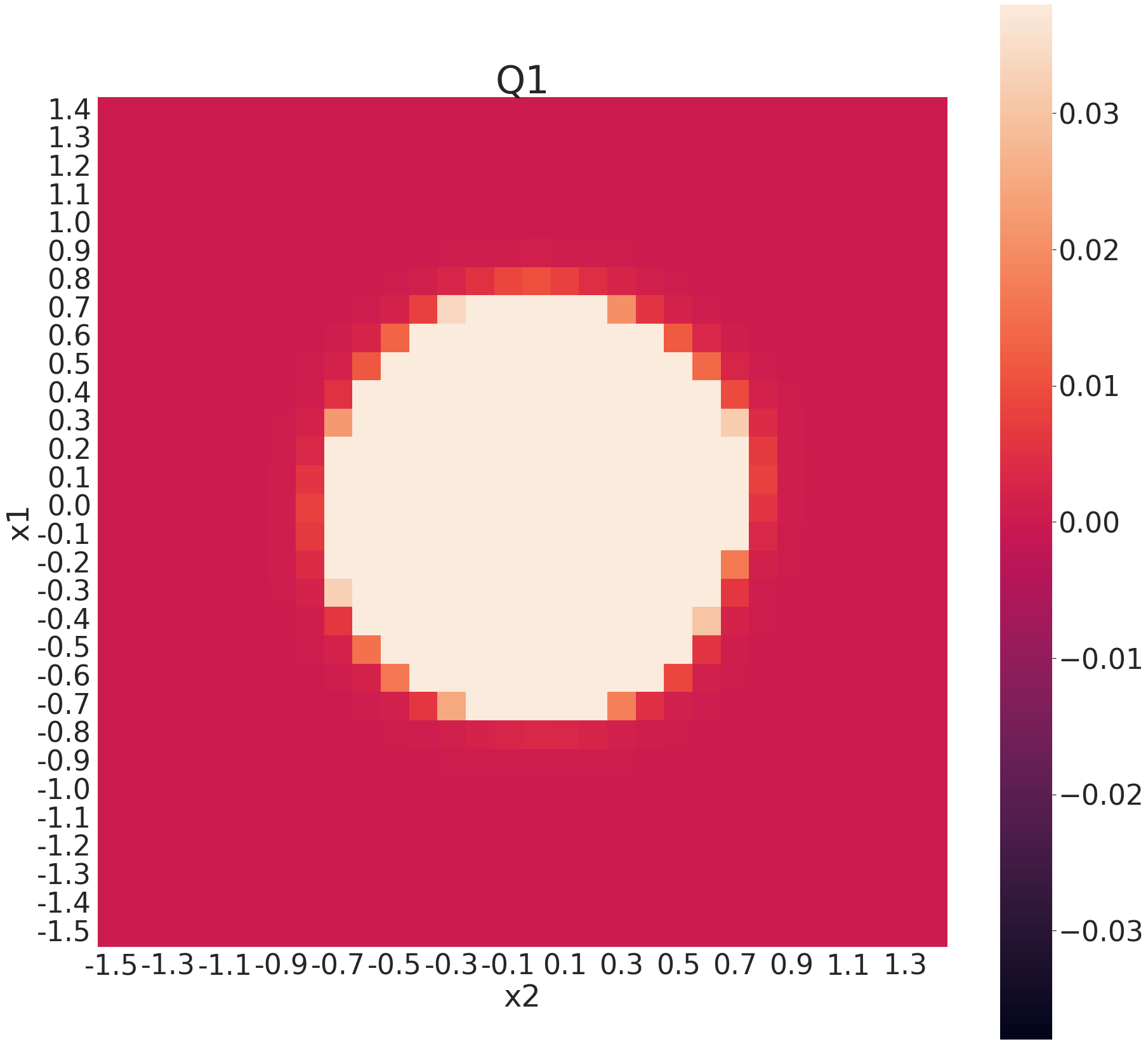}
        \captionsetup{labelformat=empty}
    \end{minipage}\hfill
    \begin{minipage}[l]{0.2\linewidth}
        \centering
        \includegraphics[trim={0 0 10cm 0}, clip, width=1\linewidth]{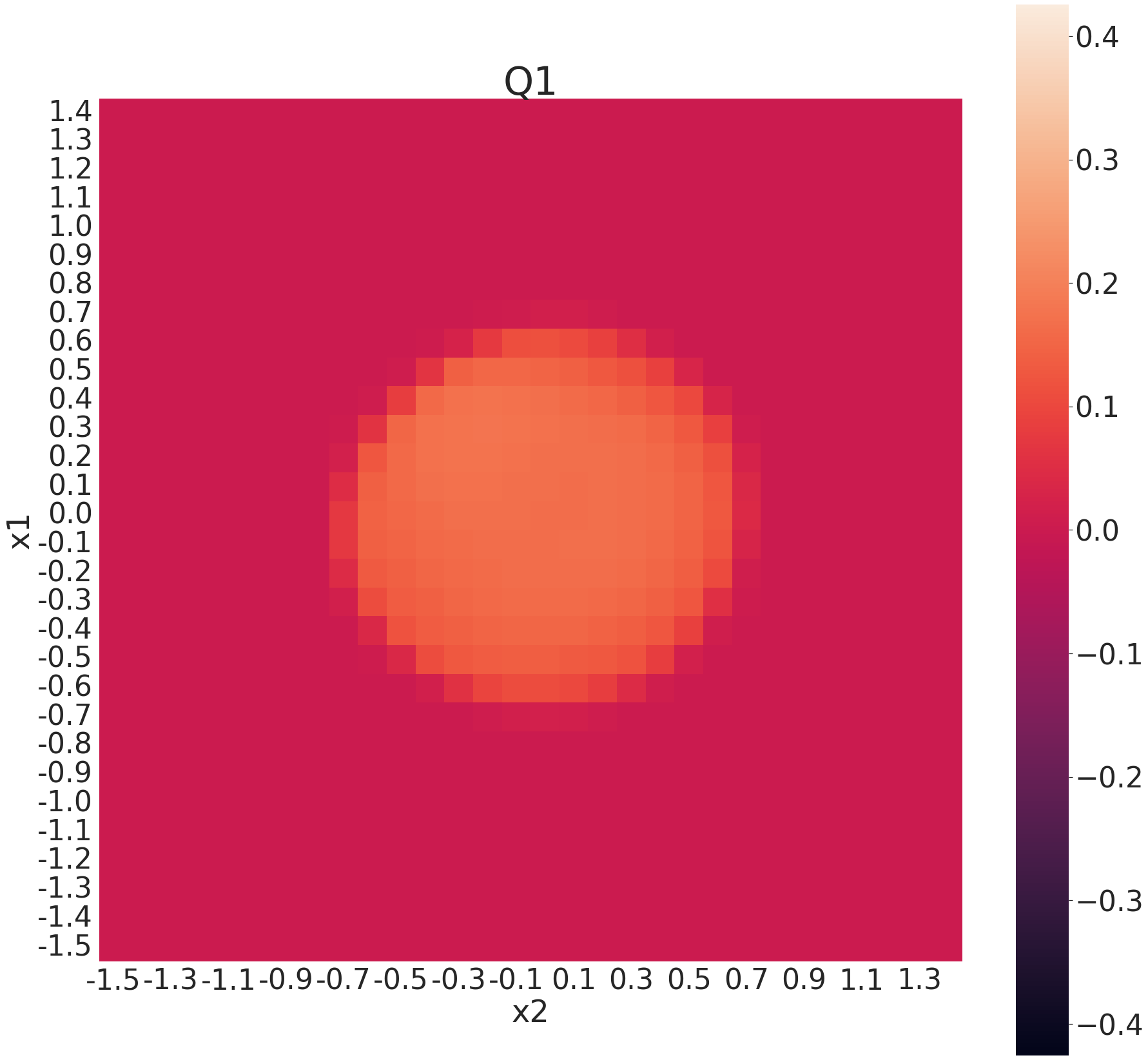}
        \captionsetup{labelformat=empty}
    \end{minipage}\hfill

    \begin{minipage}[l]{0.2\linewidth}
        \centering
        \includegraphics[trim={0 0 10cm 0}, clip, width=1\linewidth]{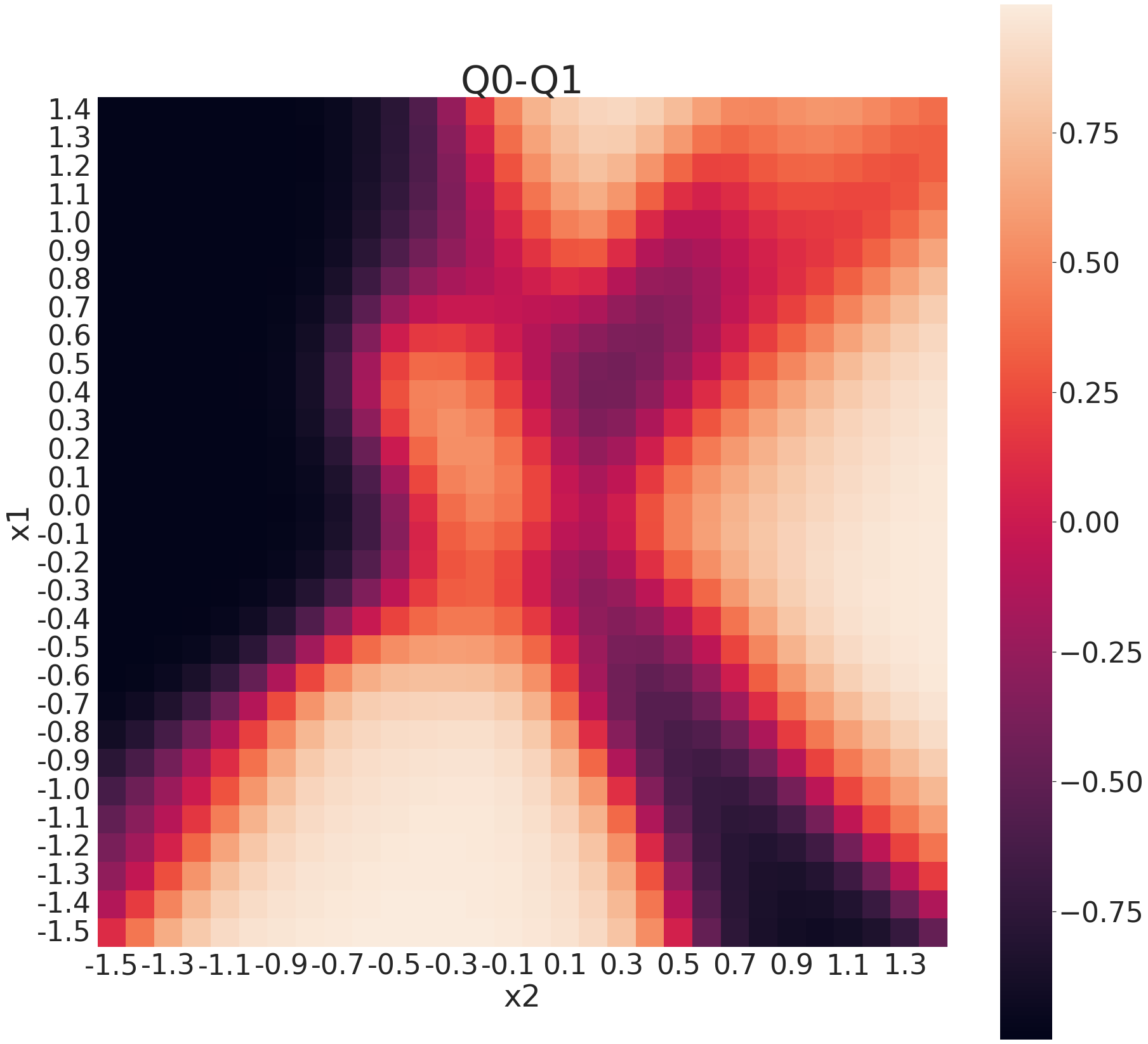}
        \captionsetup{labelformat=empty}
    \end{minipage}\hfill
    \begin{minipage}[l]{0.2\linewidth}
        \centering
        \includegraphics[trim={0 0 10cm 0}, clip, width=1\linewidth]{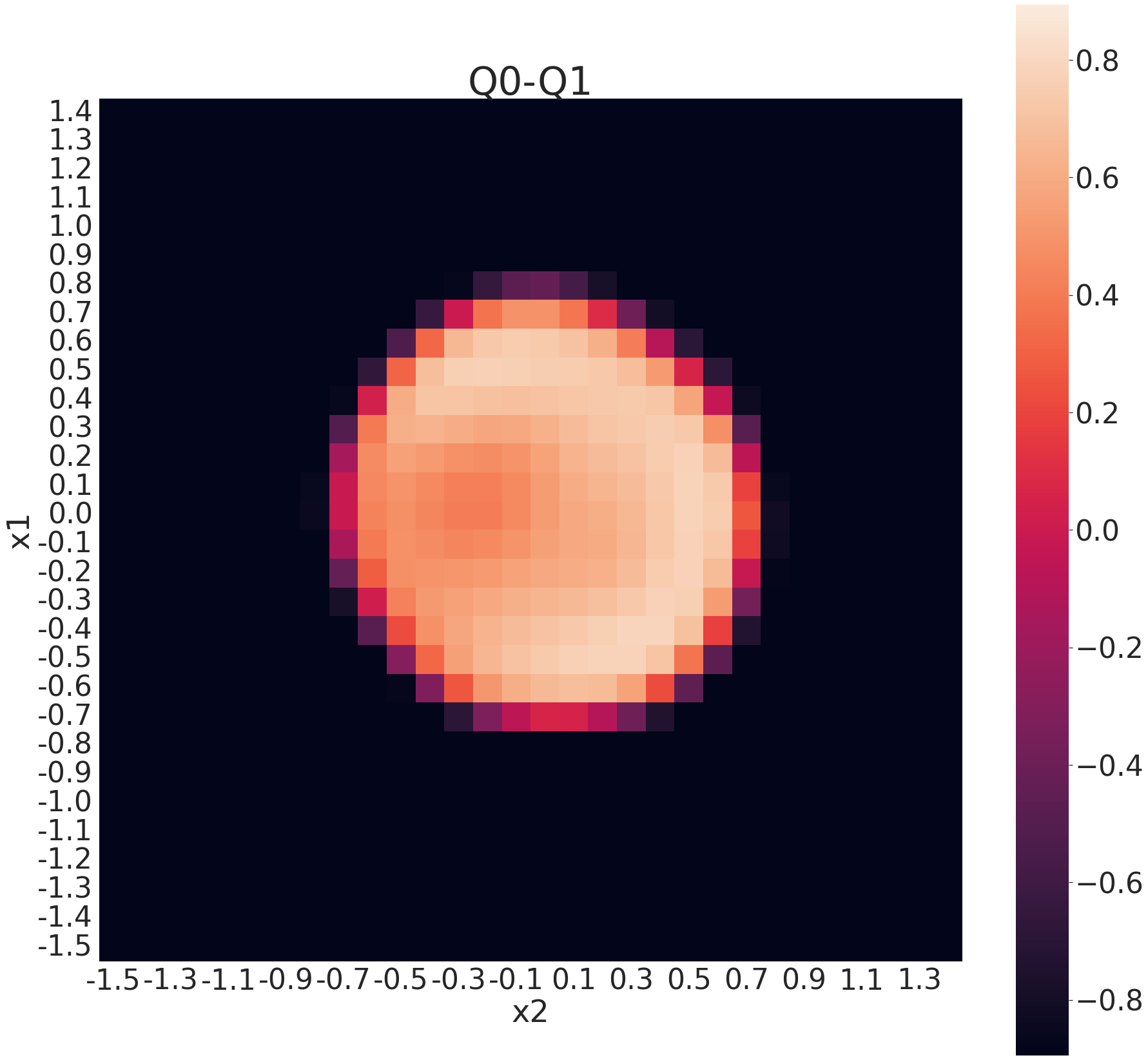}
        \captionsetup{labelformat=empty}
    \end{minipage}\hfill
    \begin{minipage}[l]{0.2\linewidth}
        \centering
        \includegraphics[trim={0 0 10cm 0}, clip, width=1\linewidth]{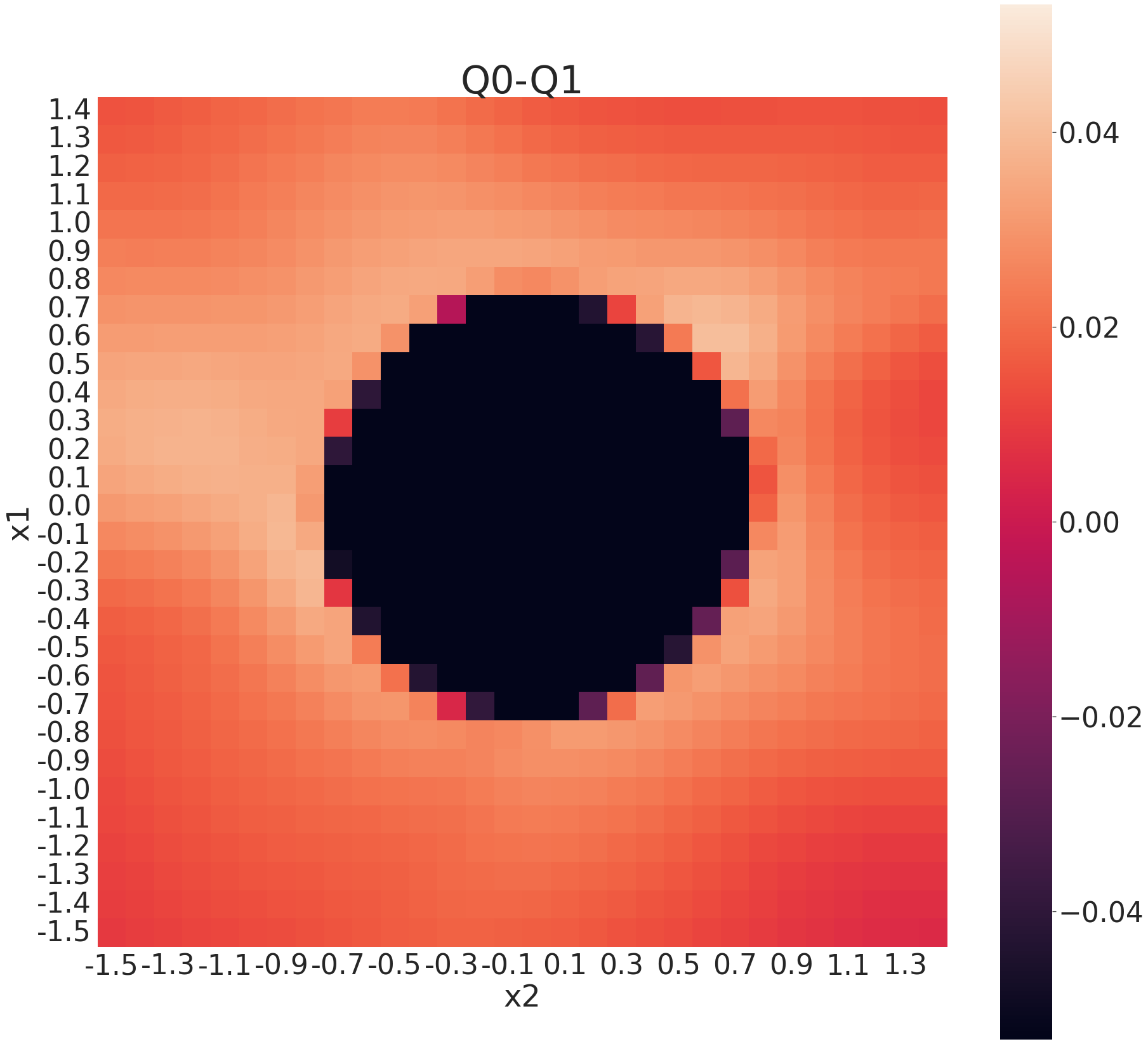}
        \captionsetup{labelformat=empty}
    \end{minipage}\hfill
    \begin{minipage}[l]{0.2\linewidth}
        \centering
        \includegraphics[trim={0 0 10cm 0}, clip, width=1\linewidth]{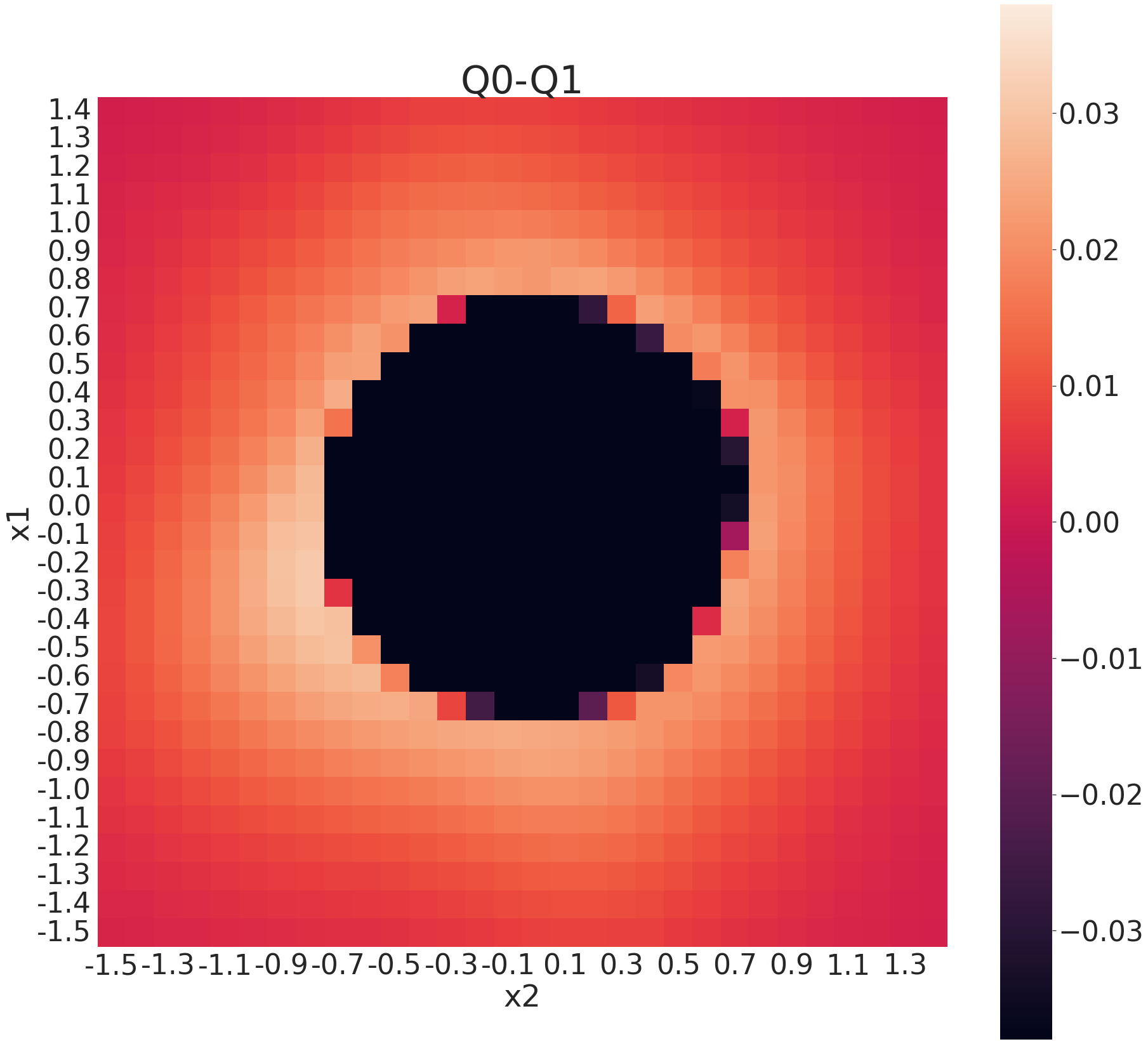}
        \captionsetup{labelformat=empty}
    \end{minipage}\hfill
    \begin{minipage}[l]{0.2\linewidth}
        \centering
        \includegraphics[trim={0 0 10cm 0}, clip, width=1\linewidth]{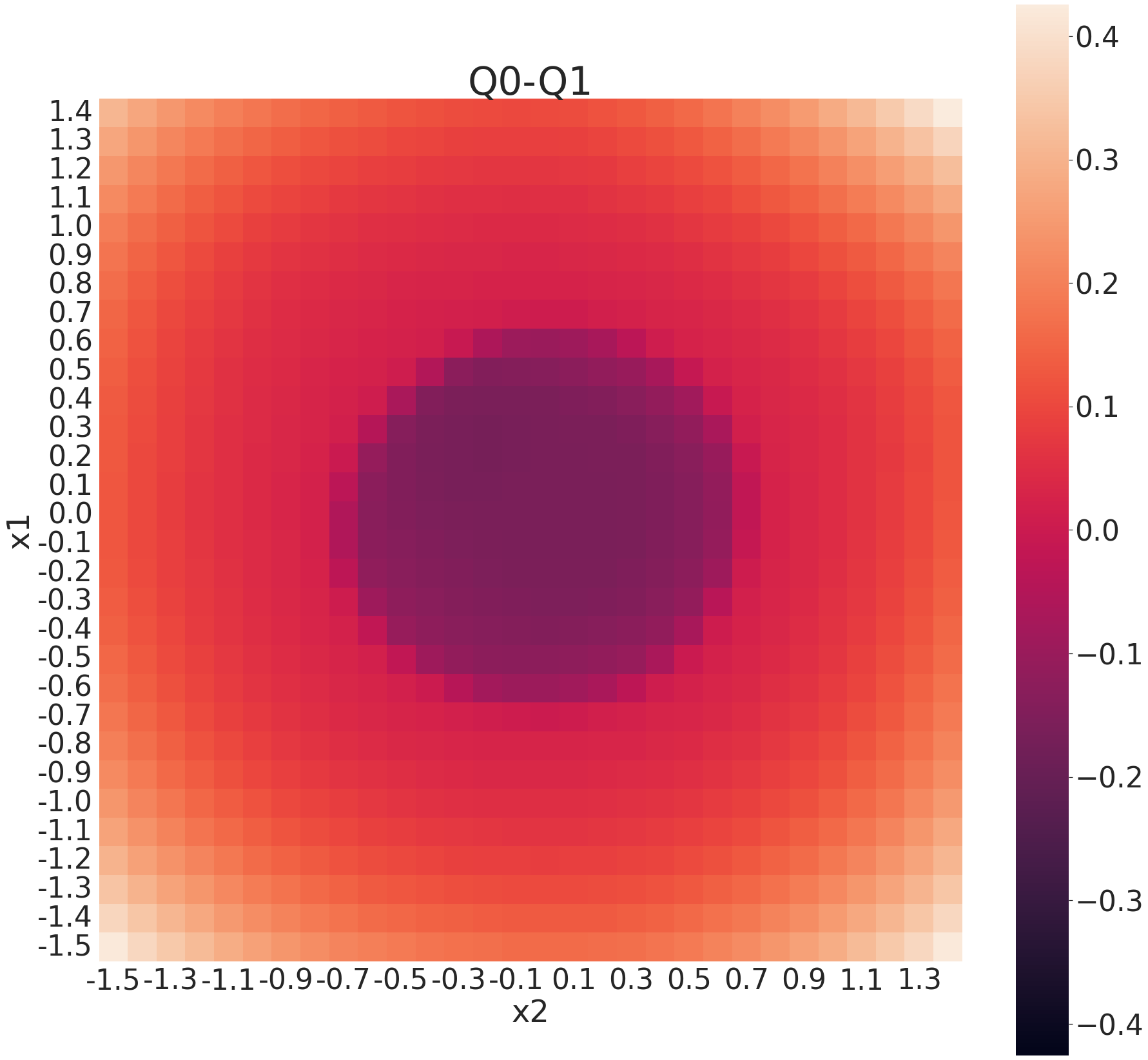}
        \captionsetup{labelformat=empty}
    \end{minipage}\hfill

    \begin{minipage}[l]{0.2\linewidth}
        \centering
        \includegraphics[trim={0 0 0 0}, clip, width=1\linewidth]{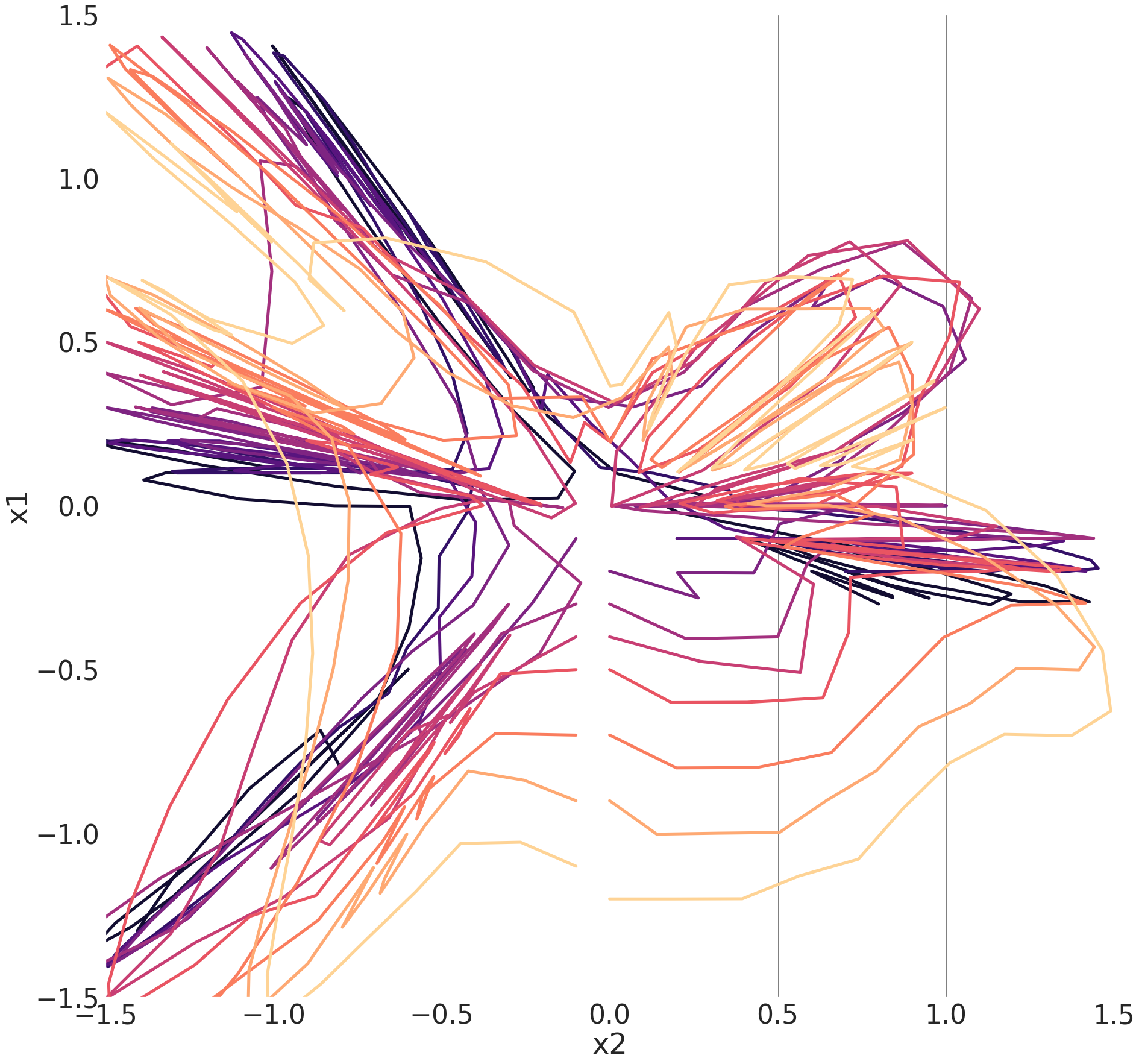}
        \captionsetup{labelformat=empty}
        \caption{Classifier}
    \end{minipage}\hfill
    \begin{minipage}[l]{0.2\linewidth}
        \centering
        \includegraphics[trim={0 0 0 0}, clip, width=1\linewidth]{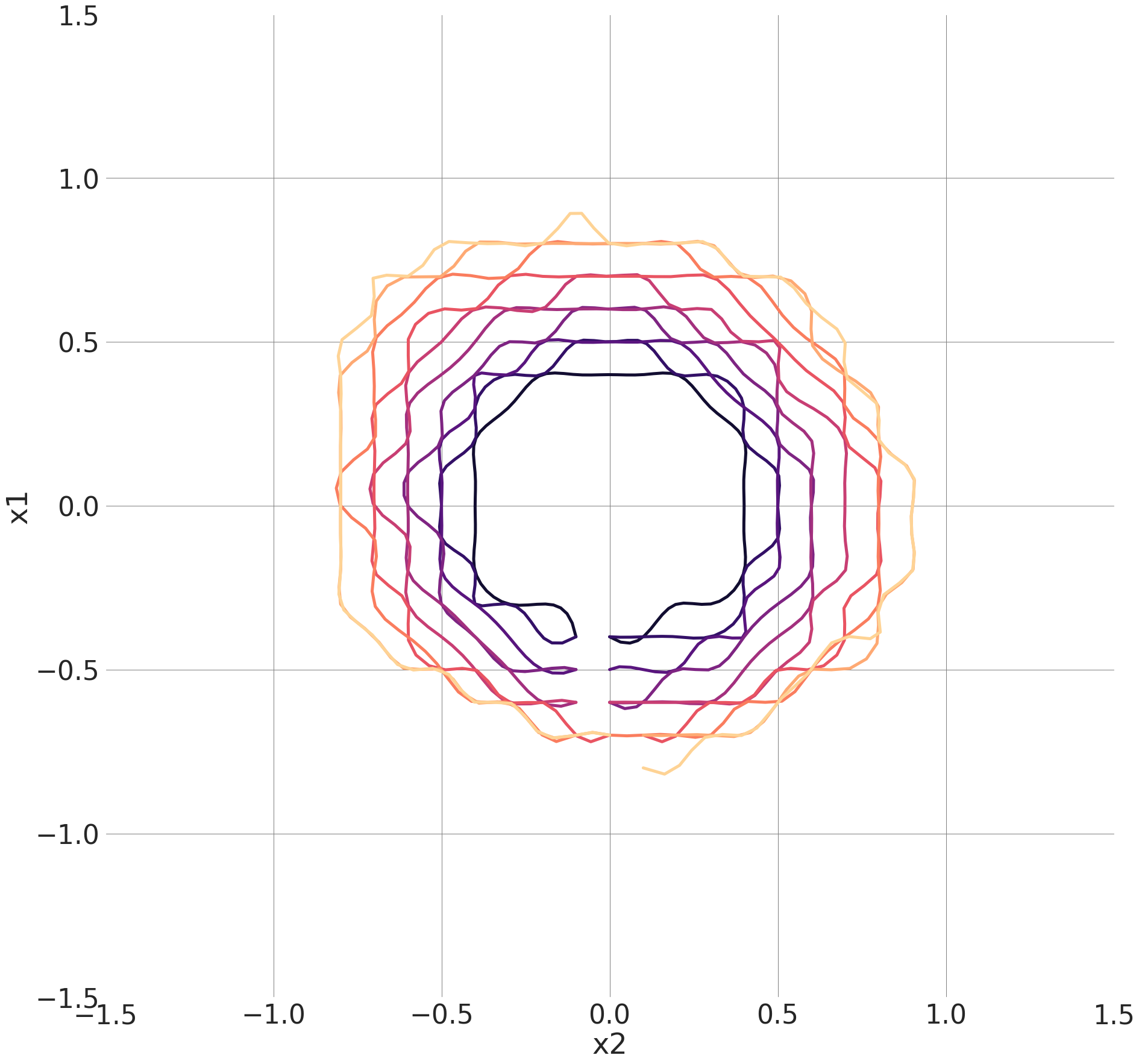}
        \captionsetup{labelformat=empty}
        \captionsetup{justification=centering}
        \caption{IQS}
    \end{minipage}\hfill
    \begin{minipage}[l]{0.2\linewidth}
        \centering
        \includegraphics[trim={0 0 0 0}, clip, width=1\linewidth]{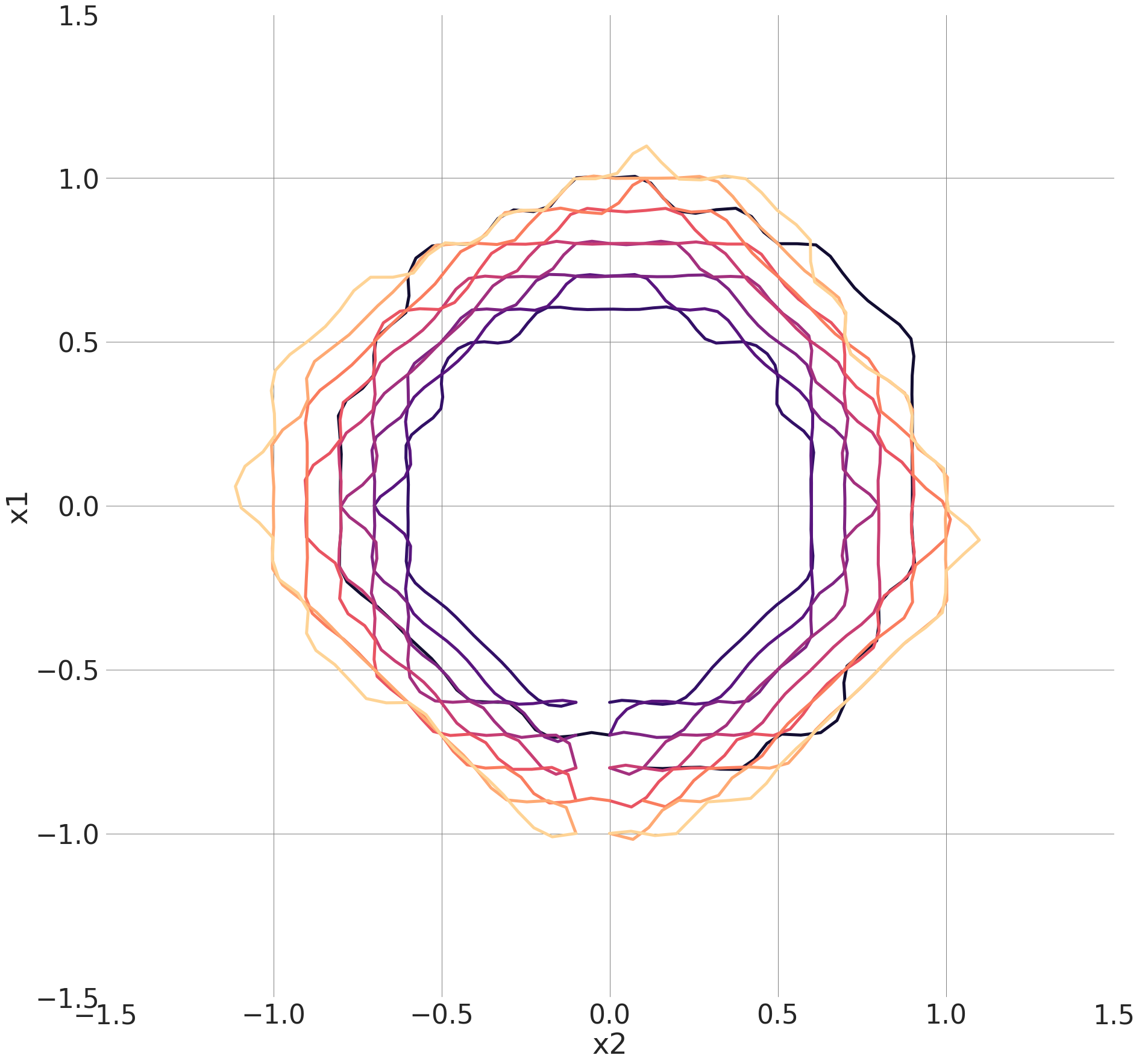}
        \captionsetup{labelformat=empty}
        \captionsetup{justification=centering}
        \caption{IQS-CS-SMOTE}
    \end{minipage}\hfill
    \begin{minipage}[l]{0.2\linewidth}
        \centering
        \includegraphics[trim={0 0 0 0}, clip, width=1\linewidth]{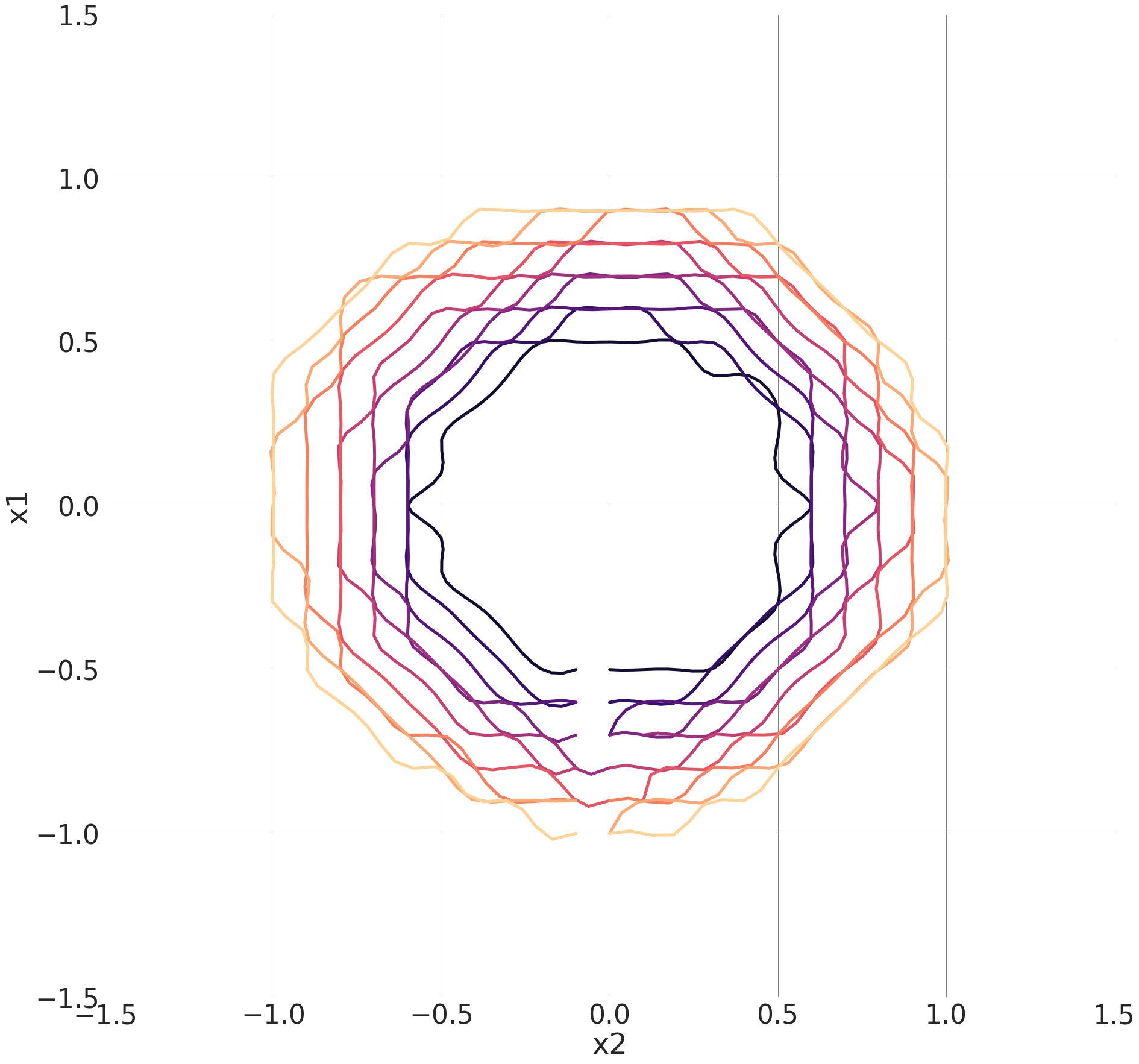}
        \captionsetup{labelformat=empty}
        \captionsetup{justification=centering}
        \caption{Model-based IQS}
    \end{minipage}\hfill
    \begin{minipage}[l]{0.2\linewidth}
        \centering
        \includegraphics[trim={0 0 0 0}, clip, width=1\linewidth]{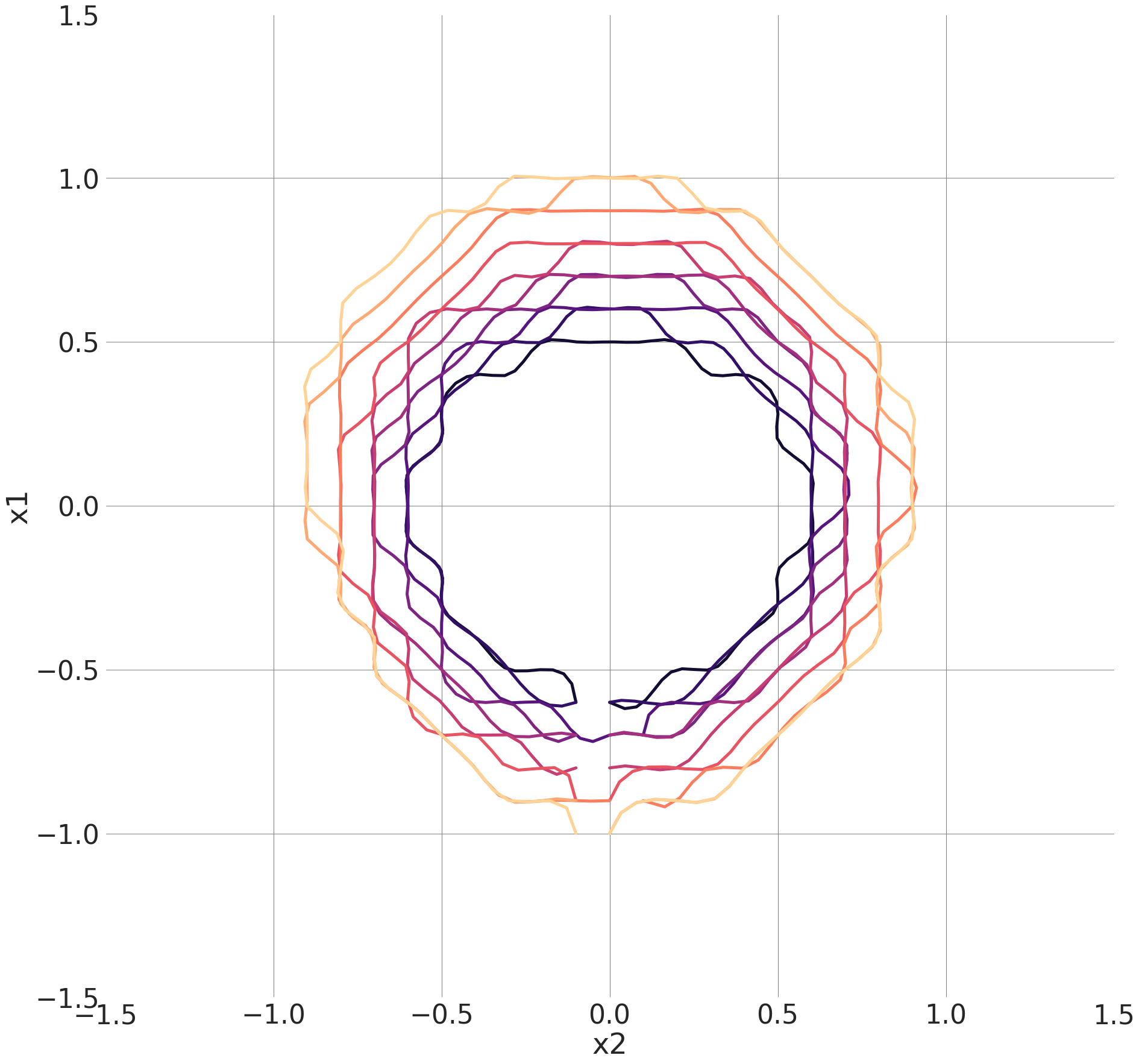}
        \captionsetup{labelformat=empty}
        \captionsetup{justification=centering}
        \caption{Model-based IQS-CS-SMOTE}
    \end{minipage}\hfill

    \begin{minipage}[l]{0.152\linewidth}
        % \centering
        \includegraphics[trim={55cm 0 0cm 0}, clip, width=1\linewidth, angle=270]{images/radial_1000_iq_cssmote_heat_q0q1.png}
        \captionsetup{labelformat=empty}
    \end{minipage}
    \setcounter{figure}{12}
    \caption{\textbf{Top to bottom}: Heat-maps of \(Q(s,a=0)\), \(Q(s,a=1)\), \(Q(s,a=0)-Q(s,a=1)\) at \(t=25\), and recovered stopping boundaries per time step for the Radial example.}
    \label{fig:radial_heatmaps}
    \vskip 5pt
\end{figure*}

\begin{figure*}[ht]
    \begin{minipage}[l]{0.2\linewidth}
        \centering
        \includegraphics[trim={0 0 10cm 0}, clip, width=1\linewidth]{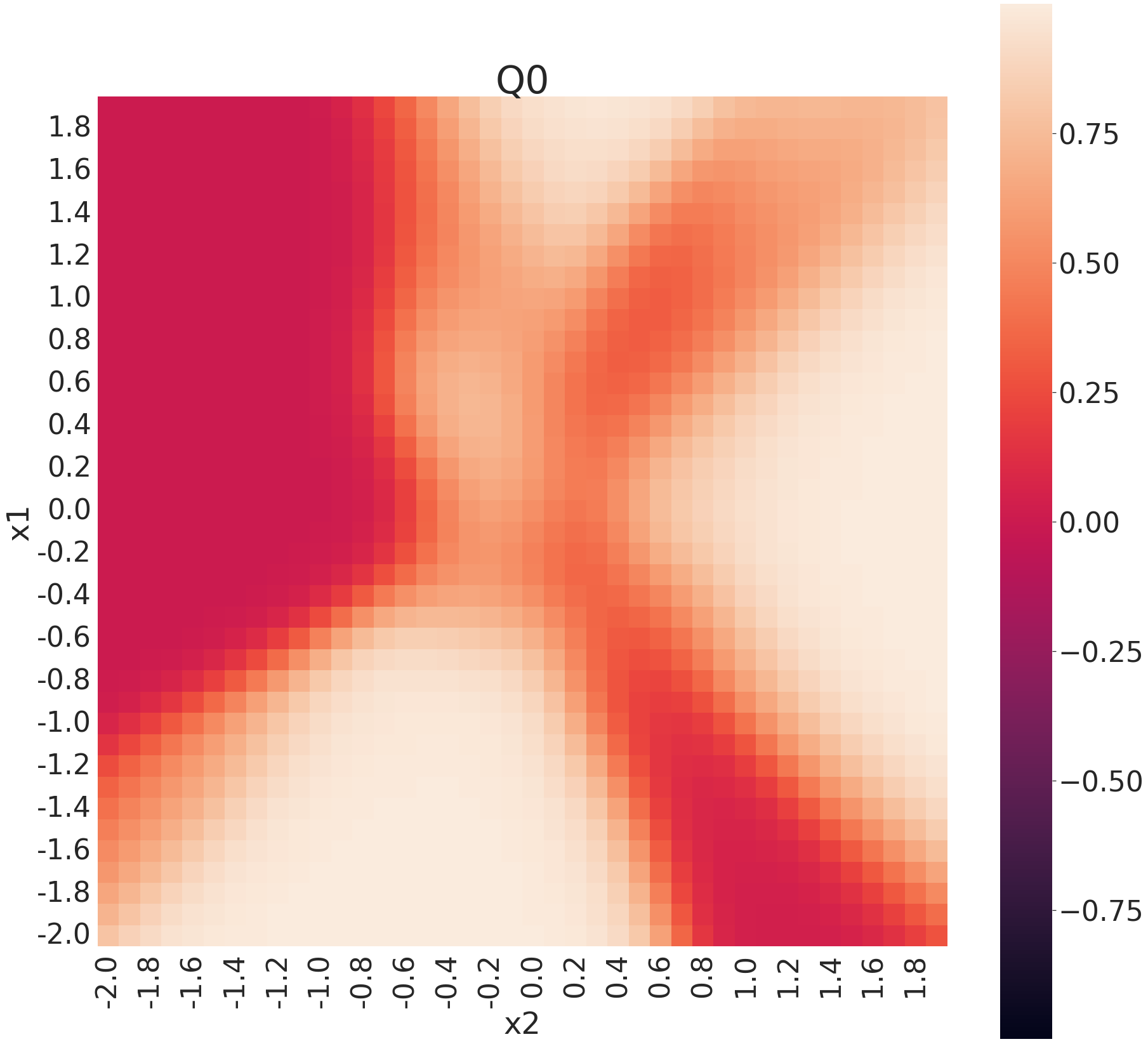}
        \captionsetup{labelformat=empty}
    \end{minipage}\hfill
    \begin{minipage}[l]{0.2\linewidth}
        \centering
        \includegraphics[trim={0 0 10cm 0}, clip, width=1\linewidth]{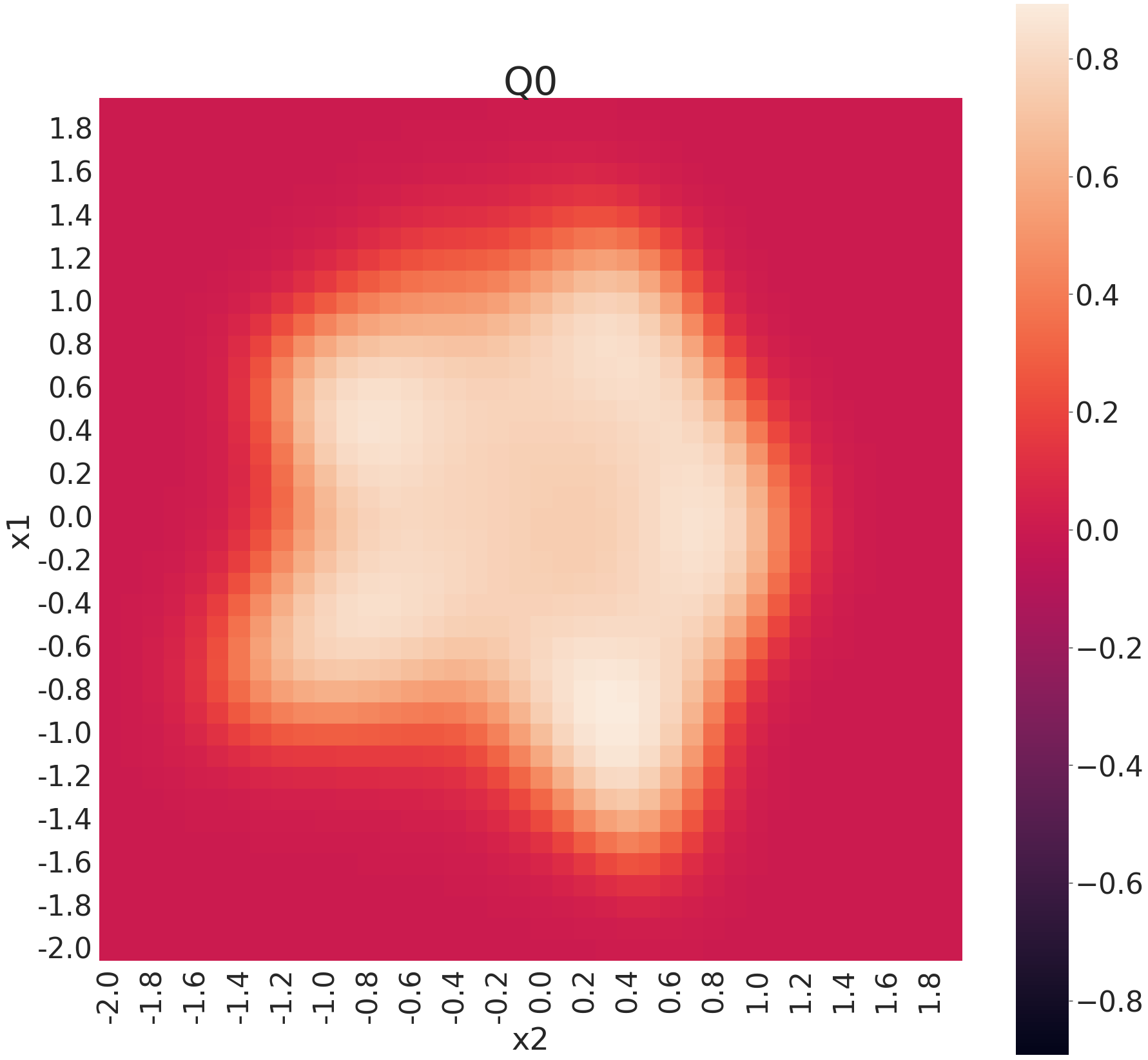}
        \captionsetup{labelformat=empty}
    \end{minipage}\hfill
    \begin{minipage}[l]{0.2\linewidth}
        \centering
        \includegraphics[trim={0 0 10cm 0}, clip, width=1\linewidth]{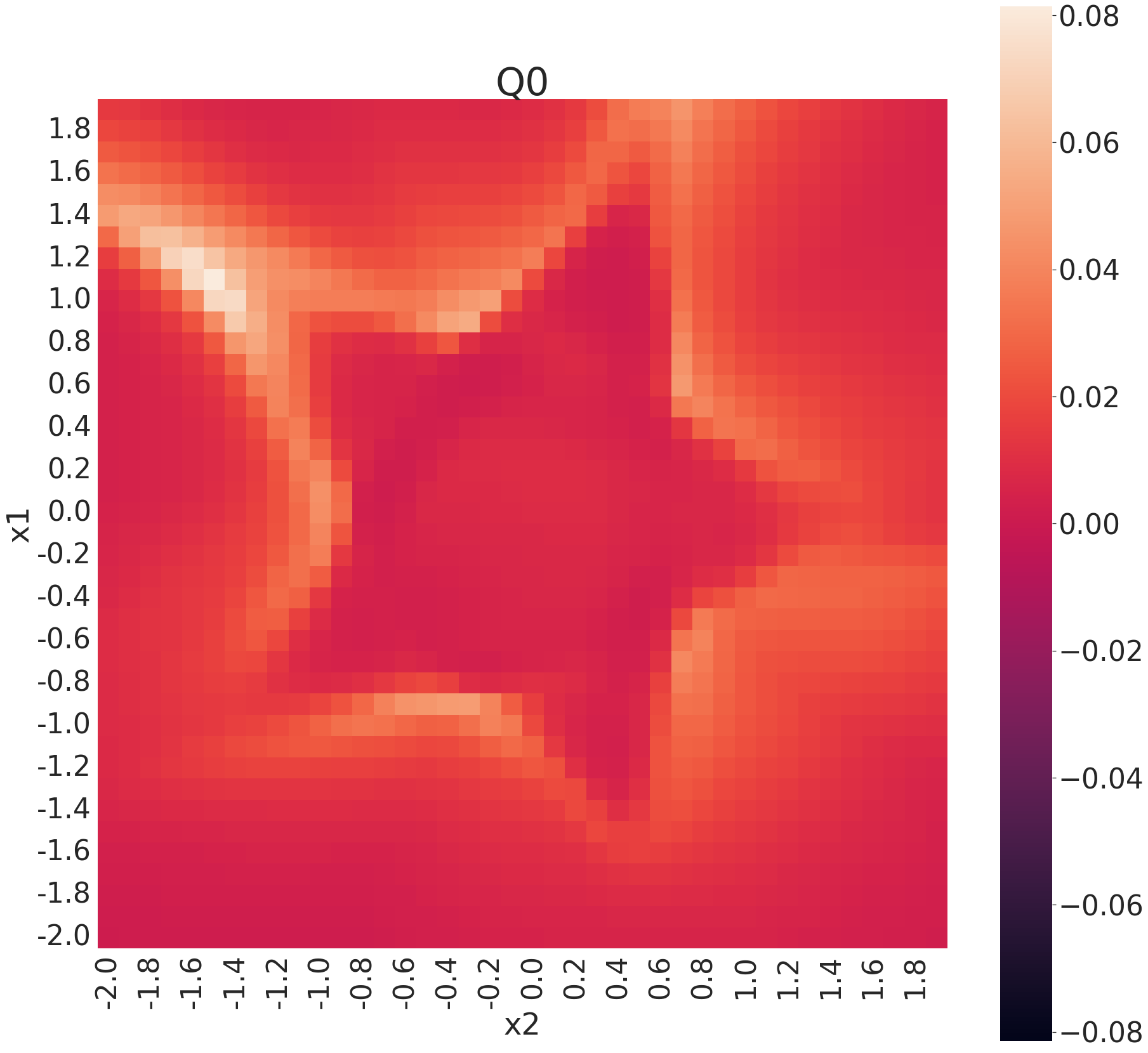}
        \captionsetup{labelformat=empty}
    \end{minipage}\hfill
    \begin{minipage}[l]{0.2\linewidth}
        \centering
        \includegraphics[trim={0 0 10cm 0}, clip, width=1\linewidth]{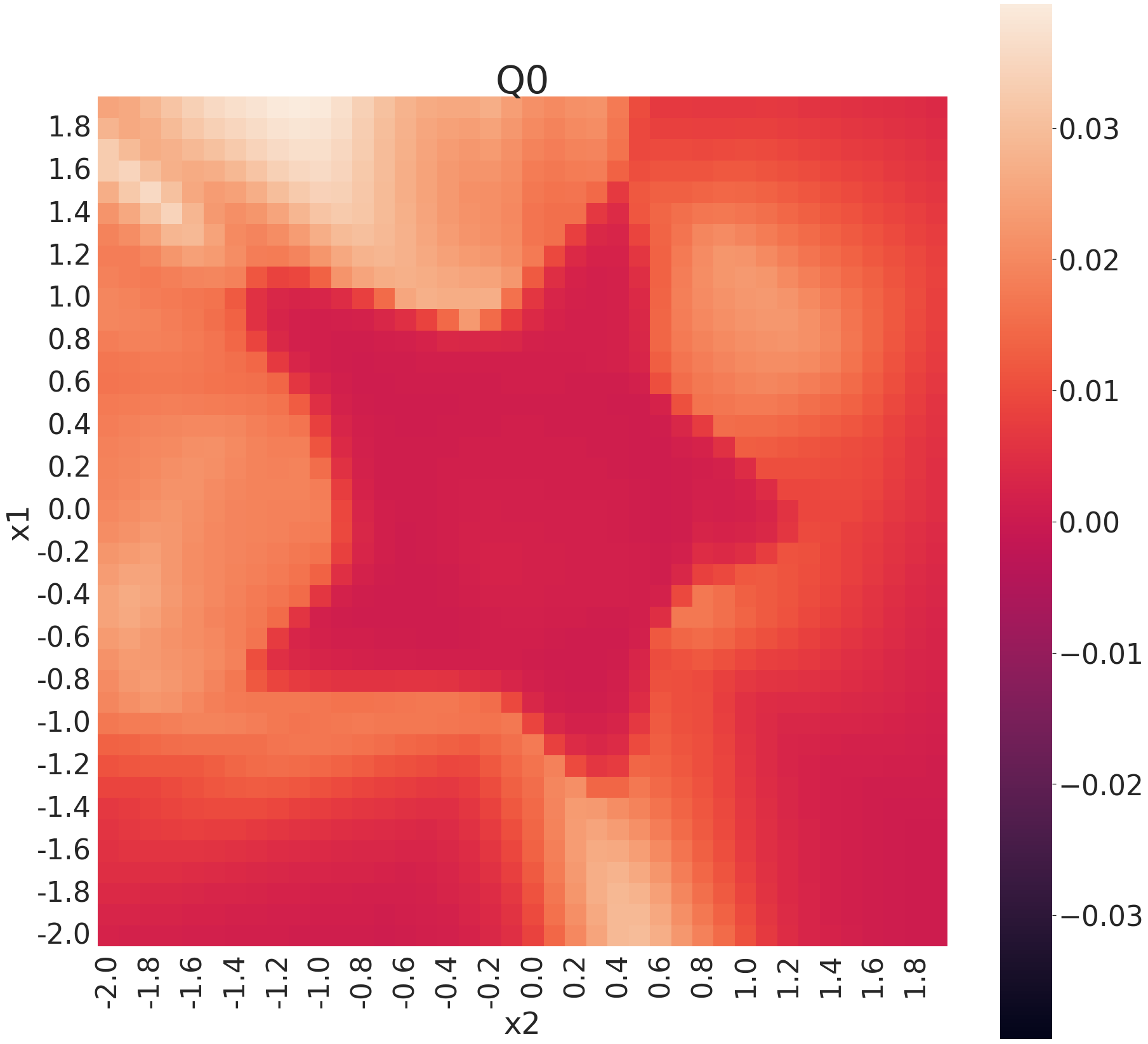}
        \captionsetup{labelformat=empty}
    \end{minipage}\hfill
    \begin{minipage}[l]{0.2\linewidth}
        \centering
        \includegraphics[trim={0 0 10cm 0}, clip, width=1\linewidth]{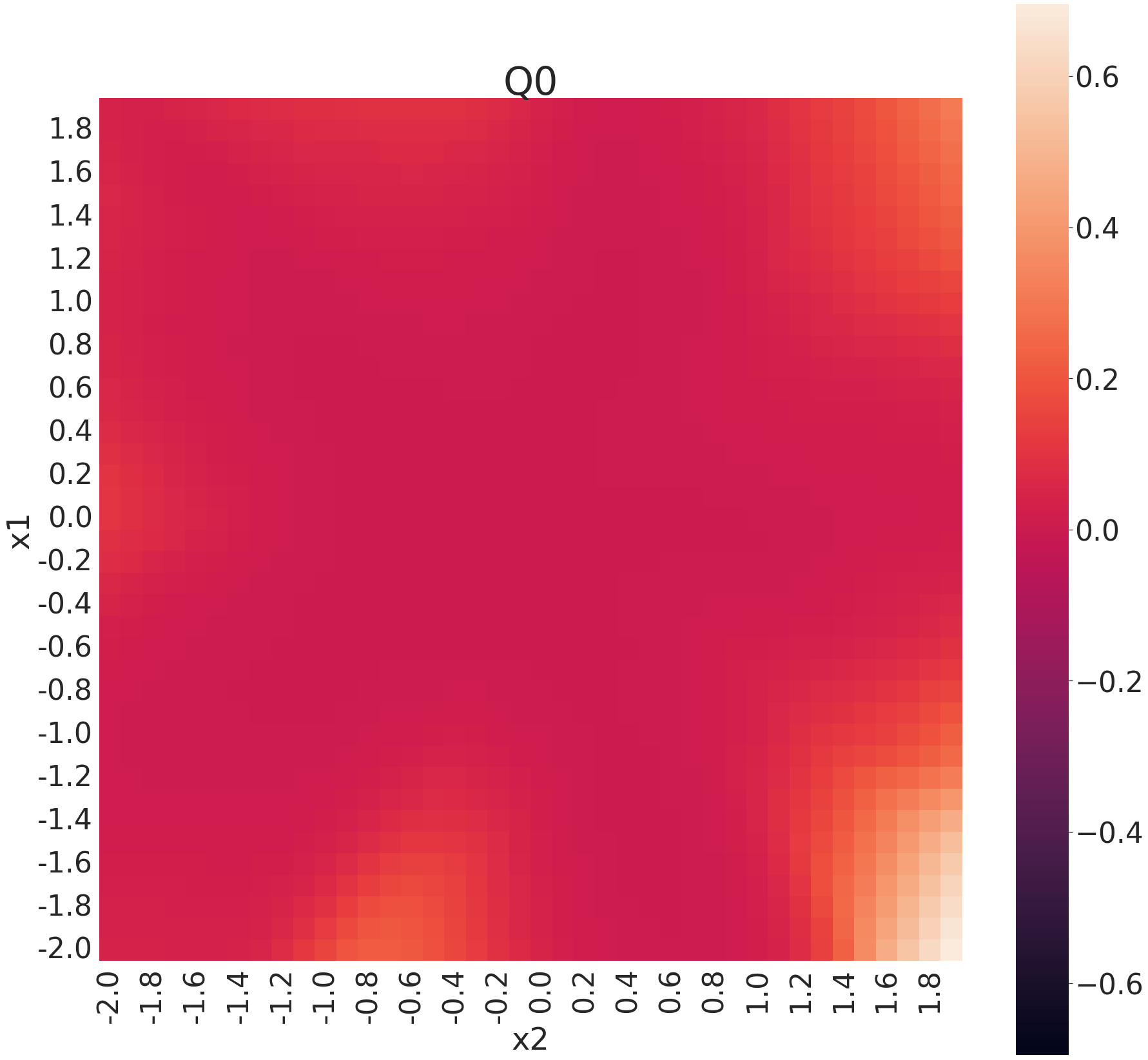}
        \captionsetup{labelformat=empty}
    \end{minipage}\hfill

    \begin{minipage}[l]{0.2\linewidth}
        \centering
        \includegraphics[trim={0 0 10cm 0}, clip, width=1\linewidth]{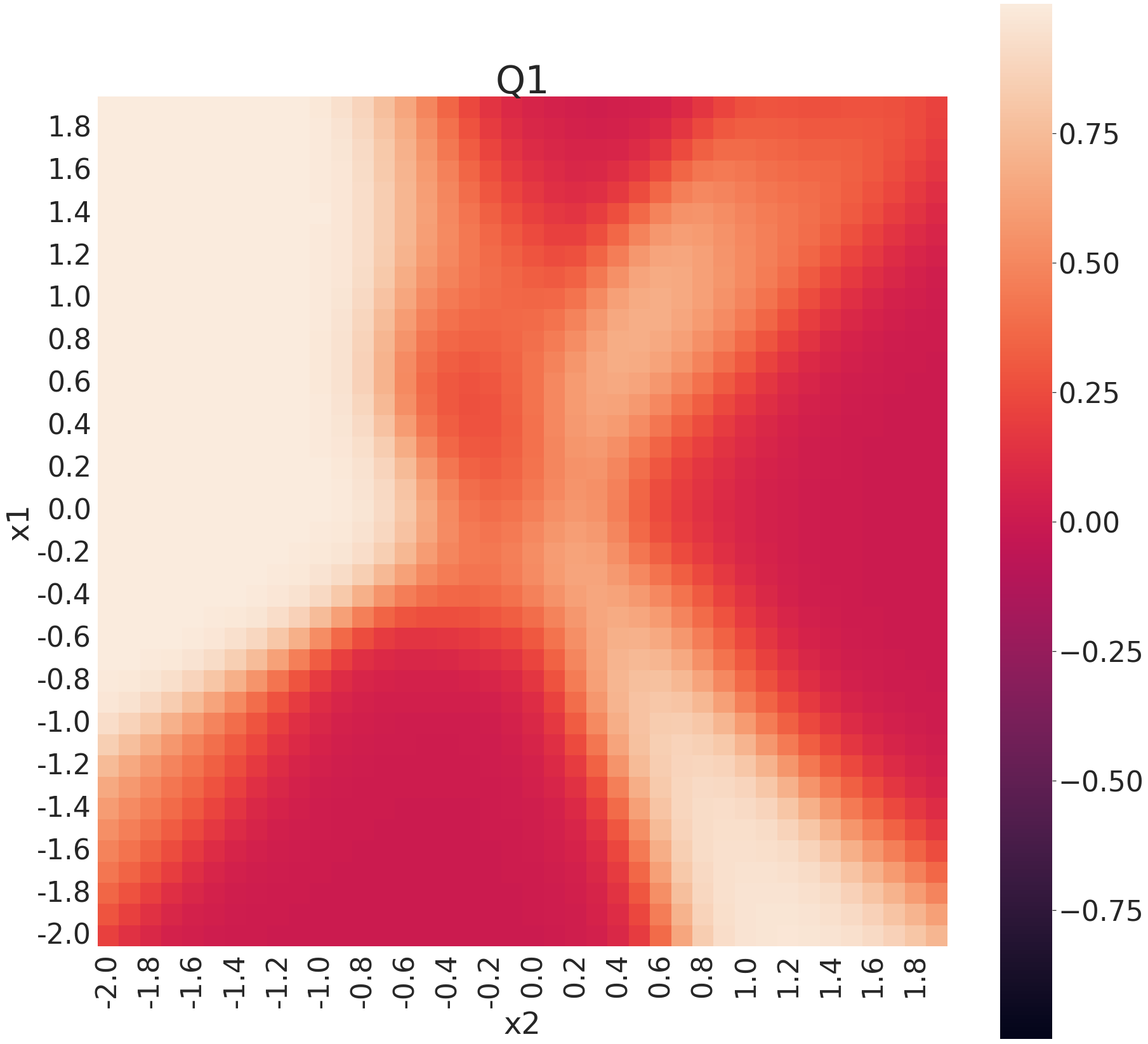}
        \captionsetup{labelformat=empty}
    \end{minipage}\hfill
    \begin{minipage}[l]{0.2\linewidth}
        \centering
        \includegraphics[trim={0 0 10cm 0}, clip, width=1\linewidth]{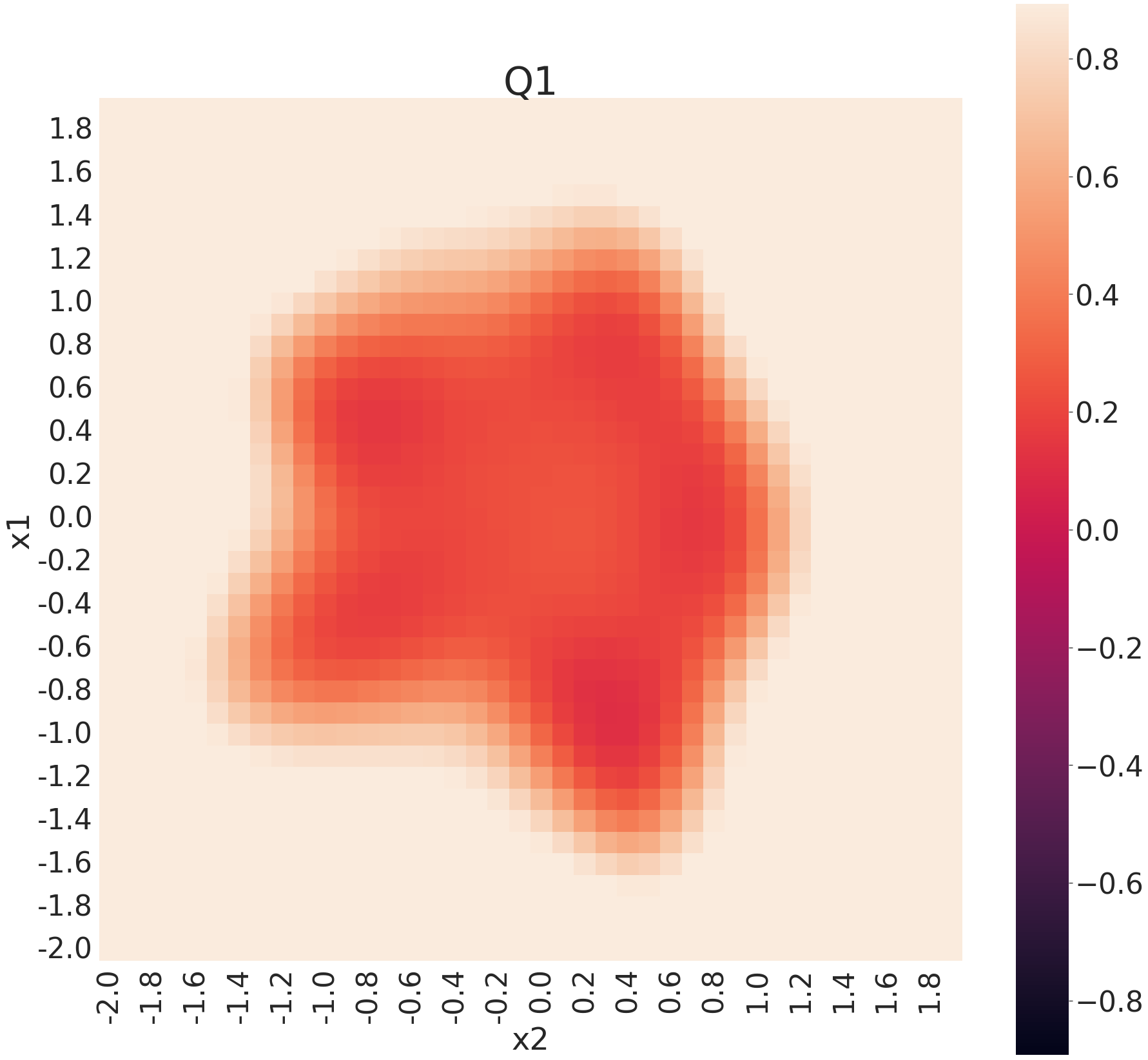}
        \captionsetup{labelformat=empty}
    \end{minipage}\hfill
    \begin{minipage}[l]{0.2\linewidth}
        \centering
        \includegraphics[trim={0 0 10cm 0}, clip, width=1\linewidth]{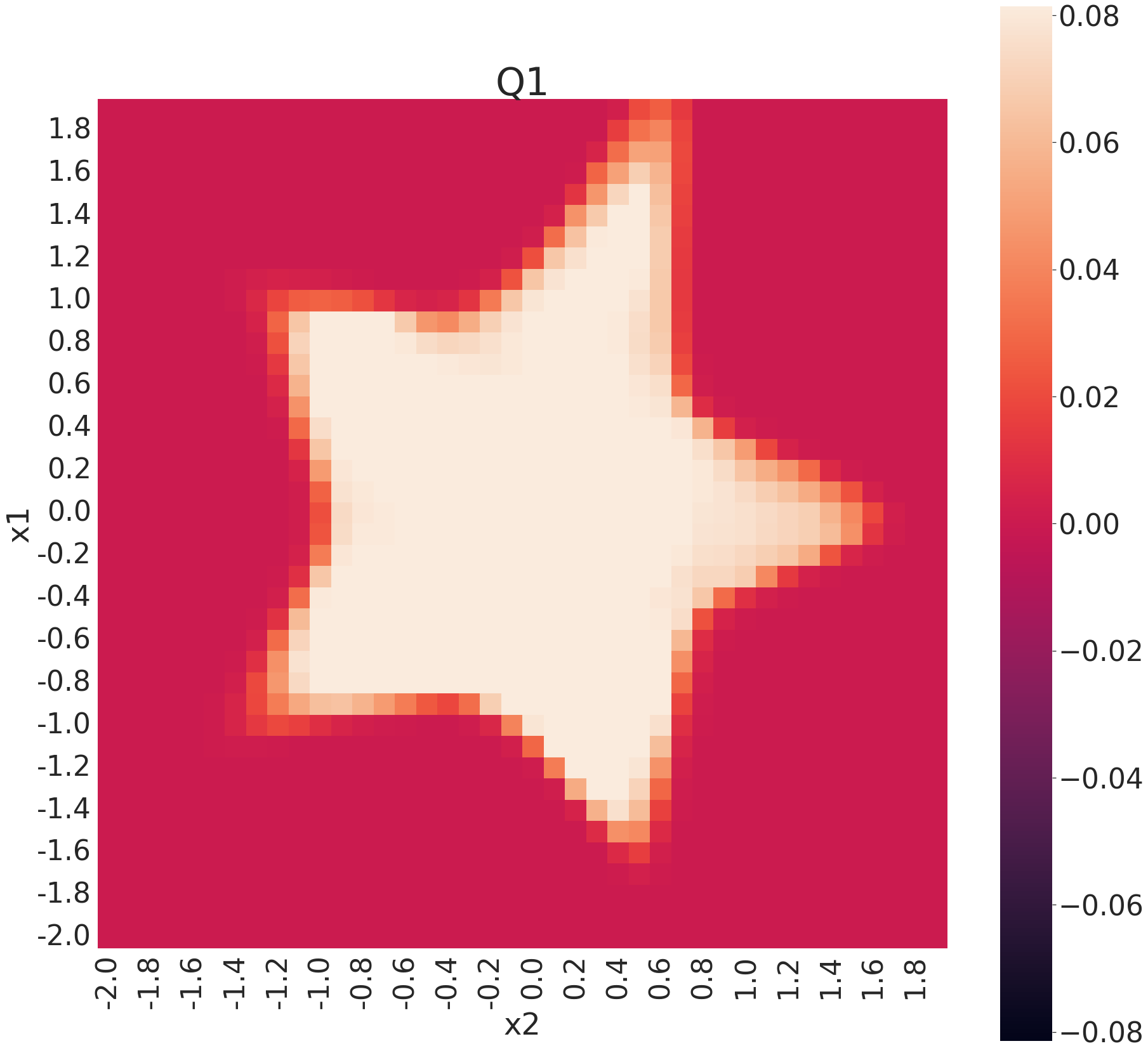}
        \captionsetup{labelformat=empty}
    \end{minipage}\hfill
    \begin{minipage}[l]{0.2\linewidth}
        \centering
        \includegraphics[trim={0 0 10cm 0}, clip, width=1\linewidth]{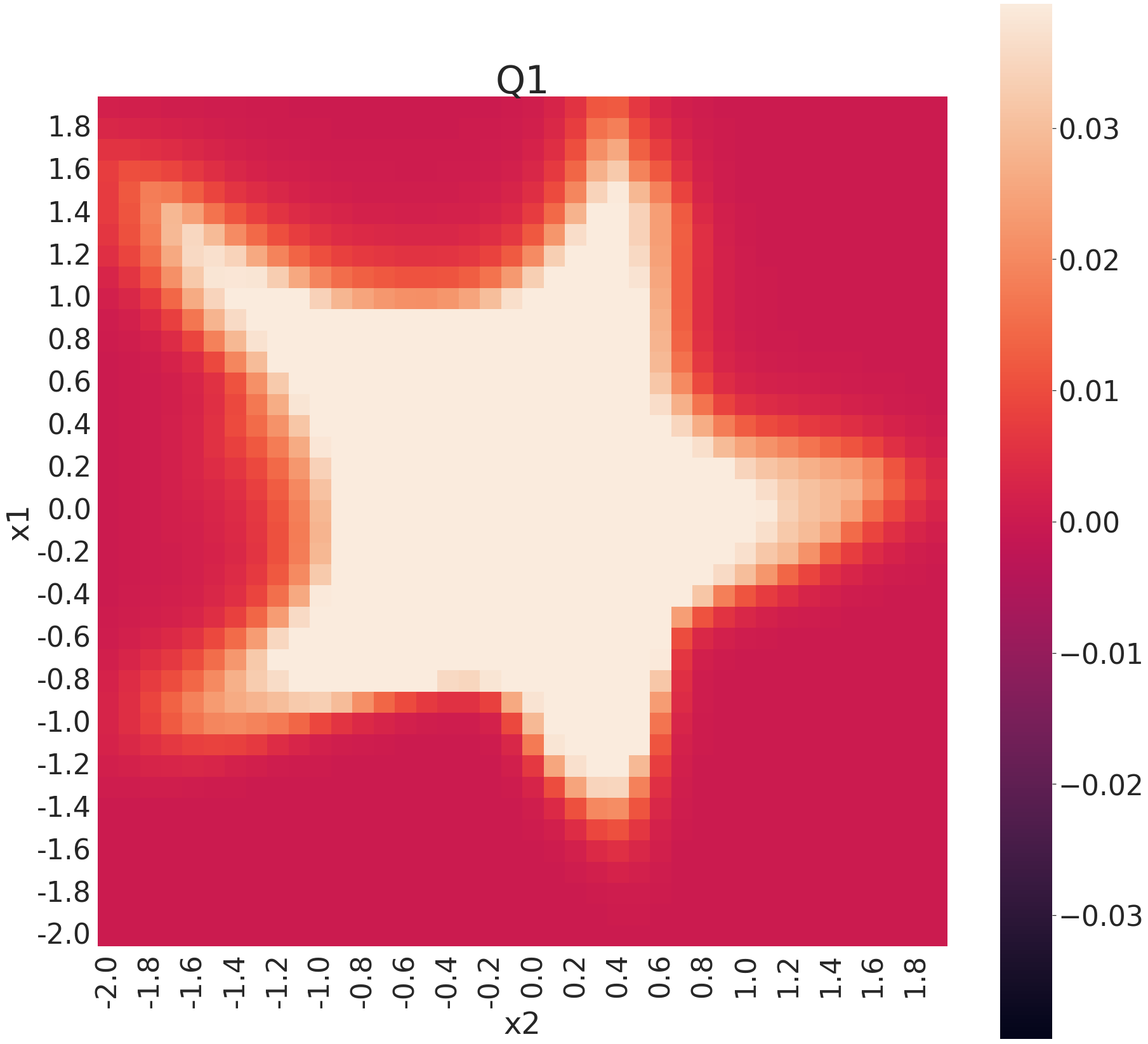}
        \captionsetup{labelformat=empty}
    \end{minipage}\hfill
    \begin{minipage}[l]{0.2\linewidth}
        \centering
        \includegraphics[trim={0 0 10cm 0}, clip, width=1\linewidth]{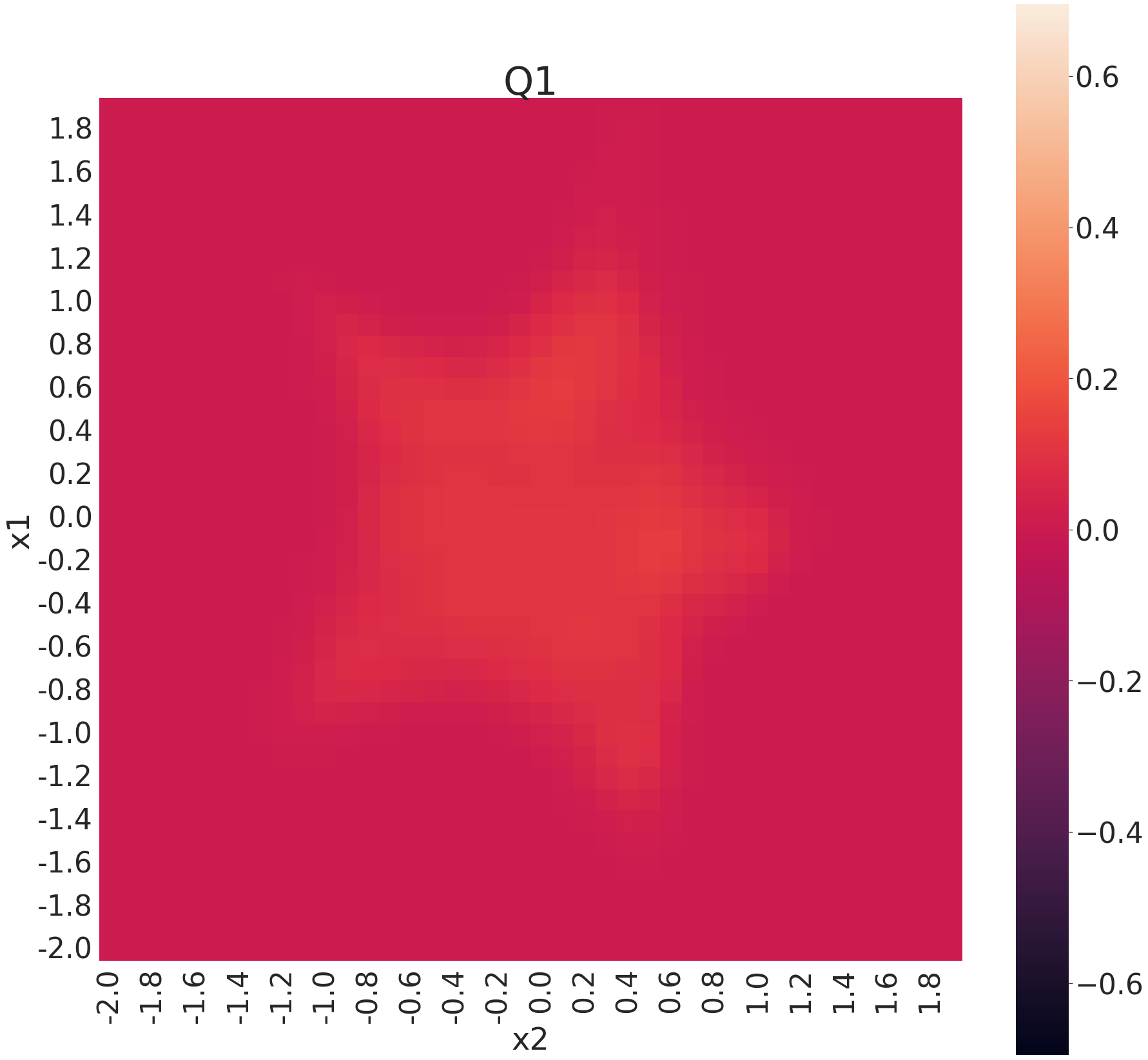}
        \captionsetup{labelformat=empty}
    \end{minipage}\hfill

    \begin{minipage}[l]{0.2\linewidth}
        \centering
        \includegraphics[trim={0 0 10cm 0}, clip, width=1\linewidth]{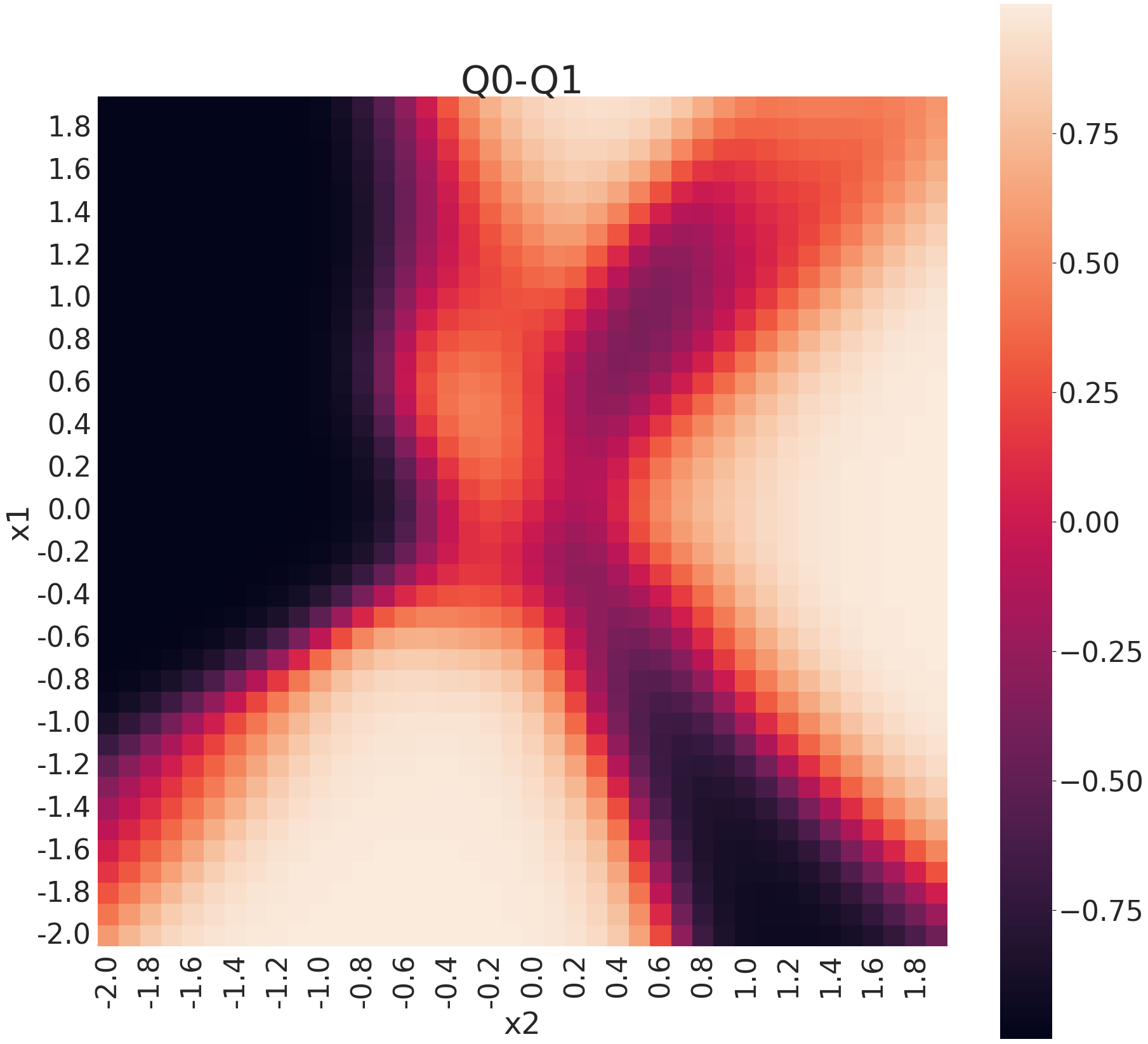}
        \captionsetup{labelformat=empty}
    \end{minipage}\hfill
    \begin{minipage}[l]{0.2\linewidth}
        \centering
        \includegraphics[trim={0 0 10cm 0}, clip, width=1\linewidth]{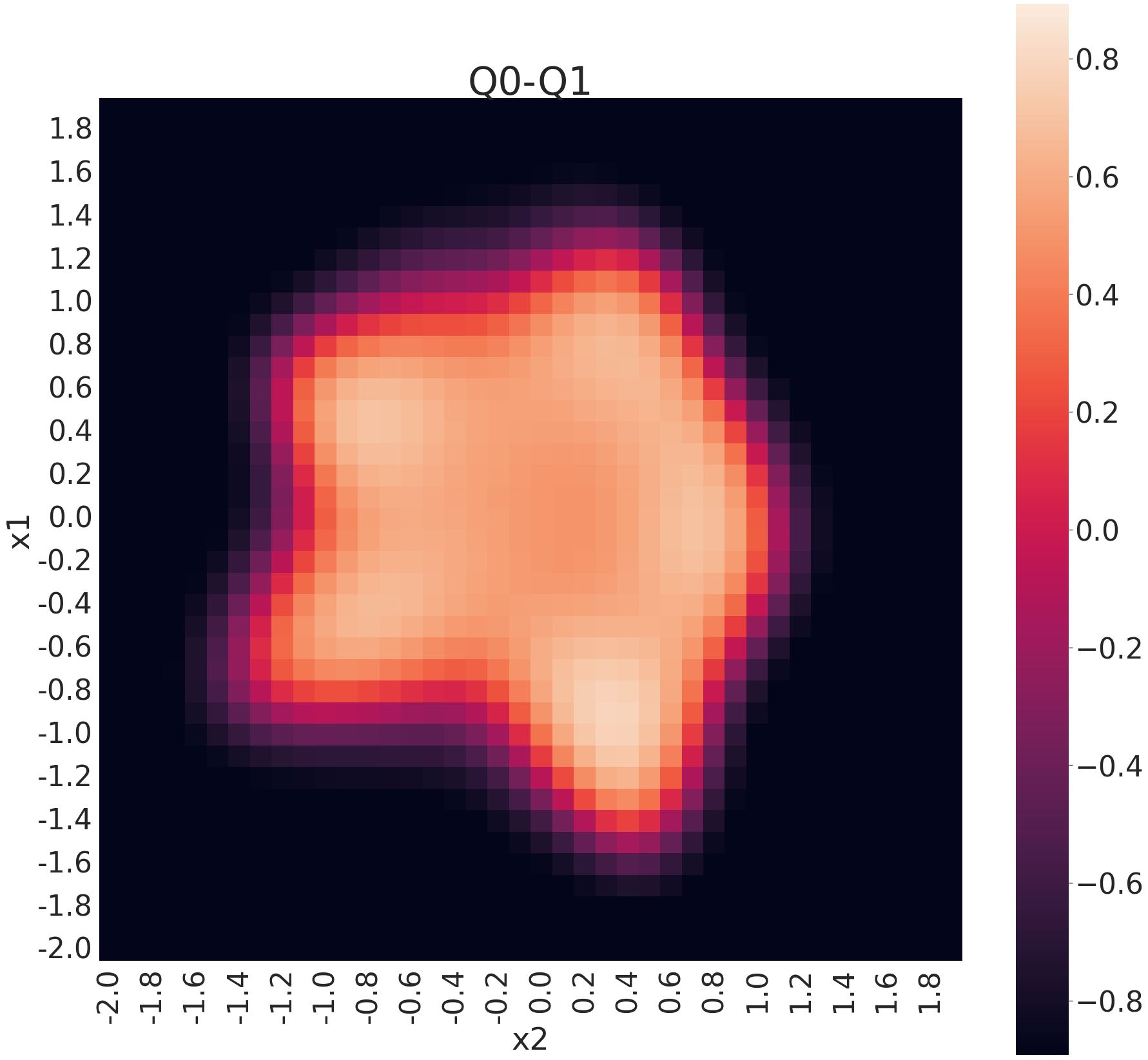}
        \captionsetup{labelformat=empty}
    \end{minipage}\hfill
    \begin{minipage}[l]{0.2\linewidth}
        \centering
        \includegraphics[trim={0 0 10cm 0}, clip, width=1\linewidth]{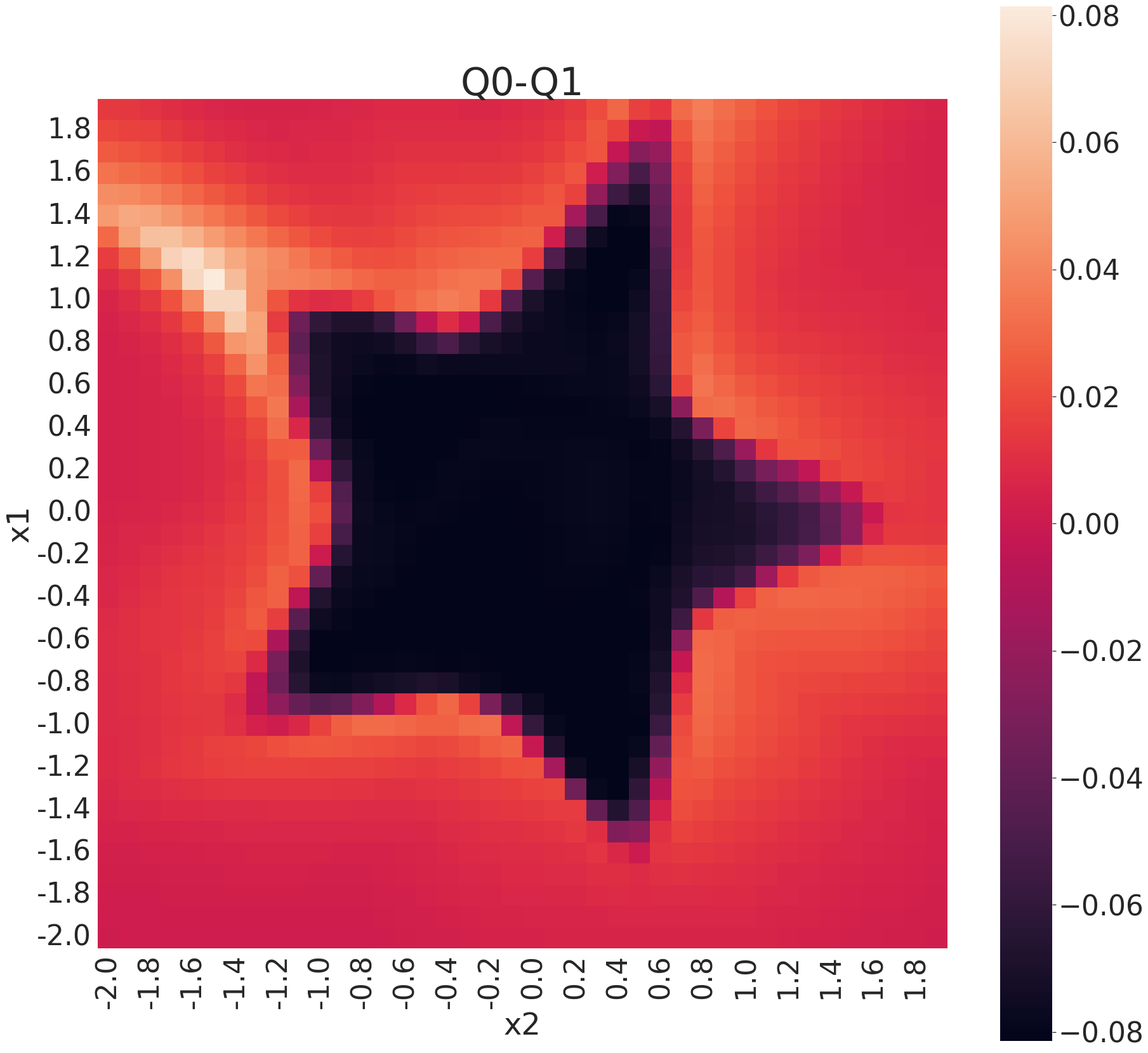}
        \captionsetup{labelformat=empty}
    \end{minipage}\hfill
    \begin{minipage}[l]{0.2\linewidth}
        \centering
        \includegraphics[trim={0 0 10cm 0}, clip, width=1\linewidth]{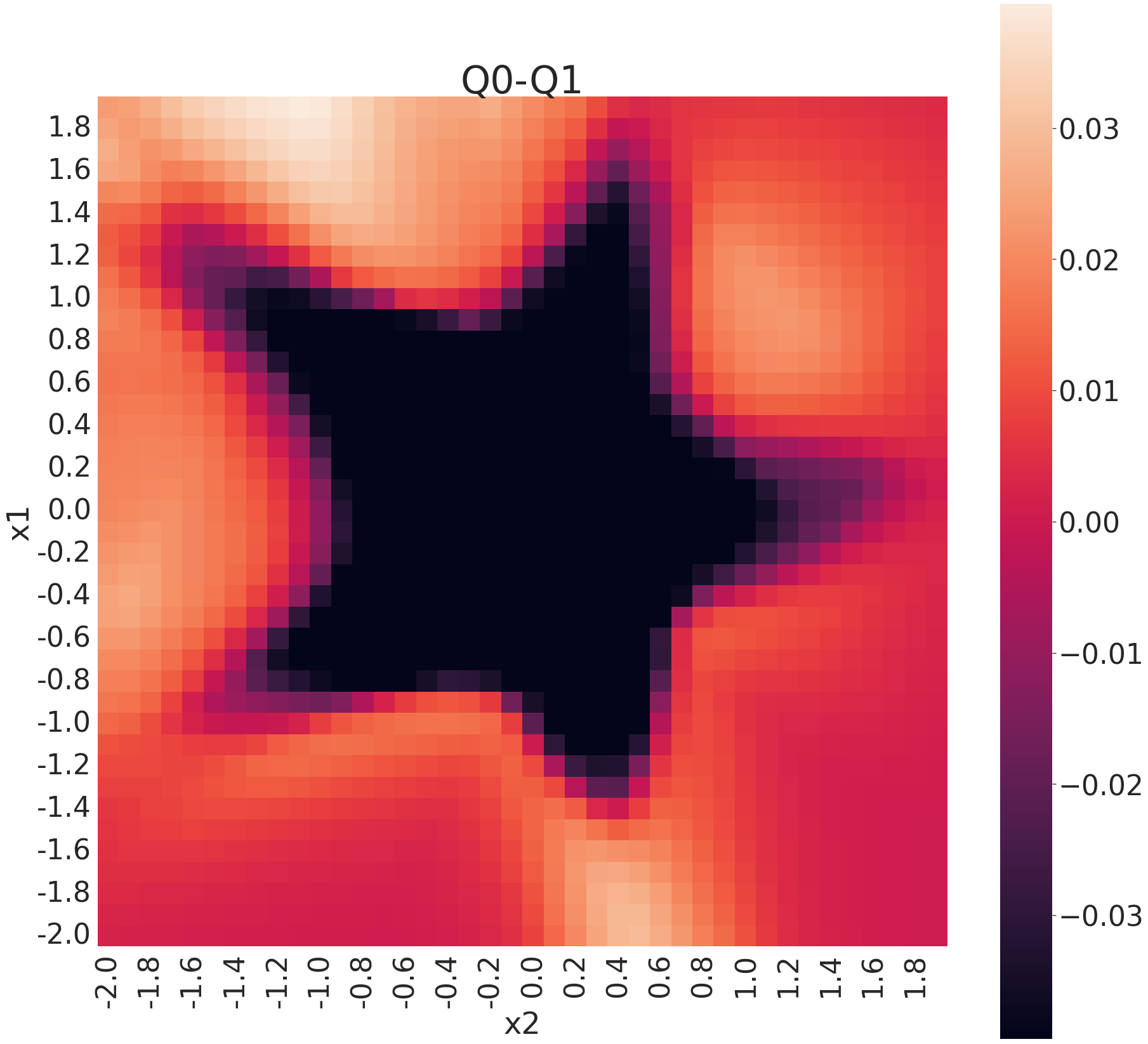}
        \captionsetup{labelformat=empty}
    \end{minipage}\hfill
    \begin{minipage}[l]{0.2\linewidth}
        \centering
        \includegraphics[trim={0 0 10cm 0}, clip, width=1\linewidth]{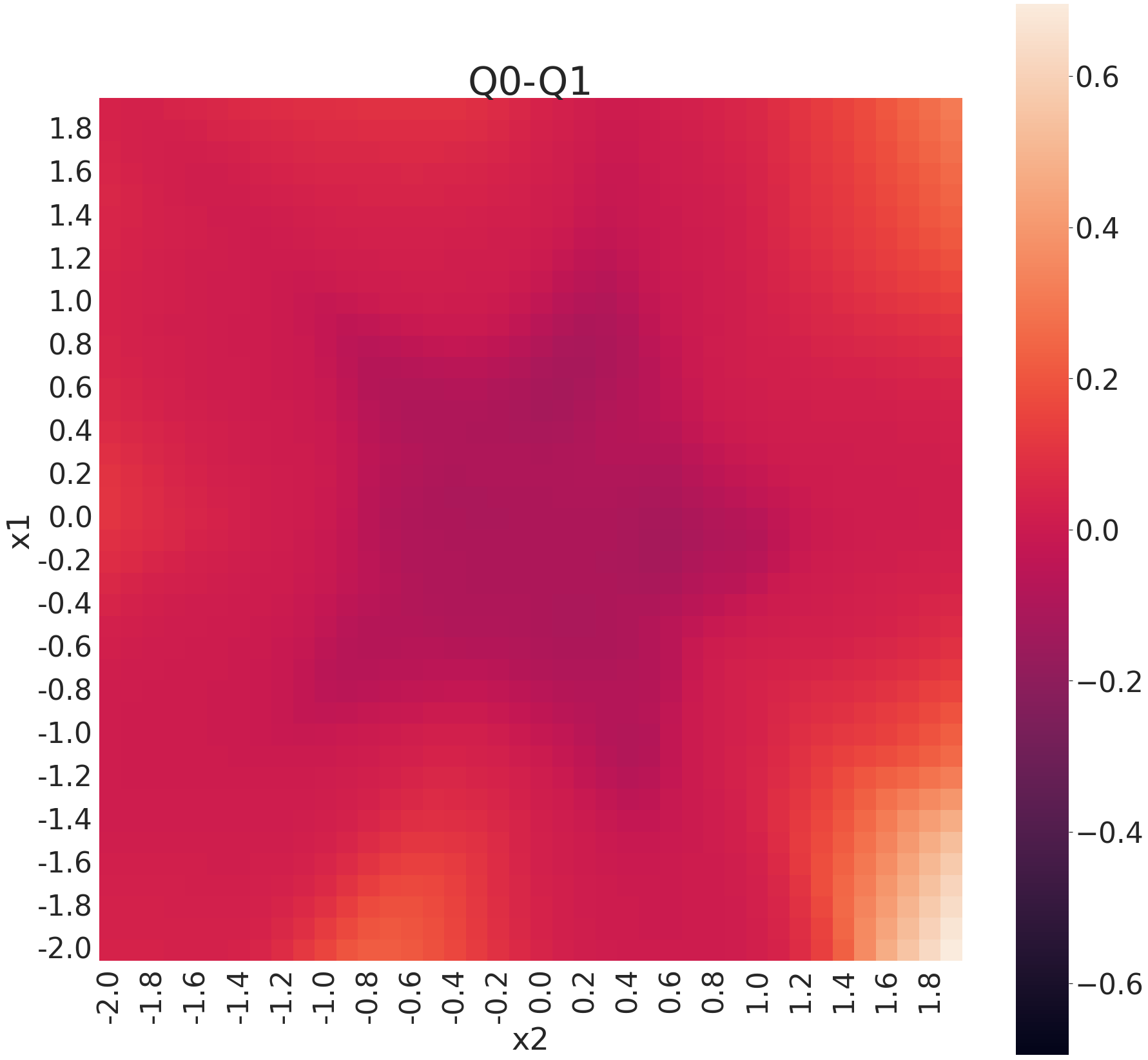}
        \captionsetup{labelformat=empty}
    \end{minipage}\hfill

    \begin{minipage}[l]{0.2\linewidth}
        \centering
        \includegraphics[trim={0 0 0 0}, clip, width=1\linewidth]{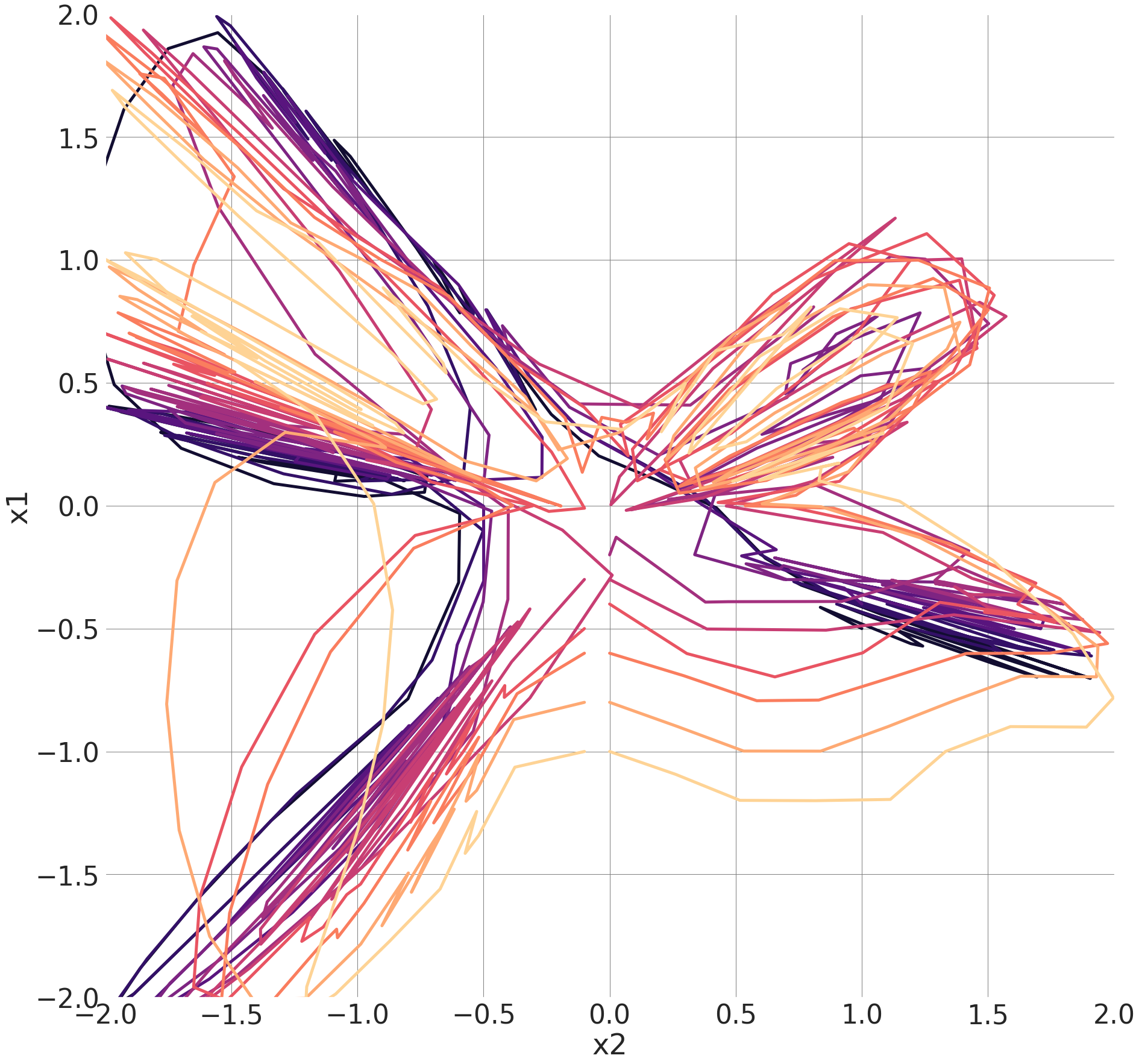}
        \captionsetup{labelformat=empty}
        \captionsetup{justification=centering}
        \caption{Classifier}
    \end{minipage}\hfill
    \begin{minipage}[l]{0.2\linewidth}
        \centering
        \includegraphics[trim={0 0 0 0}, clip, width=1\linewidth]{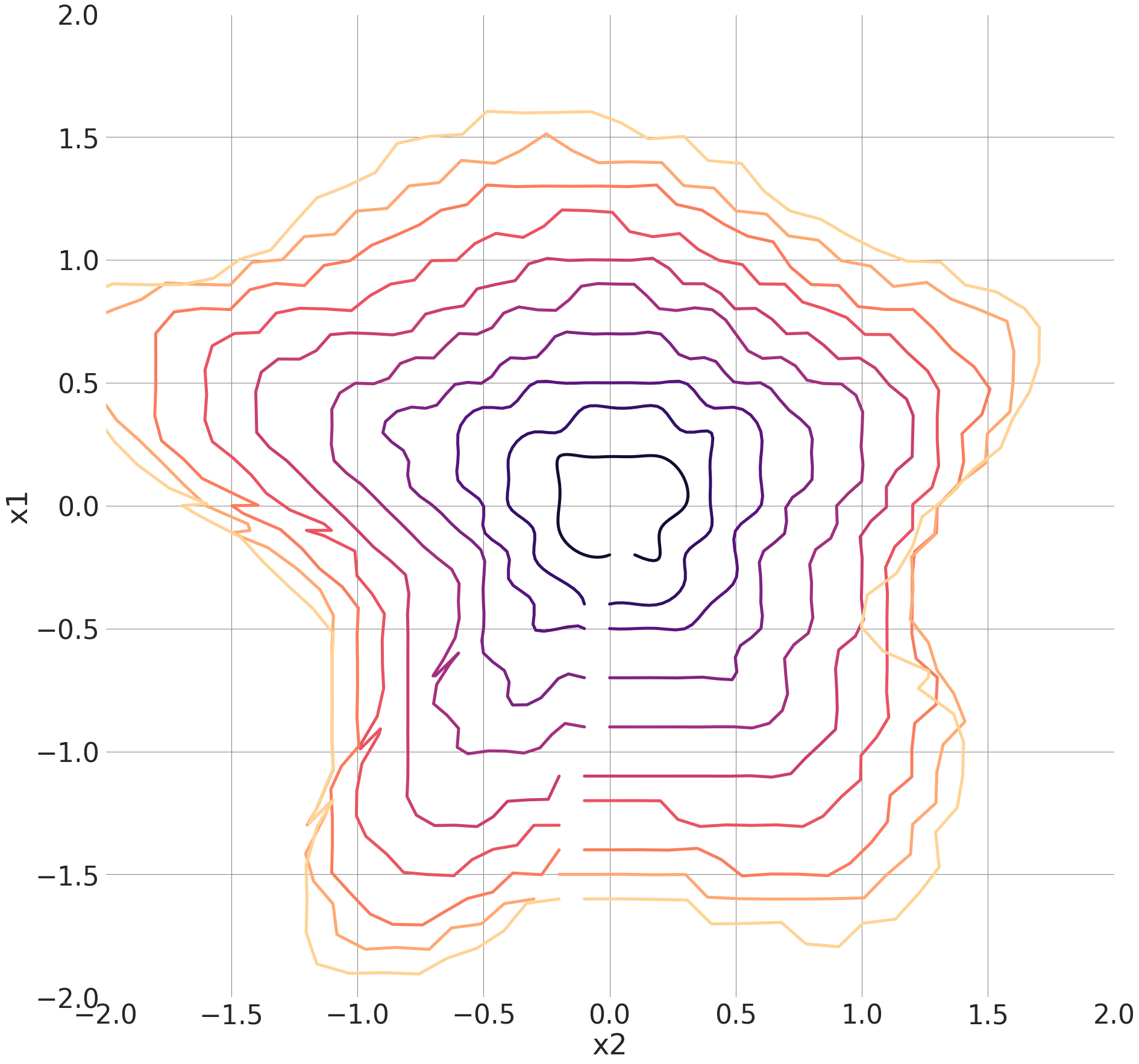}
        \captionsetup{labelformat=empty}
        \captionsetup{justification=centering}
        \caption{IQS}
    \end{minipage}\hfill
    \begin{minipage}[l]{0.2\linewidth}
        \centering
        \includegraphics[trim={0 0 0 0}, clip, width=1\linewidth]{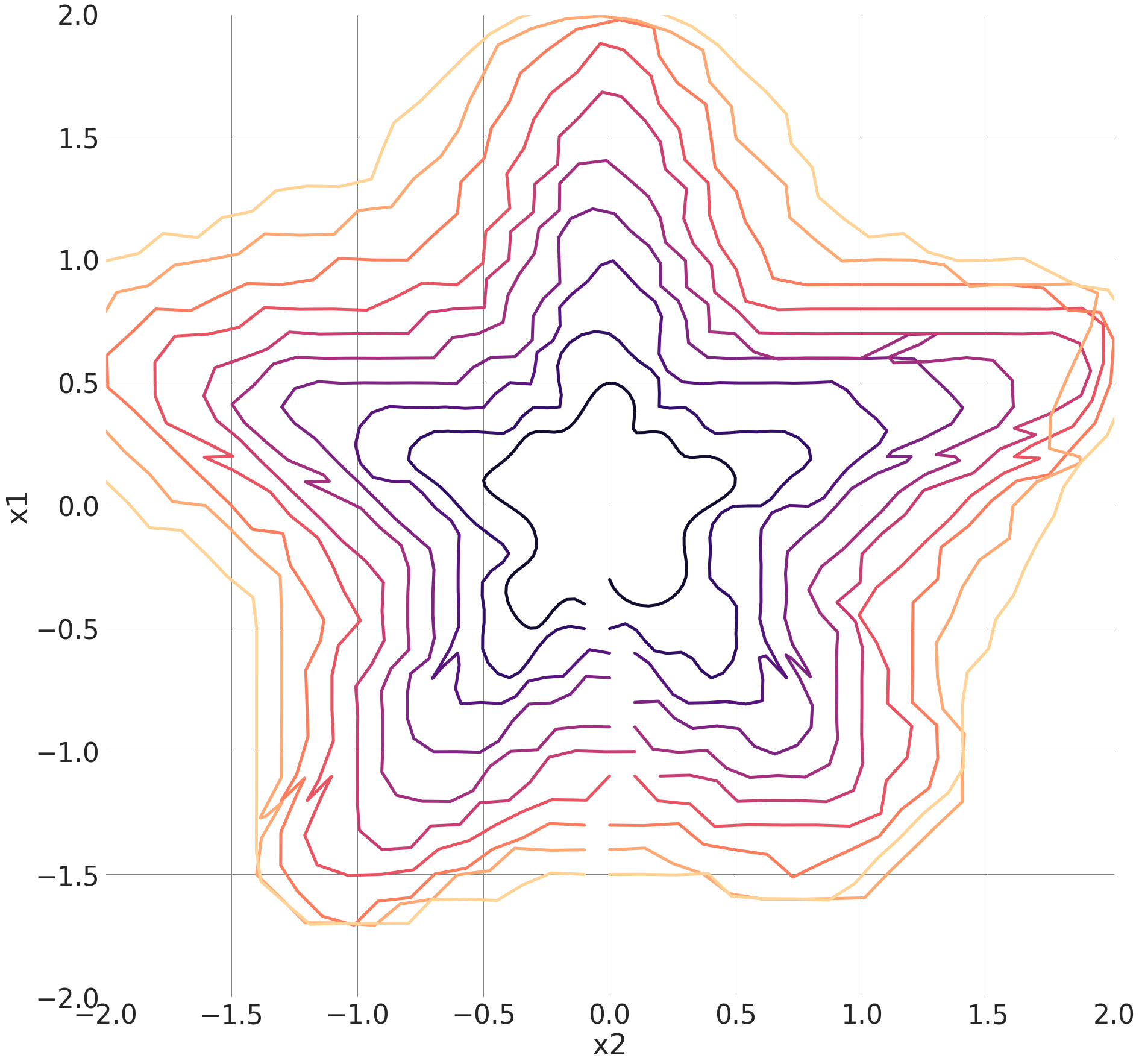}
        \captionsetup{labelformat=empty}
        \captionsetup{justification=centering}
        \caption{IQS-CS-SMOTE}
    \end{minipage}\hfill
    \begin{minipage}[l]{0.2\linewidth}
        \centering
        \includegraphics[trim={0 0 0 0}, clip, width=1\linewidth]{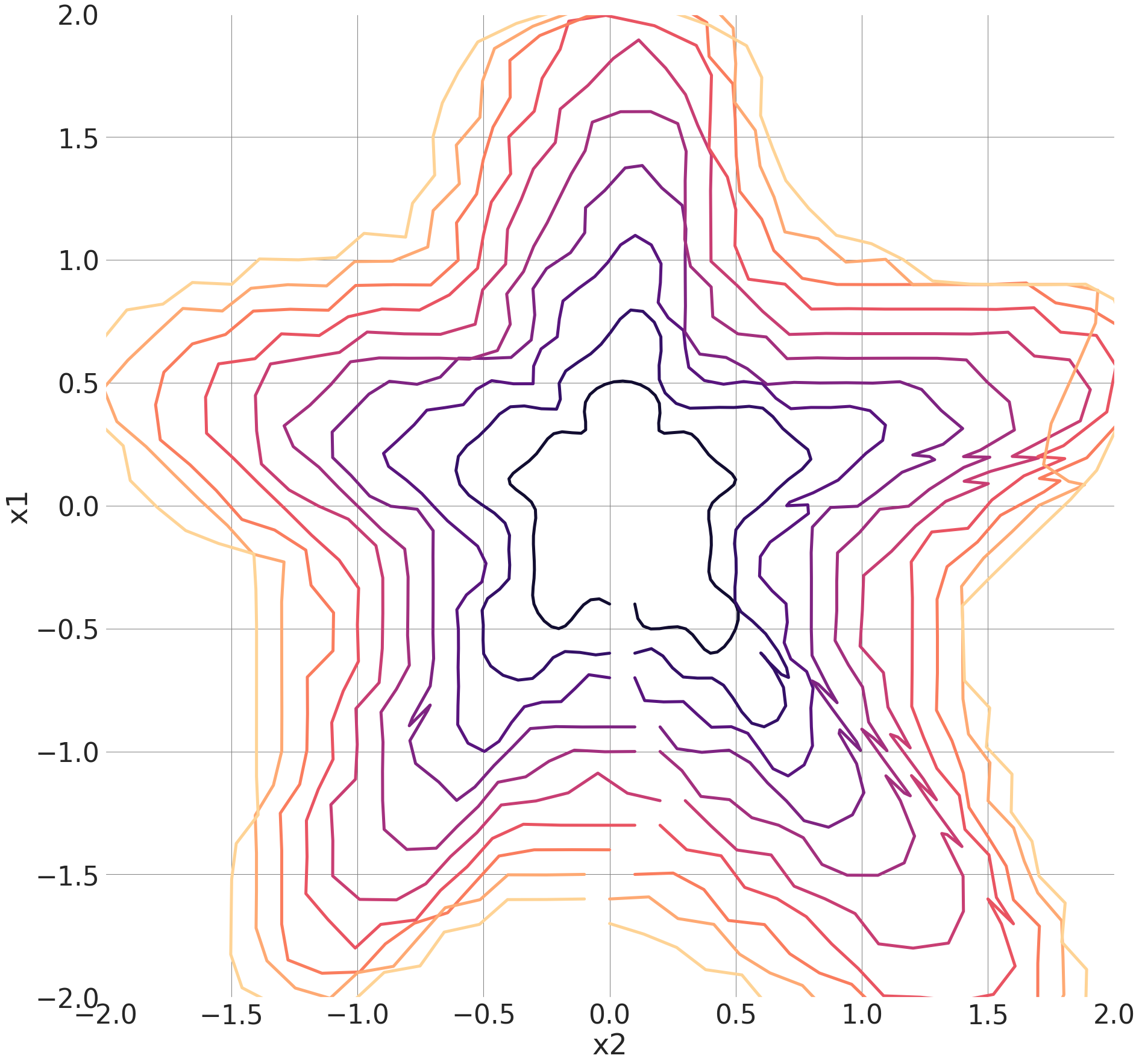}
        \captionsetup{labelformat=empty}
        \captionsetup{justification=centering}
        \caption{Model-based IQS}
    \end{minipage}\hfill
    \begin{minipage}[l]{0.2\linewidth}
        \centering
        \includegraphics[trim={0 0 0 0}, clip, width=1\linewidth]{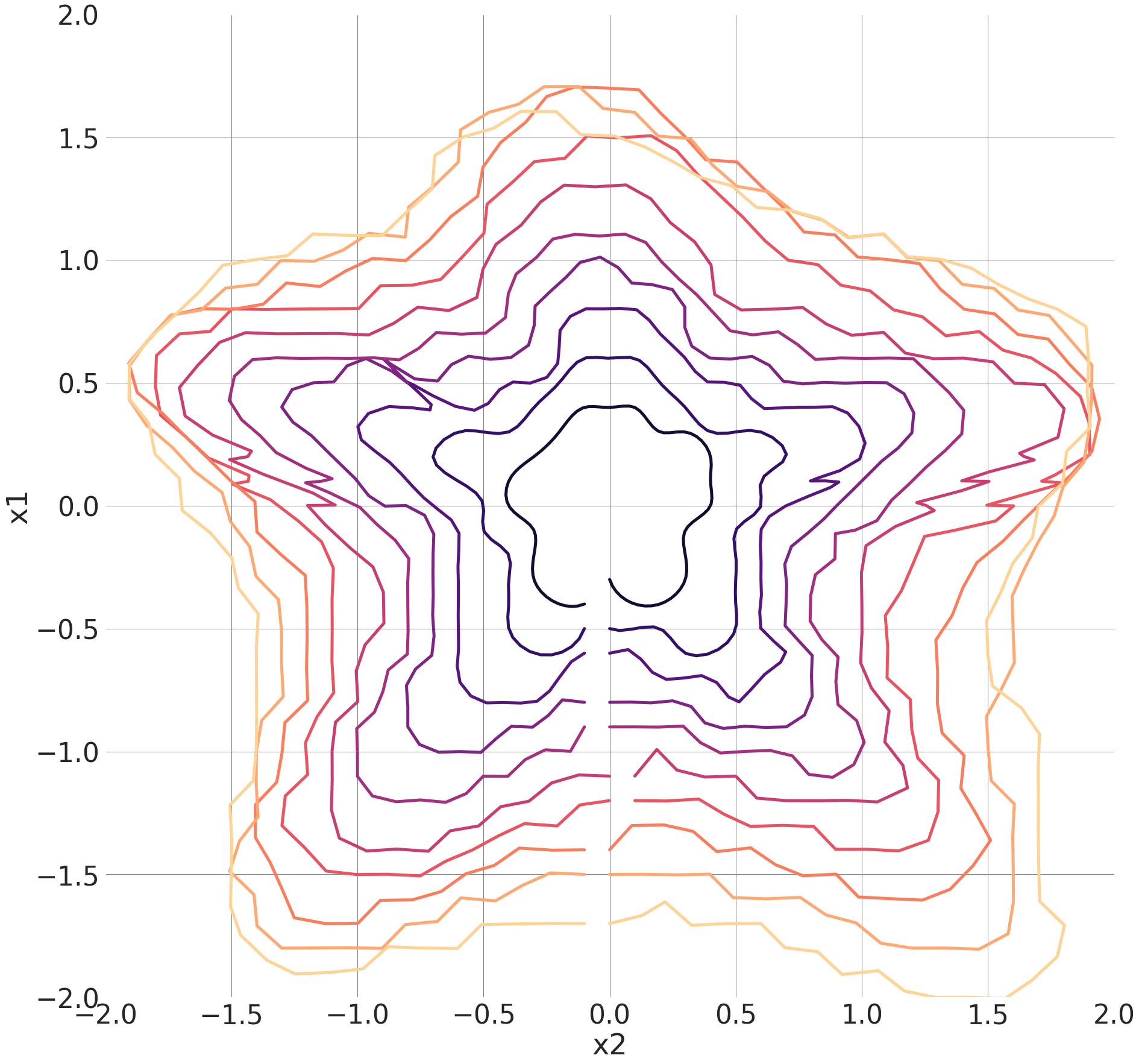}
        \captionsetup{labelformat=empty}
        \captionsetup{justification=centering}
        \caption{Model-based IQS-CS-SMOTE}
    \end{minipage}\hfill

    \begin{minipage}[l]{0.152\linewidth}
        % \centering
        \includegraphics[trim={55cm 0 0cm 0}, clip, width=1\linewidth, angle=270]{images/star_1000_iq_cssmote_heat_q0q1.png}
        \captionsetup{labelformat=empty}
    \end{minipage}
    \setcounter{figure}{13}
    \caption{\textbf{Top to bottom}: Heat-maps of \(Q(s,a=0)\), \(Q(s,a=1)\), \(Q(s,a=0)-Q(s,a=1)\) at \(t=25\), and recovered stopping boundaries per time step for the Star example.}
    \label{fig:star_heatmaps}
    \vskip 5pt
\end{figure*}
\subsection{Robustness to discount factor misspecification}
One of the caveats in solving IOS not discussed so far is that the discount factor used by the expert may not always be available to the inverse learner. We first note that the discount factor misspecification could be intrinsically corrected by IOS, if the time is the part of the state space. To see that, assume that the expert used the discount factor \(\gamma_E\) and a reward function \(r_E\) to simulate stopped trajectories. Define a discount factor as a function of time such that \(\gamma^t=\gamma(t)\).Then \(V_E(s)=\mathbb{E}_s\left[\sum_{t=0}^\infty\gamma_E^t r(s_t,a_t)\right]=\mathbb{E}_s\left[\sum_{t=0}^\infty\gamma_E(t) r(s_t,a_t)\right]=\mathbb{E}_s\left[\sum_{t=0}^\infty\tilde{r}(s_t,a_t)\right]\). Hence, the Q-function recovered by the IOS algorithm should be able to account for the discount factor misspecification.
% To confirm this, we simulate data as in example \ref{} with the discount factor \(\gamma=0.75\) and run our algorithm multiple times setting the discount factor to be \(\left[0.4,0.45,0.5,0.65,0.7,0.75,0.8,0.85,0.9,0.95,1\right]\). We repeat the experiment for 10 random seeds taking values from 0 to 10 and report the results.

\clearpage
\end{document}